\documentclass[smallcondensed]{svjour3}

\usepackage{silence}
\usepackage{ulem}
\usepackage{enumitem}
\usepackage{graphicx}
\usepackage{subcaption}
\usepackage{float}
\usepackage{url}
\usepackage{hyperref}
\hypersetup{
 colorlinks=true,
 linkcolor=darkgreen,
 filecolor=darkgreen,
 citecolor = darkgreen,      
 urlcolor=cyan,
}
\usepackage{etoolbox}
\usepackage{amsmath,amssymb,amsfonts}

\usepackage{multirow}
\usepackage{multicol}
\usepackage{color,soul}
\usepackage[svgnames, table]{xcolor} 
\usepackage{svg}

\setstcolor{red} 
\usepackage{natbib}
\usepackage{makecell}

\usepackage{pdflscape} 
\usepackage{pdfpages}
\usepackage{rotating}

\usepackage{threeparttable}

\allowdisplaybreaks

\usepackage{caption}

\usepackage{booktabs}
\usepackage{tikz}
\usetikzlibrary{arrows,positioning,shapes}
\usepackage{lipsum, adjustbox}
\usepackage{algorithm,algorithmicx,algpseudocode}

\usepackage{setspace}

\usepackage{xcolor}

\newcommand{\mca}{\mathcal}
\newcommand{\hhsp}{\-\hspace}

\newcommand*\boxSizeOfMax[1]{\makebox[\widthof{max}][c]{#1}}
\newcommand{\sthat}{\boxSizeOfMax{s.t.}}

\definecolor{wisconsin-red}{rgb}{0.6,0,0}
\definecolor{darkgreen}{rgb}{0.2,0.6,0.2}
\definecolor{maroon}{rgb}{0.5, 0.0, 0.0}
\definecolor{violet}{rgb}{0.75, 0.0, 1.0}
\definecolor{lightgray}{gray}{0.9}
\definecolor{navyblue}{rgb}{0.0, 0.0, 0.35}
\definecolor{darkmidnightblue}{rgb}{0.0, 0.2, 0.4}
\definecolor{Gray}{gray}{0.75}
\definecolor{darkgreen}{rgb}{0,0.5,0}

\newcommand{\policy}{\pi}
\newcommand{\setPolicy}{\Pi}
\newcommand{\bePt}{b}
\newcommand{\vectBePt}{\boldsymbol{\bePt}}
\newcommand{\timeHorizon}{T}

\newcommand{\genericState}{X}
\newcommand{\genericAction}{Y}

\newcommand{\setTime}{\mca{T}}
\newcommand{\setState}{\mca{S}}
\newcommand{\setAction}{\mca{A}}
\newcommand{\setObs}{\mca{O}}
\newcommand{\beSimp}{\mathcal{B}(\setState)}
\newcommand{\transitionProbs}{P}
\newcommand{\observationProbs}{Z}
\newcommand{\rewardVals}{R}

\newcommand{\rewFinal}{R}

\newcommand{\indTime}{t}
\newcommand{\indState}{i}
\newcommand{\indStateNew}{j}
\newcommand{\indStateNewNew}{\indState'}
\newcommand{\indAction}{a}
\newcommand{\indObs}{o}
\newcommand{\indGrid}{k}
\newcommand{\indGridNew}{\ell}

\newcommand{\parfRew}{w}
\newcommand{\parfCost}{c}
\newcommand{\parfObs}{z}
\newcommand{\parfTrp}{p}
\newcommand{\parfGridTr}{f}

\newcommand{\setGridIndex}{\mca{K}}
\newcommand{\setGrid}{\mca{G}}
\newcommand{\setGridEval}{\bar{\setGrid}}
\newcommand{\numGrids}{N}
\newcommand{\numGridSets}{m}
\newcommand{\indGridSet}{\iota}
\newcommand{\grPt}{g}
\newcommand{\ugrPt}{\grPt'}
\newcommand{\vectGrPt}{\boldsymbol{\grPt}}
\newcommand{\ugrPtvect}{\boldsymbol{\ugrPt}}
\newcommand{\resoVal}{\rho}

\newcommand{\parBudgetLim}{B}

\newcommand{\parDiscount}{\lambda}

\newcommand{\funcOptV}{V}
\newcommand{\funcOptC}{C}
\newcommand{\funcApxV}{\hat{\mca{\funcOptV}}}
\newcommand{\funcApxC}{\hat{\mca{\funcOptC}}}

\newcommand{\varPrimalLP}{u}
\newcommand{\varDualLP}{x}
\newcommand{\parBeDistn}{\delta}
\newcommand{\varDetPol}{\theta}

\newcommand{\setAlpha}{\Gamma}

\newcommand{\setGridAlpha}{\hat{\Gamma}}

\DeclareMathOperator{\lbV}{\funcApxV_{\text{LB}}}
\DeclareMathOperator{\ubV}{\funcApxV_{\text{UB}}}
\DeclareMathOperator{\LPV}{\funcApxV_{\text{LP}}}
\DeclareMathOperator{\MIPV}{\funcApxV_{\text{MIP}}}
\DeclareMathOperator{\SimV}{\funcApxV_{\text{Sim}}}
\DeclareMathOperator{\expCost}{C}
\DeclareMathOperator{\expVal}{V}

\newcommand{\stBelief}{\vectBePt^0}
\numberwithin{theorem}{section}
\newcommand{\xmark}{No}
\newcommand{\cmark}{Yes}

\begin{document}\sloppy
\title{Linear programming-based solution methods for constrained partially observable Markov decision processes}
\titlerunning{LP-based solution methods for constrained POMDPs}
\author{Robert K. Helmeczi \and Can Kavaklioglu \and Mucahit Cevik}

\institute{
Robert K. Helmeczi \at Toronto Metropolitan University, Toronto, Canada\\
\and
Can Kavaklioglu \at Toronto Metropolitan University, Toronto, Canada\\
\and
Mucahit Cevik \at 
Toronto Metropolitan University, Toronto, Canada\\
\email{mcevik@ryerson.ca}
}

\date{Received: date / Accepted: date}

\maketitle

\begin{abstract}
Constrained partially observable Markov decision processes (CPOMDPs) have been used to model various real-world phenomena. 
However, they are notoriously difficult to solve to optimality, and there exist only a few approximation methods for obtaining high-quality solutions.
In this study, grid-based approximations are used in combination with linear programming (LP) models to generate approximate policies for CPOMDPs.
A detailed numerical study is conducted with six CPOMDP problem instances considering both their finite and infinite horizon formulations.
The quality of approximation algorithms for solving unconstrained POMDP problems is established through a comparative analysis with exact solution methods.
Then, the performance of the LP-based CPOMDP solution approaches for varying budget levels is evaluated.
Finally, the flexibility of LP-based approaches is demonstrated by applying deterministic policy constraints, and a detailed investigation into their impact on rewards and CPU run time is provided.
For most of the finite horizon problems, deterministic policy constraints are found to have little impact on expected reward, but they introduce a significant increase to CPU run time.
For infinite horizon problems, the reverse is observed: deterministic policies tend to yield lower expected total rewards than their stochastic counterparts, but the impact of deterministic constraints on CPU run time is negligible in this case.
Overall, these results demonstrate that LP models can effectively generate approximate policies for both finite and infinite horizon problems while providing the flexibility to incorporate various additional constraints into the underlying model.

\end{abstract}
\keywords{Constrained POMDPs, linear programming, integer programming, approximation methods}

\section{Introduction}
Markov decision processes (MDPs) model sequential decisions over a finite or infinite number of discrete time intervals, known as decision epochs. 
In MDPs, the system state is fully observable, but there is uncertainty associated with the state transitions. 
Specifically, the effects of actions/decisions on the current state are characterized by transition probability values.
As the system state changes and actions are taken, some transitions accrue rewards according to a reward mechanism. 
Solving an MDP problem constitutes determining an optimal policy, which defines the actions for all states in the decision process. 
Because of the relatively simple model structure, it is possible to solve large MDP models (e.g., number of states in the order of tens of thousands) using dynamic programming methods.
For even larger problem instances, reinforcement learning strategies have been successfully employed to generate approximate policies.

Not all problems consist of fully observable state spaces. 
In many real-world problems, directly observing the current system state is not possible.
For instance, sensor noise in robotics or inconclusive diagnostic tests in health care problems contribute to ambiguities about the true state of the process. Partially observable MDPs (POMDPs) extend the MDP framework with a set of observations and observation probabilities that provide informative signals about the underlying system state.
POMDP models are used in a wide range of applications such as finding victims using UAV images~\citep{bravo2019use}, spoken dialogue systems~\citep{young2013pomdp} and the robotic manipulation of objects~\citep{pajarinen2017robotic}.

Real-world environments that are modelled by using MDPs often involve trade-offs between the expected value and the cost. 
While these two quantities can be amalgamated into a single reward function by using cost definitions as negative terms and value definitions as positive terms, a more rigorous solution is to formulate such problems as constrained POMDPs (CPOMDPs). 
CPOMDPs extend POMDPs with cost constraints that are defined separately from rewards. The resulting policies are designed to maximize the expected reward while conforming to resource limitations.
Typically, CPOMDPs extend unconstrained POMDPs with a cost function, which specifies the amount of resources required to take a particular action in a given state at a given time; and a budget term, which stipulates the total amount of resources available throughout the decision process. 
The budget term specifies an upper limit on the expected total cost incurred by the generated policy.
There are numerous examples in recent literature that use CPOMDPs to solve budget-restricted decision problems.
These examples include wideband spectrum sensing~\citep{jiang2017finding}, patient treatment design~\citep{gan2019personalized}, and robot navigation~\citep{wray2022scalable}.

In this study, grid-based approximations for CPOMDPs are used to formulate linear programming (LP) models, and various constraints are incorporated into the decision-making process.
The contributions of this study can be summarized as follows:
\begin{itemize}\setlength\itemsep{0.3em}
    \item LP models are formulated to generate approximate policies for both finite and infinite horizon CPOMDPs.
    Specifically, grid-based approximations are employed and transition probabilities between the grid points are characterized using an iterative approach.
    
    \item Six CPOMDP instances are introduced as adaptations of popular unconstrained POMDP test problems from the literature. 
    These instances provide a diverse test suite for empirical analysis and provide a means for comparing against \citet{poupart2015}'s constraint-based approximate linear programming (CALP) algorithm for CPOMDPs.
    The ensuing analysis sheds light on the effectiveness of approximate solution methods for CPOMDPs.
    
    \item A simulation mechanism is designed to test the performance of CPOMDP policies. The corresponding results provide the average collected rewards along with the level of adherence to budget limits. 
    The simulation results show that grid-based approximations can lead to a certain degree of violation of the constraints in some problem instances, which is consistent with the findings of previous studies on CPOMDPs~\citep{poupart2015}. 
    
    \item The flexibility of LP-based approaches is demonstrated by incorporating different types of constraints into the model formulation, including deterministic policy constraints and threshold-type policy constraints.
    These constraints can be highly relevant in some domains (e.g., in healthcare) and their characterization is important for the practical use of CPOMDP models.
    Accordingly, a detailed investigation into the impact of these constraints on CPOMDP rewards and run times for both finite and infinite horizons problems is provided. 
    This analysis reveals the feasibility of producing deterministic policies, and also outlines their impact on collected rewards using a direct comparison with their randomized counterparts.
\end{itemize}

The rest of the paper is organized as follows. 
Section~\ref{cpomdp:sec:lit_rev} provides a review of the literature on CPOMDPs.  
Sections~\ref{cpomdp:sec:finite} and~\ref{cpomdp:sec:infinite} provide mathematical descriptions of CPOMDPs as well as the proposed solution methods for finite and infinite horizon CPOMDP problems, respectively.
Section~\ref{cpomdp:numerical_experiments} summarizes the results of numerical experiments with the CPOMDP problem instances. 
Section~\ref{cpomdp:conclusion} concludes the paper with a summary of the empirical findings and a discussion on future work.

\section{Literature Review}
\label{cpomdp:sec:lit_rev}
This section provides an overview of MDPs, POMDPs and CPOMDPs, and reviews the recent advances on these methodologies. 
We refer the reader to \citet{puterman2014markov} for the technical foundations of MDPs, and \citet{spaan2012partially} for a thorough overview of POMDPs.

The solution to an MDP problem is a policy that specifies a decision rule corresponding to a probability distribution over the action space for all the states and decision epochs~\citep{Cassandra1995}.
Typically, Bellman optimality equations are solved to obtain such policies by using value iteration, policy iteration, or linear programming~\citep{puterman2014markov}.
\citet{alagoz2015optimally} empirically demonstrate that linear programming can be used to optimally solve many MDPs significantly faster than standard dynamic programming methods, such as value iteration and policy iteration, while also requiring less memory, and allowing for problems with a higher number of states to be solved to optimality.

In recent studies, MDPs have been used to model many real-world problems that involve excessively large state and action spaces.
Reinforcement learning (RL) has emerged as an effective approach for such problems. Model-free methods such as Q-learning and deep Q-networks, which do not require a full description of MDP model components, have been shown to achieve state-of-the-art performance as approximation methods~\citep{sutton2018reinforcement}.
Some other studies focused on developing exact solution methods for novel MDP models that are designed for complex problems.
For instance, in some domains, there exist multiple perspectives on the same phenomenon, resulting in distinct datasets for the same problem. 
These problems can be better explained by using multiple models simultaneously rather than a single model. 
Multi-model MDPs were proposed as a novel framework to solve such problems as they allow generating an optimal policy based on the weights assigned to each model~\citep{steimle2021multi}. 
Solving multi-model MDPs typically involves mixed-integer LP formulations.
Accordingly, recent studies on multi-model MDPs focus on mathematical optimization techniques such as branch-and-cut and policy-based branch-and-bound methods to improve the solvability of MDP models~\citep{steimle2021decomposition, ahluwalia2021policy}.

While standard MDP models only involve a generic reward component, it is also possible to have certain costs associated with the actions and states in the decision process.
Many real-world problems include a trade-off between available resources and expected rewards. 
Constrained MDPs can be used to incorporate budget considerations into the decision process by employing a cost function which specifies the amount of resources required to take a particular action.
There are several approaches in the literature that extend MDP models to handle costs as a separate entity. 
One approach is to develop customized linear programs that have constraints specifically designed to address any number of cost functions. This approach was employed by \citet{mclay2013dispatching} in an infinite horizon setting for the ambulance dispatch problem.
\citet{ayvaci2012effect} formulate a constrained MDP model for the finite-horizon breast cancer biopsy decision-making problem.
They used linear programming to solve their constrained MDP model at varying budget levels to capture the impact of limited resources on the resulting policies that prescribe when to recommend a patient for biopsy.

POMDP models relax the completely observable system state assumption in MDPs. 
As such, the actions and policies in POMDPs are determined based on the probability of being in a particular state. 
Hence, the values are only calculated for belief states, which are defined as probability distributions over the core state space.
POMDPs are important modeling tools for complex problems in which actions guide the system trajectory in a desired manner. 
POMDPs have applications in many domains including machine maintenance and replacement~\citep{Maillart2006}, inventory control~\citep{Treharne2002} and cancer screening~\citep{Ayer2011a, Erenay2014}.

The search for optimal POMDP policies is further complicated by the addition of observations to the model parameters, as well as the transition to continuous state space. 
\citet{sondik1971optimal} show that optimal POMDP value functions are piece-wise linear and convex, and can be written as a dot product of belief states and $\alpha$-vectors that represent the values at each state.
This representation has facilitated various exact solution methods such as Monahan's exhaustive enumeration algorithm~\citep{Monahan1982} and the incremental pruning (IP) algorithm~\citep{cassandra1998exact}.
However, these exact solution methods can typically solve only small problem instances with a few core states and actions.
Accordingly, most POMDP solution approaches rely on approximation mechanisms.
Grid-based approximations are among the most commonly used methods, and they can be used to reduce a POMDP to an MDP~\citep{sandikci2010reduction}.
\citet{Lovejoy1991} show how to use grid-based approximations to generate lower bounds and upper bounds on the optimal values obtained by POMDP policies.
In a recent study, \citet{kavaklioglu2022scalable} compare distributed, parallel, and sequential implementations of \citet{Lovejoy1991}'s lower bound and upper bound methods.

Many approximation methods for POMDPs with large state spaces focus on point-wise evaluations of the value functions over a set of grid points that correspond to specific belief states.
However, using a fixed set of grid points often fails to produce high-quality approximations, especially when the identified grid set cannot adequately describe the belief space. 
\citet{Pineau2006a}'s point-based value iteration (PBVI) algorithm proposes the exploration of the value function only at select belief points. New belief points are considered depending on the updates performed on the selected belief points, which determines the transitions between these points. 
Another popular point-based approach, proposed by \citet{silver2010monte}, relies on the Monte Carlo tree search (MCTS) algorithm. 
MCTS only considers sampled transitions rather than all possible transitions, which helps alleviate computational complexity by only considering a simulated subset of the environment. 
In practice, MCTS methods are able to quickly converge to a good solution. 
Parallel implementations allow MCTS algorithms to scale up depending on the problem size and the requested solution quality. 
While MCTS has been shown to converge to optimal solutions in theory, they fail to generate exact solutions in practice, and they do not provide any bounds on the optimal values.
There are several libraries which provide a standardized method of defining POMDP problems and utilize the previously developed solution algorithms to find optimal/approximate policies.
For instance, POMDPs.jl~\citep{egorov2017pomdps} provides access to the implementations of the IP and MCTS algorithms, among others. 

Similar to the relationship between MDP and constrained MDP problems, CPOMDPs extend POMDP models with the inclusion of cost and cost/budget constraints. 
CPOMDP solution algorithms attempt to maximize the expected reward without exceeding the available budget. 
Earlier studies for CPOMDPs focused on dynamic programming-based approaches, where separate $\alpha$-vectors are generated for cost and reward functions to keep track of the rewards and the costs in the decision process~\citep{Kim2011}.
This approach is computationally expensive as the number of $\alpha$-vectors increases exponentially. As a result, the authors also proposed using an approximate point-based method that keeps track of the reward and cost accrued by each belief point. 
These $\alpha$-vectors can be used to calculate the cost values for the generated solutions, which can then be bounded by a constraint to identify policies that adhere to the specified constraints.

Another popular approach to solving POMDP problems is to repeatedly solve linear programs that search for optimal actions at each decision epoch. 
The resulting LP models can also be augmented with additional constraints to ensure adherence to the cost requirements. 
One such solution approach is the CALP algorithm proposed by~\citet{poupart2015}. 
CALP reduces a POMDP to an MDP by defining a finite set of belief states $\setGrid$, where each $\vectBePt\in \setGrid$ acts as a state in the resulting MDP. 
The MDP is then modeled using an LP with a cost constraint and an associated budget.
An alternative approach is to use a column generation approach, where each column of the LP model represents a separate CPOMDP policy~\citep{walraven2018column}. 
In this method, the LP formulation ensures that the generated policies adhere to the model constraints. 
The CPOMDP version of the MCTS algorithm, developed by~\citet{lee2018monte}, shows promising results for large core state spaces. 
This approach searches for policies that satisfy a given cost constraint or produce no result if it fails to identify a policy that satisfies the constraint.

There are several applications of CPOMDPs in healthcare which involve enforcing problem-specific constraints.
\citet{cevik2018analysis} propose a finite-horizon CPOMDP model that maximizes a patient's quality-adjusted life years (QALYs) while limiting the number of mammographies taken in the patient's lifetime.
They use mixed-integer linear programming methods to solve their CPOMDP models.
\citet{gan2019personalized} develop CPOMDP models to investigate personalized treatment strategies for opioid use disorder.
They consider wearables designed to measure patient cravings and adapt patient treatment plans accordingly.
In their model, the state space includes several levels of patient cravings, drug relapse, detoxification, and death/withdrawal from treatment.
As the cravings cannot be quantified with certainty, this set of states is only partially observable. 
At discrete intervals, the probability of a patient occupying each state is updated based on the information gathered by the wearable. 
The costs in their model account for the expenses associated with the wearables, as well as the patient treatments.
They propose an approximation technique for solving their CPOMDP model, based on~\citet{cassandra1998exact}'s incremental pruning algorithm, which first dualizes the constraints to the value functions (hence creating an unconstrained formulation), then solves the value functions by assuming large $\epsilon$-values in their convergence check.

\citet{jiang2017finding} employ CPOMDPs for wideband spectrum sensing. 
In this problem, the spectrum is quantized into a set of subbands and at each decision epoch, a set of subbands is chosen to be accessed. 
This set undergoes a detection process in which each subband is detected as idle or busy. 
The detection process is stochastic in nature: there is some probability that a busy subband will be detected as idle and vice-versa. This problem is constrained by associating a cost with the sensing time for each subband, and the objective is to maximize the expected throughput.
\citet{wray2022scalable} formulate a robot navigation problem using CPOMDPs.
They adapt their model to a practical problem in home healthcare delivery, where the objective is to create a robot navigation policy for patient monitoring or delivering medication reminders. 
In the constrained version of this problem, certain rooms involve extra costs which discourage or prevent the robot from entering them while searching for its target.

CPOMDPs are relevant in many other domains including transportation (e.g., autonomous vehicles) and manufacturing.
For instance, \citet{celen2020integrated} study a manufacturing problem where the objective is to determine the sequence in which different product types are released into a manufacturing process so that the correct number of each product is manufactured within a given time frame.
This problem involves a set of manufacturing stations, where each station is capable of performing a subset of all possible manufacturing operations, and where the production of a particular product type may involve multiple manufacturing operations.
The problem is further complicated by the fact that each station can degrade, with the state of each station being partially observable.
The authors propose a POMDP model for this problem where the cost of repairing manufacturing stations is incorporated into the reward function.
On the other hand, a CPOMDP model can also be considered for this problem in which the objective can be defined as producing the requested number of each product, subject to a budget constraint on the number or cost of maintenance operations.
Similarly, POMDPs/CPOMDPs can also be employed for other manufacturing problems, such as assembly line balancing, where uncertainties involving random task arrivals and equipment damages need to be taken into account to ensure that manufacturing goals are met~\citep{yilmaz2020integrated, yilmaz2021tactical}.

Table~\ref{tab:lit-review} summarizes the most related studies to our work.
The existing studies are categorized based on supported decision horizons, the types of policies generated by the algorithms, whether budget adherence is guaranteed (i.e., feasible), and their main methodology.
Specifically, finite horizon problems require policies which execute actions at a predetermined number of decision epochs, after which the decision process terminates. 
Conversely, infinite horizon problems never terminate, and the problems are solved to obtain policies for the steady state. 
In practice, many problems with an initially unknown number of decision epochs may be modeled as infinite horizon problems. 
\citet{cevik2018analysis} and~\citet{gan2019personalized} provide practical use cases for finite horizon problems in healthcare.

\begin{table}[!ht]
\centering
\caption{Summary of the relevant literature on CPOMDPs.}
\label{tab:lit-review}
\resizebox{\textwidth}{!}{\begin{tabular}{lllllll}
\toprule
    & \multicolumn{2}{c}{Horizon} & \multicolumn{2}{c}{Policy-Type}& Guarantees &\\
    \cmidrule(lr){2-3} \cmidrule(lr){4-5} \cmidrule(lr){6-6}
    Research &  Finite & Infinite & Stochastic & Deterministic &  Feasible & Methodology \\
\midrule
    \citet{Kim2011} & \xmark & \cmark & \cmark & \xmark & \cmark &  PBVI \\ 
    \citet{poupart2015} & \xmark & \cmark & \cmark & \xmark & \xmark &  LP \\
    \citet{walraven2018column} & \cmark & \xmark & \cmark & \xmark & \xmark &  CG \\
    \citet{lee2018monte} & \xmark & \cmark & \cmark & \cmark & \xmark &  MCTS \\ 
    \citet{gan2019personalized} & \cmark & \xmark & \cmark & \cmark & \cmark  & IP \\ 
    \citet{wray2022scalable} & \xmark & \cmark & \cmark & \xmark & \xmark  & GA\\ 
\midrule
Our Study & \cmark & \cmark & \cmark & \cmark &  \xmark  & LP\\ 
\bottomrule
\end{tabular}}
    \begin{tablenotes}
      \footnotesize
      \item IP: Incremental Pruning, CG: Column Generation, MCTS: Monte-Carlo Tree Search, GA: Gradient Ascent, LP: Linear Programming, PBVI: Point-Based Value Iteration
    \end{tablenotes}
\end{table}

The ability to generate stochastic policies as well as deterministic ones is beneficial in the CPOMDP domain. 
This is because, unlike unconstrained POMDPs, there is no guarantee that there exists an optimal deterministic policy for a CPOMDP problem. 
Thus, when there is no restriction on the allowed policy type, stochastic policies have the potential to yield a substantially higher reward. 
However, deterministic policies, despite being suboptimal, are sometimes preferred as they facilitate implementation in practical settings. For example, in the medical domain, choosing a treatment plan for identical patients at random is generally not allowed.

Many CPOMDP algorithms iteratively work towards an optimal solution. 
Such algorithms typically show that the final converged solution will be feasible with respect to the imposed budget limit. 
In practice, however, the iterations are generally stopped early (e.g., after some time limit), and the resulting intermediate policy may not be feasible. 
Table~\ref{tab:lit-review} shows that many of the existing methods do not guarantee returning a feasible CPOMDP policy.

\section{Finite Horizon CPOMDPs}\label{cpomdp:sec:finite}
Discrete-time finite horizon POMDP models aim to maximize expected total rewards over a given decision horizon. A finite horizon unconstrained POMDP is defined as the 7-tuple $\langle\setTime, \setState, \setAction, \setObs, \transitionProbs, \observationProbs, \rewardVals\rangle$ where $\setTime = \{1,2, \hdots, \timeHorizon\}$ represents the set of decision epochs, $\setState$ represents the set of states, $\setAction$ represents the set of actions, and $\setObs$ represents the set of observations.
The last three components govern the uncertainty in a POMDP, with $\transitionProbs$, $\observationProbs$, and $\rewardVals$ corresponding to the transition probabilities, observation probabilities, and rewards in the decision process, respectively.
Furthermore, the probability of making a transition from state $\indState \in \setState$ to $\indStateNew \in \setState$ at time $\indTime \in \setTime$ based on the action $\indAction \in \setAction$ is given by $\parfTrp_{\indState \indStateNew}^{\indTime\indAction}$.
Similarly, the probability of making observation $\indObs \in \setObs$ after making a transition to state $\indStateNew \in \setState$ at time $\indTime \in \setTime$ for action $\indAction \in \setAction$ is $\parfObs_{\indStateNew \indObs}^{\indTime\indAction}$.
Lastly, the reward accrued for taking action $\indAction \in \setAction$ in state $\indState \in \setState$ at time $\indTime \in \setTime$ is $\parfRew_{\indState\indAction}^{\indTime}$.

Because the system states are not fully observable in a POMDP, a belief state $\vectBePt = [\bePt_{0}, \bePt_{1},\hdots, \bePt_{|\setState|}]$ is defined to estimate the underlying system state.
A belief state is a probability distribution over the state space, with $\bePt_{\indState}$ representing the probability of occupying state $\indState \in \setState$. 
The set of all belief states is known as the belief space, and it can be represented as $\beSimp = \big\{ \vectBePt\ |\ \sum_{\indState \in \setState} \bePt_{\indState} = 1, \ \bePt_{\indState} \geq 0, \ \indState \in \setState \big\}$.
At decision epoch $\indTime\in \setTime$, after taking action $\indAction \in \setAction$ and making observation $\indObs \in \setObs$, the belief state $\vectBePt$ (i.e., $\vectBePt_{\indTime}$) is updated to $\vectBePt'$ (i.e., $\vectBePt_{\indTime+1}$), where the $\indStateNew$th component of $\vectBePt'$ can be calculated using Bayes' rule as follows:
\begin{equation}
\bePt'_{\indStateNew}  = \frac{\sum_{\indState \in \setState} \bePt_{\indState} \ \parfObs_{\indStateNew \indObs}^{\indTime \indAction} \ \parfTrp_{\indState \indStateNew}^{\indTime \indAction}}{\sum_{\indState \in \setState}\sum_{\indStateNewNew \in \setState} \bePt_{\indState} \ \parfObs_{\indStateNew \indObs}^{\indTime \indAction} \ \parfTrp_{\indState \indStateNewNew}^{\indTime \indAction}} 
\label{eq:bayesian-update}
\end{equation}

\subsection{Unconstrained POMDPs}
The expected total reward for a given starting belief state $\stBelief$ in a discrete-time finite horizon unconstrained POMDP can be obtained as
\begin{equation}
\label{eq:unconstrainedPOMDPModel}\max_{\policy \in \setPolicy}  \ E_{\stBelief}^{\policy} \big[\sum_{\indTime=0}^{\timeHorizon-1} \parfRew_{\indTime}(\genericState_{\indTime}, \genericAction_{\indTime}) + \parfRew_{\timeHorizon}(\genericState_{\timeHorizon}) \big]
\end{equation}
where $\setPolicy$ represents the policy space, $\parfRew_{\indTime}(\genericState_{\indTime}, \genericAction_{\indTime})$ corresponds to the generic reward function for action $\genericAction_{\indTime}$ and state $\genericState_{\indTime}$ at time $\indTime$ and $w_\timeHorizon(\genericState_{\timeHorizon})$ is its counterpart for the final decision epoch $\timeHorizon$.
The optimization problem given in \eqref{eq:unconstrainedPOMDPModel} is intractable. As a result, unconstrained POMDPs are typically solved using recursive Bellman optimality equations, which are defined as follows:
\begin{subequations}\label{eq:bellman-equations}
\begin{align}
\ & \ \funcOptV_{\timeHorizon}^*(\vectBePt) = \displaystyle \sum_{\indState \in \setState} \bePt_{\indState}\rewFinal_{\indState}, \quad \vectBePt \in \beSimp\label{eq:bellman-T-optimal} \\[1.6ex]
\ & \ \funcOptV_{\indTime}^*(\vectBePt)=\max_{\indAction\in \setAction}\left\{\funcOptV_{\indTime}^{\indAction}(\vectBePt)\right\}, \quad \indTime < \timeHorizon, \ \vectBePt \in \beSimp\label{eq:bellman-t-optimal}\\[1.6ex]
\ & \ \hhsp{-0.1cm} \funcOptV_{\indTime}^{\indAction}(\vectBePt)=  \sum_{\indState \in \setState} \bePt_{\indState} \parfRew_{\indState}^{\indTime \indAction}  + \sum_{\indState \in \setState}\bePt_{\indState} \bigg( \sum_{\indObs \in \setObs} \sum_{\indStateNew \in \setState} \parfObs_{\indStateNew \indObs}^{\indTime \indAction} \ \parfTrp_{\indState \indStateNew}^{\indTime \indAction}\ \funcOptV_{\indTime+1}^*(\vectBePt') \bigg), \quad \begin{matrix}\indAction \in \setAction,\\\indTime < \timeHorizon,\\ \vectBePt \in \beSimp\end{matrix}\label{eq:bellman-iteration-step}
\end{align}
\end{subequations}
Because there are infinitely many belief states, it is not feasible to directly solve these optimality equations that represent the POMDP value functions.
Exact solution methods for unconstrained POMDPs such as Monahan's enumeration algorithm and the incremental pruning algorithm typically rely on reformulating the value functions by defining $\alpha$-vectors such that
\begin{equation}\label{eq:alpha-vectors}
\funcOptV_{\indTime}^*(\vectBePt) = \max_{\alpha_{\indTime} \in \setAlpha_{\indTime}}\{\vectBePt \cdot \alpha_{\indTime}\}
\end{equation}
where $\setAlpha_{\indTime}$ represents the set of all $\alpha$-vectors at time $\indTime$.
However, these exact solution methods can only solve unconstrained POMDPs with a few states to optimality.
Accordingly, various approximation mechanisms have been proposed to solve large unconstrained POMDP models.
The main focus of this study is grid-based approximations, which serve as the foundation for the constrained POMDP solution algorithm employed.

\subsection{Grid-based approximations for POMDPs}
The value function $\funcOptV_{\indTime}^*(\vectBePt), \ \indTime \in \setTime, \ \vectBePt \in \beSimp$ can be approximated by discretizing $\beSimp$ into a set of grid points $\setGrid = \{\vectGrPt^\indGrid \mid \indGrid\in\setGridIndex\}$ where $\setGridIndex = \{1,\hdots, |\setGrid|\}$ is index set of $\setGrid$. Using these grid points, the approximate value function, $\funcApxV_{\indTime}(\vectBePt)$, is defined as follows:
\begin{equation}\label{eq:apx-value-function}
\funcApxV_{\indTime}^{\indAction}(\vectBePt) = \sum_{\indState \in \setState} \grPt_{\indState} \parfRew_{\indState}^{\indTime \indAction} +  \sum_{\indState \in \setState} \grPt_{\indState} \sum_{\indObs \in \setObs} \sum_{\indStateNew \in \setState} \parfObs_{\indStateNew \indObs}^{\indTime \indAction} \ \parfTrp_{\indState \indStateNew}^{\indTime \indAction}\ \sum_{\indGrid \in \setGridIndex} \beta_{\indGrid} \funcApxV_{\indTime+1}(\vectGrPt^{\indGrid}), \quad \begin{matrix}\vectBePt \in \beSimp, \\ \indAction \in \setAction\end{matrix}
\end{equation}
where $\sum_{\indGrid \in \setGridIndex} \beta_{\indGrid} \funcApxV_{\indTime+1}(\vectGrPt^{\indGrid}) \approx \funcOptV_{\indTime+1}^*(\vectBePt')$. That is, the value at the updated belief state $\vectBePt'$ is approximated by a convex combination of the grid values.

By simply focusing on the belief states (i.e., grid points) in the grid set $\setGrid$, a POMDP can be reduced to an MDP~\citep{sandikci2010reduction}.
It follows that the approximate value functions can be obtained as 
\begin{align}
\funcApxV_t(\vectGrPt) & = \max_{a\in \mca{A}} \bigg \lbrace \funcApxV_{\indTime}^{\indAction}(\vectGrPt) \bigg \rbrace, & \indTime < \timeHorizon,\quad \vectGrPt \in \setGrid\label{eq:apx-updated-value-function}\\
\funcApxV_{\timeHorizon}(\vectGrPt) & = \sum_{\indState \in \setState} \grPt_{\indState}\rewFinal_{\indState}, & \vectGrPt \in \setGrid\label{eq:apx-terminal-reward}
\end{align}
\citet{Lovejoy1991b} showed that these approximate values obtained over the grid set $\setGrid$ provide an upper bound on the optimal values, i.e., $\funcOptV_{\indTime}^*(\vectGrPt) \leq \funcApxV_\indTime(\vectGrPt) \ \text{for}\ \vectGrPt \in \setGrid$.
Accordingly, this method is referred to as the upper bound method for unconstrained POMDPs, denoted by ``UB'' in this paper.

Finding convex combinations that yield better approximate values (closer to $Q_{t+1}^*(\ugrPtvect)$) is important for obtaining tighter bounds, which can be achieved by computing the $\beta$-values as a solution to the following LP model~\citep{sandikci2010reduction}:
\begin{subequations}\label{eq:backward-induction}
\begin{align}
    \min &\ \sum_{\indGrid \in \setGridIndex} \funcApxV_{\indTime+1}(\vectGrPt^{\indGrid}) \beta_{\indGrid} \\ 
    \sthat & \ \sum_{\indGrid \in \setGridIndex} \beta_{\indGrid} \grPt_{\indState}^{\indGrid}=\grPt_{\indState}', \quad \indState \in \setState \\
    & \ \sum_{\indGrid \in \setGridIndex} \beta_{\indGrid}=1\\
    & \ \beta_{\indGrid} \geq 0, \quad \indGrid \in \setGridIndex 
\end{align} 
\end{subequations}
Because this particular grid-based approximation mechanism provides an upper bound on the optimal value, the minimization objective of the LP in~\eqref{eq:backward-induction} ensures tighter bounds.
Furthermore, the constraint set for the LP in~\eqref{eq:backward-induction} guarantees that the belief state $\vectGrPt'$ can be represented using the grid points in the grid set according to the resulting $\beta$-values.
Note that these $\beta$-values can be pre-computed and stored to be used as inputs for different solution methods.
Specifically, if the agent takes action $\indAction$ in grid $\vectGrPt^\indGrid$ at time $\indTime$ and observes $\indObs$, then $\beta_{\indGrid \indGridNew}^{\indTime \indAction \indObs}$ gives the coefficient of $\vectGrPt^\indGridNew$ in the convex combination representation of $(\vectGrPt^\indGrid)'$. That is,
\begin{equation}
(\vectGrPt^{\indGrid})' = \sum_{\indGridNew \in \setGridIndex} \beta_{\indGrid \indGridNew}^{\indTime \indAction \indObs} \vectGrPt^{\indGridNew}\ 
\end{equation}

Grid-based approximations can also be used with the solution methods that rely on the $\alpha$-vector representation of value functions.
For instance, Monahan's exhaustive enumeration algorithm enumerates all of the $\alpha$-vectors at each iteration, which are then pruned to the optimal set of $\alpha$-vectors by using Eagle's reduction and a linear programming-based redundancy/dominance check.
The full enumeration of the $\alpha$-vectors can be avoided by first generating one $\alpha$-vector per grid point in the grid set, and then following the $\alpha$-vector pruning steps of the algorithm to reach an $\alpha$-vector set that leads to an approximation of the optimal value functions.
\citet{Lovejoy1991b} showed that these approximate values obtained over the grid set $\setGrid$ provide a lower bound on the optimal values, i.e., $\funcApxV_\indTime(\vectGrPt) \leq \funcOptV_{\indTime}^*(\vectGrPt)\ \text{for}\ \vectGrPt \in \setGrid$.
Accordingly, this method is referred to as the lower bound method for unconstrained POMDPs, denoted by ``LB'' in this paper.
Proofs for the UB and LB methods are provided in Appendix~\ref{sec:appendix_ubproof} and Appendix~\ref{sec:appendix_lbproof}, respectively.

\subsubsection{Constructing the grid set}\label{subsec:grid-construction}
There are various grid construction techniques employed for approximating POMDP value functions, each with its own set of advantages and disadvantages.
A fixed-resolution grid approach samples belief states at equidistant intervals in each dimension as specified by the resolution parameter $\resoVal$~\citep{sandikci2010reduction}. 
The drawback of this approach is that the size of the resulting grid set $\setGrid$ grows exponentially with the number of states and $\resoVal$. 
The cardinality of the resulting grid set $\setGrid$ for this approach can be obtained as $|\setGrid| = \binom{|\setState| + \resoVal - 1}{|\setState| - 1}$.
Alternatively, a random sampling strategy can be used to generate a grid set of a desired size~\citep{Suresh2005}.
However, this approach makes it difficult to represent different parts of the grid space with equal likelihood.
Additionally, a randomly generated grid set may increase the computational overhead associated with generating the $\beta$-values for belief states that fall outside the grid set.

In this paper, a modified version of the fixed-resolution grid approach is employed in order to obtain a grid set with $\numGrids$ grid points. 
For each problem instance, grid sets $\setGrid_1,\hdots,\setGrid_{\numGridSets}$ are generated iteratively until $|\setGrid_{\numGridSets}| > \numGrids$, where $\setGrid_{\indGridSet}$ is obtained by using $\resoVal=\indGridSet$.
Then, $\numGrids - |\setGrid_{\numGridSets-1}|$ grid points are drawn at equal intervals from the set of grid points $\bar{\setGrid} = \setGrid_{\numGridSets}\setminus \setGrid_{\numGridSets-1}$ and added to $\setGrid_{\numGridSets-1}$ to obtain the final grid set.
A more thorough discussion on grid construction can be found in Appendix~\ref{sec:appendix_grid_construction}.

\subsubsection{Linear programming model}

When a POMDP is reduced to an MDP using a grid-based approximation, the resulting model can be solved using a backward induction mechanism.
Alternatively, LP models can be used to obtain the optimal values/policies for the MDP model~\citep{puterman2014markov}.
In order to simplify the notation used in the LP model, the transition probabilities between the grid points are charactarized as follows:
\begin{align}\label{eq:gridTransitionProbs}
\parfGridTr_{\indGrid \indGridNew}^{\indTime \indAction} = \sum_{\indObs\in \setObs} \sum_{\indState \in \setState} \sum_{\indStateNew \in \setState} \beta_{\indGrid \indGridNew}^{\indTime \indAction \indObs}
\grPt_{\indState}^{\indGrid} \parfObs_{\indStateNew \indObs}^{\indTime \indAction} \ \parfTrp_{\indState \indStateNew}^{\indTime \indAction}
\end{align}
That is, $\parfGridTr_{\indGrid \indGridNew}^{\indTime \indAction}\ \equiv \Pr_{\indTime}(\vectGrPt^{\indGridNew}|\vectGrPt^{\indGrid}, \indAction)$ corresponds to the probability of transitioning to grid point $\vectGrPt^{\indGridNew} \in \setGrid$ starting from grid point $\vectGrPt^{\indGrid} \in \setGrid$ at time $\indTime$ after taking action $\indAction$. With this updated notation, the approximate value function simplifies to:
\begin{equation}\label{eq:apx-value-given-grid-trans}
    \funcApxV_{\indTime}(\vectBePt) = \max_{\indAction\in\setAction} \left\{\sum_{\indState \in \setState} \grPt_{\indState} \parfRew_{\indState}^{\indTime \indAction} + \sum_{\indGrid \in \setGridIndex} \parfGridTr_{\indGrid \indGridNew}^{\indTime \indAction}\funcApxV_{\indTime+1}(\grPt^\indGrid)\right\}
\end{equation}

Let $\varPrimalLP_{\indTime\indGrid}$ denote the decision variables that correspond to the optimal values for grid point $\vectGrPt^{\indGrid}$ at time $\indTime$. The LP model for the corresponding finite-horizon MDP model is formulated as follows:
\begin{subequations}\label{eq:lp1}
\begin{align}
\label{eq:lp1Obj}\min \ \ & \sum_{\indGrid\in \setGridIndex} \parBeDistn_{\indGrid}\ \varPrimalLP_{0\indGrid}\\[1.0ex]
\label{eq:lp1Eqn1}\sthat \ \ & \varPrimalLP_{\indTime \indGrid}\ \geq\  \sum_{\indState \in \setState} \grPt_{\indState}^{\indGrid} \parfRew_{\indState}^{\indTime \indAction}  + \sum_{\indGridNew \in \setGridIndex} \parfGridTr_{\indGrid \indGridNew}^{\indTime \indAction}\ \varPrimalLP_{t+1 \indGridNew}, && \quad \indAction \in \setAction, \ \indGrid\in \setGridIndex, \  \indTime < \timeHorizon,\\[1.0ex]
\label{eq:lp1Eqn2} \ & \varPrimalLP_{\timeHorizon \indGrid} = \sum_{\indState \in \setState}\grPt_{\indState}^{\indGrid}\rewFinal_{\indState}, && \quad  \indGrid\in \setGridIndex,\\[1.0ex]
\label{eq:lp1Eqn3} \ & \varPrimalLP_{\indTime \indGrid} \ \ \textit{free}, && \quad \indGrid\in \setGridIndex,\ \indTime \leq \timeHorizon .
\end{align}
\end{subequations}
The objective function along with the constraints in \eqref{eq:lp1Eqn1} and \eqref{eq:lp1Eqn2} ensure that the values for the grid points are the ones corresponding to the maximizing action. 
Note that the $\parBeDistn$ parameter specifies the weight assigned to each grid point, with $\sum_{\indGrid \in \setGridIndex} \parBeDistn_{\indGrid} = 1$.
The constraints in~\eqref{eq:lp1Eqn1} link the values at successive decision epochs, and the constraints in~\eqref{eq:lp1Eqn2} determine the values at the final decision epoch where there are no decisions involved.

While the LP model provided in \eqref{eq:lp1} can be used to determine the values at each grid point, it does not directly generate the corresponding optimal policy.
As also discussed in \citep{puterman2014markov}, the \textit{dual} of this LP model can be employed to obtain the optimal policies for the MDP model.
Let $\varDualLP_{\indTime \indGrid \indAction}$ denote the decision variables that correspond to the \textit{occupancy measures}, which specify the fraction of time $\vectGrPt^{\indGrid}$ is visited and action $\indAction$ is taken at time $\indTime$. Similarly, $\varDualLP_{\timeHorizon \indGrid}$ can be defined for the final epoch, which does not involve any decisions regarding the actions.
The dual LP model is formulated as follows:
\begin{subequations}\label{eq:duallp1}
\begin{align}
\label{eq:duallp1Obj}\max \ \ & \sum_{\indTime < \timeHorizon}\sum_{\indAction\in \setAction}\sum_{\indGrid \in \setGridIndex}\sum_{\indState \in \setState} \grPt_{\indState}^{\indGrid} \parfRew_{\indState}^{\indTime \indAction} \varDualLP_{\indTime \indGrid \indAction} + \sum_{\indGrid \in \setGridIndex}\sum_{\indState \in \setState}\grPt_{\indState}^{\indGrid} \rewFinal_{\indState}\varDualLP_{\timeHorizon\indGrid}\\[1.0ex]
\label{eq:duallp1Eqn1}\sthat \ \ & \sum_{\indAction \in \setAction} \varDualLP_{0 \indGrid \indAction} = \parBeDistn_{\indGrid}, && \hspace{0.01cm} \indGrid\in \setGridIndex,\\[1.0ex]
\label{eq:duallp1Eqn2}\ & \sum_{\indAction \in \setAction} \varDualLP_{\indTime \indGrid \indAction} - \sum_{\indAction \in \setAction} \sum_{\indGridNew \in \setGridIndex}  \parfGridTr_{\indGridNew \indGrid}^{\indTime-1 \indAction} \varDualLP_{\indTime-1 \indGridNew \indAction} = 0, && \hspace{0.01cm}  \indGrid\in \setGridIndex, \  0 < \indTime < \timeHorizon,\\[1.0ex]
\label{eq:duallp1Eqn3}\ & \varDualLP_{\timeHorizon \indGrid} - \sum_{\indAction \in \setAction} \sum_{\indGridNew \in \setGridIndex} \parfGridTr_{\indGridNew \indGrid}^{\timeHorizon-1 \indAction} \varDualLP_{\timeHorizon-1 \indGridNew \indAction} = 0, && \hspace{0.01cm}  \indGrid\in \setGridIndex,\\[1.0ex]
\label{eq:duallp1Eqn4}\ & \varDualLP_{\timeHorizon \indGrid} \geq 0, \qquad \varDualLP_{\indTime \indGrid \indAction} \geq 0, && \hspace{0.01cm} \indAction \in \setAction, \  \indGrid \in \setGridIndex, \  \indTime < \timeHorizon.
\end{align}
\end{subequations}
Similar to \eqref{eq:lp1}, the objective function aims to maximize the expected total reward over the decision horizon as a weighted average over the grid points.
The constraints in~\eqref{eq:duallp1Eqn1} link the weights at the first decision epoch with the occupancy measures, implying that $\sum_{\indAction \in \setAction} \sum_{\indGrid \in \setGridIndex} \varDualLP_{0 \indGrid \indAction} = 1$.
The constraints in \eqref{eq:duallp1Eqn2} and \eqref{eq:duallp1Eqn3} link the occupancy measures at the successive decision epochs, and the constraints in \eqref{eq:duallp1Eqn4} state logical conditions on variables.
Appendix~\ref{sec:appendix_numerical_ex} provides illustrative examples on LP formulations for CPOMDPs.

\subsection{Finite Horizon CPOMDP Formulation}
The discrete-time finite horizon CPOMDP model can be formulated as follows:
\begin{subequations}
\label{eq:constrainedPOMDPModel}
\begin{align}
  \max_{\policy \in \setPolicy}  \ & E_{\stBelief}^{\policy} \big[\sum_{\indTime=0}^{\timeHorizon-1} \parfRew_{\indTime}(\genericState_{\indTime}, \genericAction_{\indTime}) + \parfRew_{\timeHorizon}(\genericState_{\timeHorizon}) \big]\\
\sthat \ & E_{\stBelief}^{\policy} \big[\sum_{\indTime=0}^{\timeHorizon-1} \parfCost_{\indTime}(\genericState_{\indTime}, \genericAction_{\indTime}) + \parfCost_{\timeHorizon}(\genericState_{\timeHorizon}) \big] \leqslant \parBudgetLim
\end{align}
\end{subequations}
where $\parfCost_{\indTime}(\genericState_{\indTime}, \genericAction_{\indTime})$ corresponds to the generic cost function 
for action $\genericAction_{\indTime}$ and state $\genericState_{\indTime}$ at time $\indTime$ and $\parfCost_\timeHorizon(\genericState_{\timeHorizon})$ is its counterpart for the final decision epoch $\timeHorizon$.
That is, the expected total reward should be maximized over the decision horizon $\setTime$ starting from the belief state $\stBelief$, while ensuring that the expected total cost does not exceed the budget limit, $\parBudgetLim$.

Similar to \eqref{eq:unconstrainedPOMDPModel}, the CPOMDP model provided in \eqref{eq:constrainedPOMDPModel} is intractable. 
Accordingly, the LP-based approximations for unconstrained POMDPs can be extended to obtain an approximation mechanism for CPOMDPs.
Specifically, the dual LP model provided in \eqref{eq:duallp1} can be extended as follows:
\begin{subequations}\label{eq:duallpConstrained}
\begin{align}
\nonumber\max \ \ & \eqref{eq:duallp1Obj} \\[1.0ex]
\nonumber\sthat \ \ & \eqref{eq:duallp1Eqn1} - \eqref{eq:duallp1Eqn4} \\[1.0ex]
\label{eq:duallpConstrained-budget}\ & \sum_{\indTime < \timeHorizon} \sum_{\indGrid \in \setGridIndex} \sum_{\indAction\in \setAction} \sum_{\indState \in \setState} \grPt_{\indState}^{\indGrid} \parfCost_{\indState}^{\indTime\indAction} \varDualLP_{\indTime \indGrid \indAction} \leq \parBudgetLim
\end{align}
\end{subequations}

This LP-based formulation also provides opportunities to incorporate many other problem-specific constraints to CPOMDPs. 
For instance, deterministic policies can be generated by introducing binary decision variables $\varDetPol$ to the formulation and adding the following constraints to \eqref{eq:duallpConstrained}:
\begin{subequations}
\begin{align}
\ & \varDualLP_{\indTime \indGrid \indAction} \leq \varDetPol_{\indTime \indGrid \indAction}, && \indTime < \timeHorizon, \ \ \indGrid\in \setGridIndex, \ \ \indAction\in \setAction,\\[1.4ex]
\ & \sum_{\indAction\in \setAction} \varDetPol_{\indTime \indGrid \indAction} = 1, \quad && \indTime < \timeHorizon, \ \ \indGrid\in \setGridIndex,\\[1.4ex]
\ & \varDetPol_{\indTime \indGrid \indAction} \in \{0,1\} && \indTime < \timeHorizon, \ \ \indGrid\in \setGridIndex, \ \ \indAction\in \setAction.
\end{align}
\end{subequations}
Similarly, additional constraints can be defined to ensure that the resulting policies are of threshold-type, that is, the policies are defined based on the threshold levels over the belief states~\citep{puterman2014markov}. 
However, designing such threshold-type policies often requires domain expertise, and might not be applicable to all the CPOMDP applications.
Also note that introducing binary variables to the model leads to a mixed-integer linear programming (MIP) model, which is significantly more difficult to solve than its LP counterpart. Furthermore, the addition of deterministic policy constraints often results in a reduction in the expected total reward. 
Section~\ref{sec:results-deterministic-policy} provides an investigation into the impact of these additional constraints on both run time performance and the value of the resulting policies. 
Additionally, a representative analysis with threshold-type policies is provided in Appendix~\ref{sec:appendix_threshold}.

\section{Infinite Horizon CPOMDPs}
\label{cpomdp:sec:infinite}
Discrete-time infinite horizon POMDPs are defined similarly to finite horizon POMDPs.
The most important distinction for infinite horizon POMDPs is that $\timeHorizon = \infty$, which implies the transition probabilities, observation probabilities, and rewards are time-invariant.
Accordingly, the same notation is used for infinite horizon POMDPs, except that time index $\indTime$ is excluded from the notation and mathematical formulas.
The Bellman optimality equations for infinite horizon POMDPs are as follows:
\begin{equation}
\funcOptV^{\indAction}(\vectBePt) =  \sum_{\indState \in \setState} \bePt_{\indState} \parfRew_{\indState}^{\indAction}  + \parDiscount\sum_{\indState \in \setState}\bePt_{\indState} \bigg( \sum_{\indObs \in \setObs} \sum_{\indStateNew \in \setState} \parfObs_{\indStateNew \indObs}^{\indAction} \ \parfTrp_{\indState \indStateNew}^{\indAction}\     \funcOptV^*(\vectBePt')\bigg), \quad \indAction \in \setAction, \ \vectBePt \in \beSimp
\end{equation}
Note that, unlike in the finite horizon case, a discount factor  $\parDiscount\in[0, 1)$ must be specified to ensure that the value functions converge.

The same grid-based approximation scheme is used for the infinite horizon problem to reduce the POMDP model to an MDP model.
Accordingly, the infinite horizon counterpart of \eqref{eq:constrainedPOMDPModel} can be approximated by using the following LP model:
\begin{subequations}\label{eq:duallpConstrainedInfinite}
\begin{align}
 \max \ \ & \sum_{\indAction\in \setAction}\sum_{\indGrid \in \setGridIndex}\sum_{\indState \in \setState} \grPt_{\indState}^{\indGrid} \parfRew_{\indState}^{\indAction} \varDualLP_{\indGrid \indAction}\label{eq:infinite-lp-obj}\\[1.0ex]
\sthat \ \ & \sum_{\indAction \in \setAction}\varDualLP_{\indGrid \indAction} - \parDiscount \sum_{\indAction \in \setAction} \sum_{\indGridNew \in \setGridIndex} \parfGridTr_{\indGridNew \indGrid}^{\indAction} \varDualLP_{\indGridNew \indAction} = \parBeDistn_{\indGrid} && \hspace{0.01cm}  \indGrid\in \setGridIndex\label{eq:infinite-lp-dist-constraint}\\[1.0ex]
\ & \sum_{\indGrid \in \setGridIndex} \sum_{\indAction\in \setAction} \sum_{\indState \in \setState} \grPt_{\indState}^{\indGrid} \parfCost_{\indState}^{\indAction} \varDualLP_{\indGrid \indAction} \leq \parBudgetLim\label{eq:infinite-lp-budget-constraint}\\[1.0ex]
\ & \varDualLP_{\indGrid \indAction} \geq 0 && \hspace{0.01cm} \indAction \in \setAction, \  \indGrid \in \setGridIndex\label{eq:infinite-lp-x-constraint}
\end{align}
\end{subequations}
In \eqref{eq:duallpConstrainedInfinite}, the decision variables, $\varDualLP_{\indGrid \indAction}$, are the occupancy measures for grid point-action pairs. These variables correspond to the fraction of time $\vectGrPt^{\indGrid} \in \setGrid$ is visited and action $\indAction \in \setAction$ is taken in the long run.
Note that the transition probabilities between the grid points $\parfGridTr_{\indGridNew \indGrid}^{\indAction}$ can be obtained similar to \eqref{eq:gridTransitionProbs}.
Time-invariant beta values for this calculation can be pre-calculated by using a value iteration algorithm over the grid-based approximation for the unconstrained POMDP model.
Specifically, the beta values $\beta_{\indGrid \indGridNew}^{\indAction \indObs}$ calculated in the last iteration of the value iteration algorithm can be used to obtain $\parfGridTr_{\indGridNew \indGrid}^{\indAction}$ values.
Additional details on the calculation of grid transition probabilities can be found in Appendix~\ref{sec:appendix_trp}.

\section{Numerical Experiments}
\label{cpomdp:numerical_experiments}
This section details numerical analysis results.
First, the experimental setup and problem instances are summarized.
This is followed by an examination of the performance of grid-based approximations for unconstrained POMDPs.
Then, the performance of the proposed CPOMDP solution algorithm, iterative transition-based linear programming (ITLP), is discussed. 
This analysis includes a performance comparison of ITLP in an infinite horizon setting with~\citet{poupart2015}'s CALP algorithm. 
The section is concluded with a detailed analysis of the effects of deterministic policy constraints on both CPU run time and the collected rewards.

\subsection{Experimental setup}
\label{cpomdp:sec:experiment_models}

Experiments in this work focus on extensions of the six popular POMDP problem instances listed in Table~\ref{cpomdp:table:experiment_models}. Specifically, the formulation of each problem has been extended to include both a budget constraint and costs.
Below, details of each of the studied problems including their CPOMDP extensions are provided.

\setlength{\tabcolsep}{6pt} 
\renewcommand{\arraystretch}{1.1}
\begin{table}[!ht]
    \begin{center}
    \caption{Specifications of the POMDP problem instances
    }
    \label{cpomdp:table:experiment_models}
    \resizebox*{0.6\textwidth}{!}{
        \begin{tabular}{lrrrl}
        \toprule
        Problem Name &
        \textbf{$|\setState|$} & 
        \textbf{$|\setAction|$} & 
        \textbf{$|\setObs|$} &
        Reference \\
        \midrule 
        tiger &  2 & 3 & 2 & \citet{cassandraActingOptimallyPartially1994} \\
        mcc &  4 & 3 & 3 & \citet{CassandraExamples}\\
        paint &  4 & 4 & 2 & \citet{CassandraExamples}\\
        query &  9 & 2 & 3 & \citet{CassandraExamples}\\
        $4 \times 3$ maze &   11 & 4 & 6 & \citet{parrApproximatingOptimalPolicies}\\
        rocksample & 129 & 8 & 2 & \citet{smith2012heuristic} \\
        \bottomrule 
        \end{tabular}
    }
    \end{center}
\end{table}

\begin{itemize}\setlength\itemsep{0.5em}
    \item \textit{tiger}: 
    In this problem, the agent is placed in front of two doors. One door has a tiger behind it and incurs a large negative reward if it is opened, while opening the other door leads to a large positive reward.
    The problem encodes this information with the two states: tiger-left and tiger-right.
    Three actions are available to the agent: listen, open-left, and open-right. 
    The listen action incurs a reward of $-1$ and leads to an observation of either tiger-left or tiger-right, which can be used to reduce the uncertainty in the belief state. 
    Opening the doors with and without the tiger behind them leads to rewards of $-100$ and $+10$, respectively.
    Note that the listen action helps reduce uncertainty in the belief state by observing with relatively high accuracy if there is a tiger behind each door. 
    In practice, it is logical to listen until the belief state is clearly biased towards one of the doors, indicating the tiger's location. 
    Accordingly, in the constrained model, restrictions are introduced on the number of listen actions that the agent can take before it chooses to open one of the doors. 
    Specifically, the listen action is assigned a cost of 2, while the two door opening actions have a cost of 1. A budget is then chosen which adequately restricts the agent's decisions.
    For instance, in a finite horizon problem with $\timeHorizon=5$ (i.e., four decision epochs), the agent can take at most 1 listen action when $\parBudgetLim=5$.
    
    \item \textit{paint}: The paint problem involves a part painting task, where a factory automation system aims to ship unblemished and unflawed items while rejecting blemished and flawed ones. 
    There are four states that contain aggregate information about the environment: NFL-NBL-NPA, NFL-NBL-PA, FL-NBL-PA, FL-BL-NPA, where FL/NFL is flawed/unflawed, BL/NBL is blemished/unblemished, PA/NPA is painted/unpainted. 
    Given a part, the automation system can choose to paint, inspect, ship, or reject. 
    If the system ships an unblemished, unflawed, painted part, which corresponds to state NFL-NBL-PA, it receives a reward of $+1$; in any other state, the ship action incurs a reward of $-1$. The system also receives a reward of $+1$ if it rejects a flawed, blemished, unpainted part. Rejecting a flawed, unblemished, painted part yields no reward, and rejecting in any other state incurs a negative reward of $-1$. The paint and inspect actions yield no reward regardless of the state that they are taken in.
    Similar to the tiger model, the inspect action can be used to reduce uncertainty in the belief state. 
    Accordingly, the number of inspect actions is limited by setting the cost of inspecting to 2, while the other actions incur a cost of 1.
    
    \item \textit{mcc}: 
    The mcc problem involves an e-commerce website that aims to show suitable ads to target customers.
    Customers looking to purchase a product can be categorized into one of two states: customers who prefer group 1 items (S1) or customers who prefer group 2 items (S2). Additionally, a customer may be in the process of completing a purchase (SB) or exiting the website (SX).
    The system can take the following actions: show items from group 1 (A1), show items from group 2 (A2), or show a random mix of items (AN). 
    Showing items from the wrong group increases the risk that a customer will leave the website while showing the correct items increases the probability of a purchase. 
    Observations O1 and O2, received after each action, indicate that the user is in S1 or S2, respectively. An observation of OX indicates that the user has left the website (i.e., transitioned to state SX). 
    If the user buys a product from the website (i.e., reaches state SB), then a reward of +1 is collected.
    In this problem, showing a mix of items (i.e., action AN) can be seen as a last resort when the underlying belief state does not favor a particular core state. 
    In this regard, it can be treated as an exploration action similar to the inspect/listen actions in the previous two problems.
    Accordingly, the use of this exploration action is limited by setting the cost of AN to 20 and the costs for all other actions to 10.
    
    \item \textit{query}: 
    The query problem involves a query system with two servers using an aggregate state definition. 
    Each server state can be described as unloaded, loaded, or down. In total, there are nine states for a two-server system.
    The agent can query one of the servers (i.e., there are two actions) to learn about its state.
    The query action leads to one of three observations: no-response, slow-response, and fast-response.
    The rewards are also linked to the observations, with a reward of 0 for no-response, a reward of 3 for slow-response, and a reward of 10 for fast-response.
    Unlike the previous problems, the query POMDP formulation does not have an exploration action. For the purposes of imposing budgetary limitations, one can imagine a scenario where querying the first server would be more expensive than querying the second, either through bandwidth costs, subscription fees, or some other means. As such, the cost of querying servers one and two is 2 and 1, respectively.
    
    \item \textit{$4\times 3$ maze}: 
    The $4\times 3$ maze problem uses the grid world with 11 cells to model an agent trying to reach a designated positive cell in the grid while avoiding a specific negative cell. 
    The agent can take one of the four actions to move between the cells: north, south, east, and west. 
    The observations correspond to the number of walls that the agent observes for the current cell that it occupies.
    Moreover, the agent can detect the positive cell and the negative cell once it occupies those cells.
    When the positive cell is reached, a reward of $+1$ is collected, whereas reaching the negative cell leads to a reward of $-1$.
    Transitioning to any other cells (i.e., states) incurs a small negative reward of $-0.04$.
    In the CPOMDP formulation of the $4\times 3$ maze problem, the cost of taking an action in a single designated cell is set to 1, and the cost is 0 in all other cells. The agent must then navigate the maze while limiting the number of times it passes through this particular cell. Notice that, unlike in the previous CPOMDP formulations, the cost of taking an action in the $4\times 3$ problem is state dependent.
    
    \item \textit{rocksample}: 
    In this problem, a robot is tasked with locating rocks with significant scientific value and sampling them accordingly.
    We consider a square grid with a side length of $4$ containing $3$ rocks.
    Each rock can be in one of two states: \textit{good} or \textit{bad}. 
    The robot is rewarded when it samples a good rock, but it is penalized if it samples a bad rock. 
    To detect the quality of rocks, the robot is equipped with a sensor, the measurement quality of which deteriorates exponentially according to the Euclidean distance between the robot and the rock.
    The robot can move in any of the cardinal directions, it can sample the rock that it is standing on, and it can check each of the rocks using its sensor, for a total of 8 actions. 
    The robot can also occupy any of the squares on the grid, of which there are 16. Each of the 3 rocks is either \textit{good} or \textit{bad}, thus the total number of combined states of the rocks is $2^3=8$. 
    There is a terminal state, which is reached by moving east off of the right side of the grid, hence the total number of states is $16\times 8 + 1=129$.
    Each time the robot takes a step (except into the terminal state), it receives a small penalty of $-1$. 
    Stepping into the terminal state yields a reward of $+5$ and, the belief is set to the \textit{target belief} in the next epoch. 
    If the robot samples a good rock, it is awarded $+10$ and the state of the rock becomes bad. 
    If instead, it samples a bad rock, it is penalized by $-10$. 
    If the robot samples from a square with no rock on it, it is penalized by $-1$. Check actions do not incur a negative reward, but as they are used to reduce uncertainty in the belief, a cost of 1 is assigned to checking any of the rocks. All other actions incur a cost of $0.5$.
    
\end{itemize}

Tables~\ref{cpomdp:tab:finite-budgets} and \ref{cpomdp:tab:inf-budgets} show the allocated budgets for the six CPOMDP problem instances in both their finite horizon and infinite horizon formulations.
Table~\ref{cpomdp:tab:finite-budgets} also shows the decision horizon $\timeHorizon$ for each problem instance.
In the finite horizon case, since query, $4\times 3$, and rocksample are larger, more difficult problem instances, they are solved for a horizon length of $\timeHorizon=5$. 
All the other problems are solved for $\timeHorizon=20$.
Three budget levels are considered: small, medium and large. These budgets help us to examine how the CPOMDP algorithm responds to varying cost constraints.
Each selected budget allows for feasible policies to exist and was selected to adequately constrain the model.

\setlength{\tabcolsep}{4.5pt} 
\renewcommand{\arraystretch}{0.91}
\begin{table}[!ht]
\centering
\caption{CPOMDP budget limits for each problem instance}
\label{cpomdp:tab:CPOMDP-budgets}
    \subfloat[Finite horizon \label{cpomdp:tab:finite-budgets}]{
    \scalebox{0.905}{
    \begin{tabular}{lrrr}
    \toprule
    Problem & small & medium & large \\
    \midrule
    tiger ($\timeHorizon = 20$) & 21.00 & 25.00 & 50.00 \\
    paint ($\timeHorizon = 20$) & 19.50 & 22.00 & 50.00 \\
    mcc ($\timeHorizon = 20$) & 118.00 & 123.00 & 140.00 \\
    query  ($\timeHorizon = 5$) & 4.16 & 4.47 & 7.00\\
    $4\times 3$  ($\timeHorizon = 5$) & 0.17 & 0.25 & 0.60 \\
    rocksample ($\timeHorizon=5$) & 2.00 & 2.10 & 2.40 \\ 
    \bottomrule
    \end{tabular}
    }
    }
    \subfloat[Infinite horizon \label{cpomdp:tab:inf-budgets}] 
    {
    \scalebox{0.905}{
    \begin{tabular}{lrrr}
    \toprule
    Problem & small & medium & large \\
    \midrule
    tiger & 11.50 & 14.00 & 22.00 \\
    paint & 10.50 & 12.00 & 18.00 \\
    mcc & 63.70 & 66.00 & 75.00 \\
    query & 10.32 & 11.58 & 16.00 \\
    $4\times 3$ & 0.40 & 0.45 & 2.00 \\
    rocksample & 5.04 & 5.64 & 6.00 \\ 
    \bottomrule
    \end{tabular}
    }
    }
\end{table}

For the $4\times 3$ problem, the initial (target) belief state $\stBelief$ is taken as a uniform distribution over all of the states except for the positive and negative cells that represent terminal states, that is, in these two cells, the initial probability is zero. 
For the rocksample problem, $\stBelief$ is specified as a uniform distribution over all of the states where the starting position is taken as the cell $(1, 3)$. 
As there are three rocks, this gives 8 equally likely starting states. 
For all other problems, $\stBelief$ is a uniform distribution over all of the states. 
For example, in the paint problem, $\stBelief=[0.25, 0.25, 0.25, 0.25]$. 
A discount factor ($\parDiscount$) of 1.0 is used for finite horizon problems, and 0.9 for infinite horizon problems.
Our preliminary analysis revealed that, for most problem instances, high-quality solutions can be achieved with 200 grid points. 
As the rocksample problem is substantially larger than each of the other problem instances, a grid set with 1000 grid points was used for this problem instead. 
The subsequent experiments also explore the impact of different grid set sizes
on a representative problem instance.
All the experiments are run on a PC with an i7-8700K processor and 32GB of RAM.
Implementations are done in python, and the LP models are solved using Gurobi 9.0.

\subsection{Performance of the grid-based approximations}\label{cpomdp:subsec:unconstrained-results}

The numerical analysis begins with an investigation into the efficiency and efficacy of the UB approximation method for unconstrained POMDPs, as the proposed CPOMDP solution approach is built on this method. 
In particular, an overview of the computational performance of the LB and UB methods is provided by calculating the optimality gaps,
and the need for approximation algorithms is demonstrated by identifying three instances in the set of toy problems where exact solution algorithms are intractable.

The summary results provided in Table~\ref{cpomdp:tab:unconstrained-summary} provide a comparison of the LB, UB, and IP methods. 
These results provide the expected value, $\expVal(\stBelief)$, at each model's respective starting belief ($\stBelief$).
Additionally, the performance of the approximations is measured over the remainder of the belief space by constructing an evaluation grid set $\setGridEval$ of size 100, and calculating the LB--UB gap for each $\vectBePt\in \setGridEval{}$ as:
\begin{equation}
    \text{Gap}= 100 \times \frac{\ubV(\vectBePt) - \lbV(\vectBePt)}{\lbV(\vectBePt)}
\end{equation}
The evaluation grid set $\setGridEval$ is constructed by randomly sampling belief states, as described by~\citet{Suresh2005}.

\begin{table}[!ht]
\centering
\caption{Summary results for unconstrained POMDP solution algorithms.}
\resizebox{0.95\textwidth}{!}{
\begin{threeparttable}
\begin{tabular}{lllrrrrrrrrrr}
\toprule
&     &     & \multicolumn{3}{c}{ $\expVal(\stBelief)$ } & \multicolumn{3}{c}{CPU (sec.)} & \multicolumn{4}{c}{Gap (\%)} \\
\cmidrule(lr){4-6} \cmidrule(lr){7-9} \cmidrule(lr){10-13}
&     &     &     LB &     IP &     UB &     LB &       IP &      UB &    min &  mean & median &  max \\
\cmidrule(lr){4-6} \cmidrule(lr){7-9} \cmidrule(lr){10-13}
T & Problem & $|\setGrid{}|$ &        &        &        &        &          &         &        &        &        &\\
\midrule
\multicolumn{13}{c}{\textbf{Finite horizon}} \\ \midrule
20 & tiger & 200 &  20.39 &  20.39 &  20.39 &   2.24 &    57.49 &   18.53 &   0.01 &   0.04 &   0.01 &    0.58 \\
         & paint & 200 &   3.04 &   3.04 &   3.08 &   2.57 &    22.73 &   26.11 &   1.03 &   1.48 &   1.37 &    3.97 \\
         & mcc & 200 &   1.21 &   1.21 &   1.22 &   2.19 &     8.26 &   27.46 &   0.24 &   0.36 &   0.34 &    0.69 \\
5 & query & 50  &  18.56 &  18.56 &  23.83 &   2.60 &     0.62 &    1.28 &  21.59 &  28.17 &  27.86 &   40.11 \\
         &     & 200 &  18.56 &      - &  23.70 &   2.53 &        - &    5.75 &  21.36 &  27.36 &  26.99 &   39.42 \\
         &     & 500 &  18.56 &      - &  23.53 &   2.21 &        - &   23.81 &  21.17 &  26.47 &  26.22 &   38.29 \\
         & $4\times 3$ & 200 &   0.12 &   0.12 &   0.19 &   1.21 &     1.87 &   21.36 &  0.00 &  21.77 &  16.26 &  145.95 \\ 
         & rocksample & 1000 &   8.00 &  4.00\tnote{$\dagger$} &  12.08 &   5.14 &     14400\tnote{$\dagger$} & 105.10 & 0.00 &  43.61 & 0.00 & 134.95 \\         
         \midrule
\multicolumn{13}{c}{\textbf{Infinite horizon}} \\ \midrule
         & tiger & 200 &   8.51 &   8.51 &   8.51 &   2.14 &    97.85 &   96.97 &   0.01 &   0.01 &   0.01 &    0.06 \\
         & paint & 200 &   1.33 &   1.34 &   1.34 &   2.19 &    30.19 &  126.13 &   0.03 &   0.35 &   0.06 &    4.18 \\
         & mcc & 200 &   0.69 &   0.69 &   0.69 &   2.60 &    31.42 &   81.13 &   0.07 &   0.17 &   0.15 &    0.63 \\
         & query & 50  &  48.02 &  19.02\tnote{$\dagger$} &  48.78 &   2.77 &   14400\tnote{$\dagger$} &   10.90 &   0.92 &   1.67 &   1.63 &    2.91 \\
         &     & 200 &  48.02 &      - &  48.52 &   8.03 &        - &   68.52 &   0.56 &   1.07 &   1.03 &    2.30 \\
         &     & 500 &  48.02 &      - &  48.21 &  37.11 &        - &  336.13 &   0.20 &   0.44 &   0.42 &    1.00 \\
         & $4\times 3$ & 200 &   0.75 &   0.33\tnote{$\dagger$} &   0.84 &  15.03 &  14400\tnote{$\dagger$} &  231.18 &   3.96 &  10.25 &   9.89 &   20.21 \\
         & rocksample & 1000 &   13.21 &  3.5\tnote{$\dagger$} &  22.59 &   569.91 &     14400\tnote{$\dagger$} & 7873.83 & 13.75 &  56.17 & 41.51 & 107.46 \\ 
\bottomrule
\end{tabular}
\begin{tablenotes}
    \item[$\dagger$] value as of 4 hour timeout
    \item $\expVal(\stBelief)$ values are reported for target belief states $\stBelief$
\end{tablenotes}
\end{threeparttable}
}
\label{cpomdp:tab:unconstrained-summary}
\end{table}

The LB approximation algorithm estimates $\expVal(\stBelief)$ (i.e., the value for the starting belief state) exceptionally well for all six POMDP problem instances, while UB approximation quality deteriorates as the number of states increases, particularly for the finite horizon problems. 
The gap calculations reveal similar results, where once again the solution quality deteriorates significantly for finite horizon problems with larger state spaces. 
The query problem is taken here as a representative problem instance for investigating the impact of grid set size.
The results show only modest improvements when increasing the number of grid points for both the finite and infinite horizon problems, at the expense of a significant increase in computation time. 
The IP algorithm is unable to solve the infinite horizon $4\times 3$ and query problems within a four-hour time limit, while the approximation algorithms can solve these problems in a relatively short amount of time.
The IP algorithm also fails to find any solution for the rocksample problem despite the short horizon length of $\timeHorizon=5$ in the finite horizon case. 
In contrast, the LB approximation can find an approximate solution in just over five seconds. 
These results highlight the need for high-quality approximation algorithms, as exact solution methods tend to become intractable as the number of states increases.

\subsection{Analysis with CPOMDP methods}
The proposed CPOMDP solution algorithm, ITLP, allows specifying a probability distribution $\parBeDistn$ over $\vectGrPt \in\setGrid{}$. 
Accordingly, the corresponding expected reward and expected cost for such policies depend on both $\setGrid{}$ and $\parBeDistn$ as follows:
\begin{align} 
\expVal(\setGrid, \parBeDistn) &= \sum_{\indGrid\in\setGridIndex} \parBeDistn_{\indGrid} \expVal(\vectGrPt^{\indGrid}), \qquad 
\expCost(\setGrid, \parBeDistn) = \sum_{\indGrid\in\setGridIndex} \parBeDistn_{\indGrid} \expCost(\vectGrPt^{\indGrid})
\end{align}
The approximate values of these quantities are denoted by $\funcApxV(\setGrid{}, \parBeDistn{})$ and $\funcApxC(\setGrid{}, \parBeDistn{})$, and they are calculated using 10,000 simulations where, for each simulation, the initial belief $\stBelief$ is randomly selected using the weights of $\parBeDistn$ corresponding to $\vectGrPt\in\setGrid{}$. The objective value returned from the linear program is also reported, and is denoted by $\LPV(\setGrid, \parBeDistn)$ and, for the starting belief $\stBelief$, $\LPV(\setGrid, \stBelief)$.

As was also observed by \citet{poupart2015}, a CPOMDP policy generated by a grid-based approximation does not necessarily adhere to its budget in a general simulation environment. 
Such a policy is described as infeasible. 
In practice, one searches only for feasible policies and discards any policy that exceeds its budget. 
\citet{poupart2015} propose performing a binary search for a budget constraint which produces a feasible, near-optimal policy. 
However, policy generation and evaluation are time consuming processes, which provides considerable motivation for reducing the likelihood of generating an infeasible policy in the first place. 
In the subsequent analysis, rather than discarding infeasible policies, their $\funcApxV(\setGrid, \parBeDistn)$ are reported, and the degree to which $\funcApxC(\setGrid, \parBeDistn)$ exceeds the budget limit $\parBudgetLim$ is provided using the following formula:
\begin{equation}
    \text{\%-Over} = \begin{cases}
    100\times \dfrac{\funcApxC(\setGrid, \parBeDistn) - \parBudgetLim}{\parBudgetLim},&\text{if } \funcApxC(\setGrid, \parBeDistn) > \parBudgetLim\\
    0,&\text{otherwise}
    \end{cases}
    \label{cpomdp:eq:percent-over}
\end{equation}
By retaining infeasible policies, the effects of model complexity and grid size on budget adherence can be quantified, where Equation~\eqref{cpomdp:eq:percent-over} provides a measure for the degree to which a policy exceeds its budget.

The process of generating CPOMDP policies using ITLP involves two main steps: (1) generating the transition probabilities (i.e., $\parfGridTr_{\indGrid \indGridNew}^{\indTime \indAction}$ values), and (2) setting up and solving the final LP given these transition probabilities. 
The elapsed time for each of these steps is reported separately as ``cpu-trans'' and ``cpu-lp'', respectively.

\subsubsection{Finite horizon CPOMDPs}

Table~\ref{cpomdp:tab:finite-complete-summary} provides the approximate/simulated expected reward $\SimV(\setGrid, \parBeDistn)$ and the \%-Over for each model, with an additional column for the simulated expected reward of IP (i.e., for solving the corresponding unconstrained POMDP model). 
For problems where IP was unable to find an optimal policy within the time limit, the LB approximate policy result is used instead.
Note that IP/LB results are also dependent on $\setGrid$ and $\parBeDistn$ as the same simulation mechanism is employed for all cases.

\begin{table}[!ht]
\centering
\caption{Finite-horizon CPOMDP results ($\parBeDistn$ is taken as a uniform distribution over the corresponding $\setGrid$).}
\label{cpomdp:tab:finite-complete-summary}
\resizebox*{!}{0.995\textwidth}{
\begin{tabular}{llllrrrr}
\toprule
& & & & \multicolumn{3}{c}{Budget level} & \\
\cmidrule(lr){5-7}
Problem & $\timeHorizon$ & $|\setGrid{}|$ & & small & medium & large & IP/LB\\
\midrule
tiger & 20  &  200 & $\SimV(\setGrid, \parBeDistn)$ & -650.53 & -320.78 &  28.89 & 29.10 \\
     &       & & $\LPV(\setGrid, \parBeDistn)$ & -604.37 & -277.29 &  22.27 & - \\
     &       & & \% Over &    0.00 &    0.03 &   0.00 & - \\
     &       & &  cpu-trans &   19.63 &   19.63 &  19.63 & - \\
     &       & &  cpu-lp &    6.15 &    4.38 &   5.15 & - \\ 
\midrule
paint  & 20  &  200 & $\SimV(\setGrid, \parBeDistn)$ &    1.00 &    2.30 &   3.46 & 3.69 \\
     &       & & $\LPV(\setGrid, \parBeDistn)$ &    0.81 &    2.06 &   3.41 & - \\
     &       & & \% Over &    0.60 &    0.77 &   0.00 & - \\
     &       & &  cpu-trans &   28.32 &   28.32 &  28.32 & - \\
     &       & &  cpu-lp &   41.24 &   26.55 &  27.17 & - \\ 
\midrule
mcc  & 20  &  200 & $\SimV(\setGrid, \parBeDistn)$ &    0.71 &    0.93 &   1.19 & 1.22\\
     &       & & $\LPV(\setGrid, \parBeDistn)$ &    0.70 &    0.94 &   1.22 & - \\
     &       & & \% Over &    0.33 &    0.03 &   0.00 & - \\
     &       & &  cpu-trans &   31.82 &   31.82 &  31.82 & - \\
     &       & &  cpu-lp &    8.67 &    6.74 &   6.96 & - \\ 
\midrule
query & 5  &  50 & $\SimV(\setGrid, \parBeDistn)$ &   21.08 &   22.10 &  24.18 & 24.60\\
     &       & & $\LPV(\setGrid, \parBeDistn)$ &   20.77 &   22.25 &  24.75 & - \\
     &       & & \% Over &    0.17 &    0.34 &   0.00 & - \\
     &       & &  cpu-trans &    0.66 &    0.66 &   0.66 & - \\
     &       & &  cpu-lp &    0.46 &    0.46 &   0.46 & - \\
\cmidrule(lr){3-8}
     & & 200 & $\SimV(\setGrid, \parBeDistn)$ &   20.94 &   22.10 &  24.43 & 24.80\\
     &       & & $\LPV(\setGrid, \parBeDistn)$ &   20.88 &   22.20 &  24.64 & - \\
     &       & & \% Over &    0.00 &    0.00 &   0.00 & - \\
     &       & &  cpu-trans &    5.35 &    5.35 &   5.35 & - \\
     &       & &  cpu-lp &    2.36 &    2.35 &   2.35 & - \\
\cmidrule(lr){3-8}
     & & 500 & $\SimV(\setGrid, \parBeDistn)$  &  20.51 &   21.58 &  24.40 & 24.46\\
     &       & & $\LPV(\setGrid, \parBeDistn)$ &   20.52 &   21.89 &  24.39 & - \\
     &       & & \% Over &    0.00 &    0.00 &   0.00 & - \\
     &       & &  cpu-trans &   27.26 &   27.26 &  27.26 & - \\
     &       & &  cpu-lp &   11.96 &   11.92 &  11.90 & - \\ 
\midrule
$4\times 3$ & 5  &  200 & $\SimV(\setGrid, \parBeDistn)$ &    0.13 &    0.19 &   0.22 & 0.27\\
     &       & & $\LPV(\setGrid, \parBeDistn)$ &    0.18 &    0.25 &   0.28 & - \\
     &       & & \% Over &   17.07 &    1.98 &   0.00 & - \\
     &       & &  cpu-trans &   22.38 &   22.38 &  22.38 & - \\
     &       & &  cpu-lp &    3.17 &    3.14 &   3.14 & - \\
\midrule
rocksample & 5  &  1000 & $\SimV(\setGrid, \parBeDistn)$ &   6.53 &    6.72 &   6.93 & 9.02\\ 
     &       & & $\LPV(\setGrid, \parBeDistn)$ &    10.81 &   11.11 &   11.26 & - \\
     &       & & \% Over &   3.07 &    3.42 &   0.00 & - \\
     &       & &  cpu-trans &  303.10 &    303.10 &  303.10 & - \\
     &       & &  cpu-lp &    289.48 &    287.45 &   287.42 & - \\
\bottomrule
\end{tabular}
}
\end{table}

Without access to optimal CPOMDP policies, it is not possible to evaluate the performance of ITLP by considering each budget level individually. 
Instead, policy feasibility and the change in $\SimV(\setGrid, \parBeDistn)$ given a change in the budget limits are considered as evaluation metrics. 
For all but the $4\times 3$ and rocksample problems, the generated policies adhere to the budget or exceed it by less than 1\%. Table~\ref{cpomdp:tab:finite-complete-summary} establishes a clear positive relationship between the budget level $\parBudgetLim$ and $\SimV(\setGrid, \parBeDistn)$, with the expected total reward for a large budget policy approaching that of an optimal, unconstrained policy's reward (IP). 
Table~\ref{cpomdp:tab:finite-complete-summary} also provides the run times for each problem, which shows that for all but the paint problem, the most time consuming step of ITLP is generating the transition probabilities, with the LP taking a smaller (but not insignificant) amount of time.

While the rocksample problem is considerably larger than each of the other problem instances, the results in Table~\ref{cpomdp:tab:finite-complete-summary} demonstrate that problems of this size are well within the capabilities of the ITLP algorithm. 
Specifically, policies can be generated within a reasonable amount of time; budget adherence is similar to that of the $4\times 3$ problem, a much smaller problem instance; and simulated rewards increase as the budget increases. 
The runtime results suggest that one could also increase the number of grids used to achieve a higher quality approximation, solve problems with many more states and actions than the rocksample problem, and increase the number of decision epochs in the problem formulation while using the ITLP algorithm.

In addition to providing a comparison of ITLP for varying grid sizes, Figure~\ref{cpomdp:fig:finite-constrained-query-s2-graph} provides a more granular depiction of how budget adherence and $\SimV(\setGrid, \parBeDistn)$ depend on the budget level, $\parBudgetLim$, for a representative problem instance, query. 
Figure~\ref{cpomdp:fig:query-finite-values} shows that $\SimV(\setGrid, \parBeDistn)$ increases with $\parBudgetLim$ until some threshold, at which point further increases to the budget limit do not lead to an increase in $\SimV(\setGrid, \parBeDistn)$. 
This behaviour is expected, as increasing the budget corresponds to an increased reward until there is a sufficient budget available to produce an unconstrained policy, at which point the expected reward cannot increase. 
This is also reflected in the higher budget limits in Figure~\ref{cpomdp:fig:query-finite-costs} where, for a sufficiently high budget, \%-Over never exceeds zero.

\begin{figure}[!ht]
\centering
\subfloat[Simulated value ($\SimV(\setGrid, \parBeDistn)$).\label{cpomdp:fig:query-finite-values}]{\includegraphics[width=0.490\textwidth]{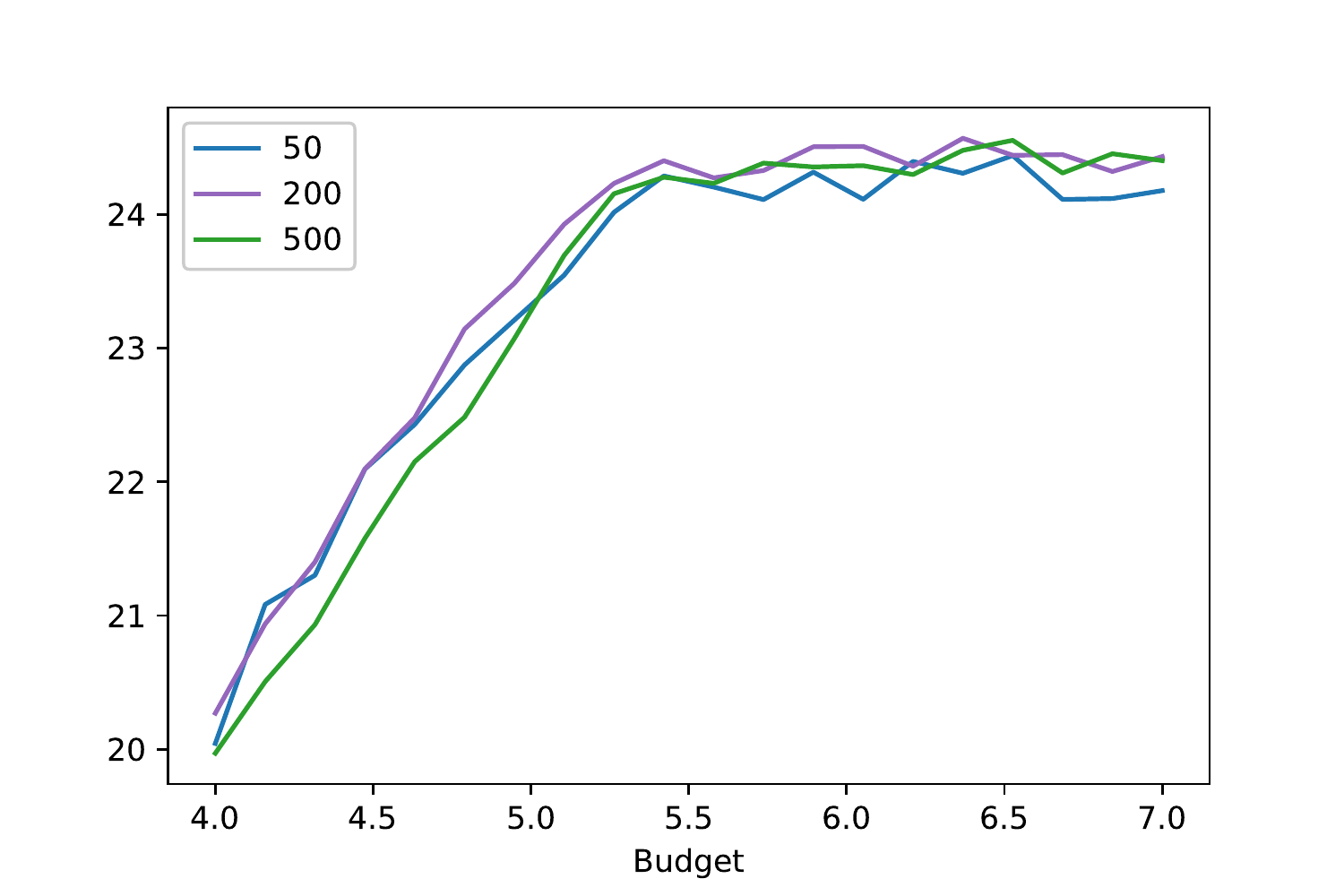}}
\hfill
\subfloat[Simulated cost ($\funcApxC(\setGrid, \parBeDistn)$).\label{cpomdp:fig:query-finite-costs}]{\includegraphics[width=0.490\textwidth]{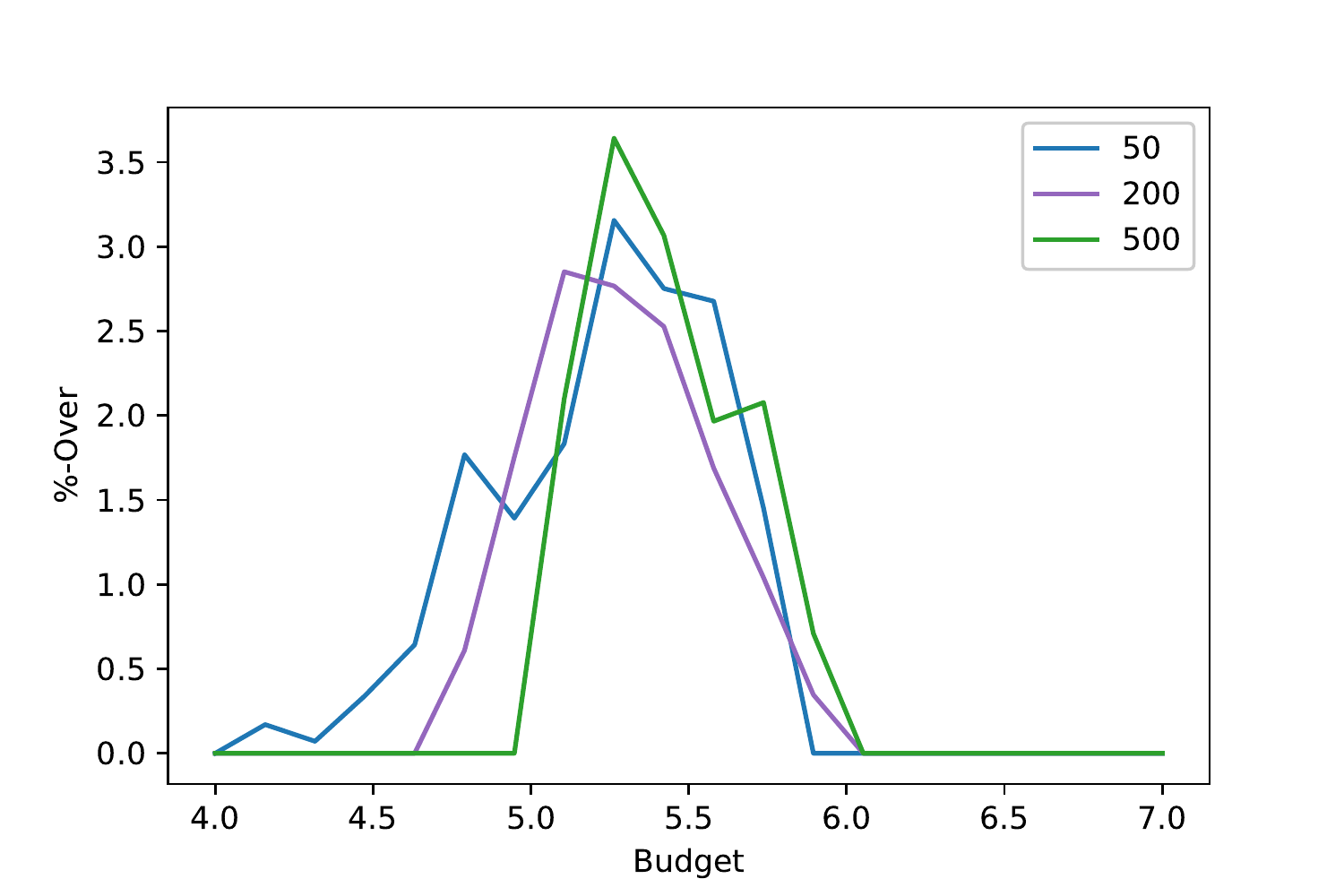}}
\caption{Detailed finite-horizon CPOMDP results for query problem for different grid set sizes as obtained by simulating the resulting policies.}
\label{cpomdp:fig:finite-constrained-query-s2-graph}
\end{figure}

Figure~\ref{cpomdp:fig:query-finite-costs} shows that using larger grid sets increases the likelihood of generating a feasible policy.
These results demonstrate that increasing the number of grid points does not consistently lead to a lower or higher expected total reward. 
In particular, in the budget range of 4.00 to 4.63, the model with 200 grid points consistently produces feasible policies that outperform their 500 grid point counterparts. 
It turns out that the policies obtained with a grid set of size 500 tend to be more conservative with the budget constraint, which allows the model with 200 grid points to produce policies that better utilize the available budget. 
For larger budgets, models with 200 and 500 grid points are shown to produce a higher reward than the model with 50 grid points. 
The results also show that the CPU run times increase considerably as the number of grid points increases (see Table~\ref{cpomdp:tab:finite-complete-summary}).

Producing a feasible policy can be a time consuming task. \citet{poupart2015} suggest performing a binary search over budgets for a feasible policy with an acceptable reward, and Figure~\ref{cpomdp:fig:query-finite-costs} suggests that increasing the grid size can also help with generating feasible policies. 
Both of these processes can add considerable computational overhead and, in the case of increasing grid size, the resulting policy may not fully utilize the available budget. 
For particularly large problems, increasing the grid size may not be feasible and performing multiple runs may be costly, which motivates choosing a more conservative budget to increase the likelihood of producing a feasible policy in one shot. 
On the other hand, Figure~\ref{cpomdp:fig:query-finite-values} demonstrates that even small changes to the budget tend to impact $\SimV(\setGrid, \parBeDistn)$. As a result, a conservative budget choice may result in a considerable decrease in reward.

\subsubsection{Infinite horizon CPOMDPs}
As in the finite horizon case, the expected reward and the expected cost are approximated using 10,000 simulations which employ a horizon of 100 to simulate an infinite horizon. 
The infinite horizon version of ITLP is compared to \citet{poupart2015}'s CALP algorithm. 
As CALP policies assume a single starting belief $\stBelief$, the corresponding ITLP belief distribution $\parBeDistn$ is given by:
\begin{equation}
\parBeDistn(\vectBePt)=\begin{cases}
1,&\text{if }\vectBePt=\stBelief\\
0,&\text{otherwise}
\end{cases}
\end{equation}
Thus, $\funcApxV(\setGrid, \parBeDistn)$ becomes $\funcApxV(\stBelief)$, and $\funcApxC(\setGrid, \parBeDistn)$ becomes $\funcApxC(\stBelief)$. Any two policies modeling the same problem and sharing the same grid size use the same $\setGrid$. Under these conditions, ITLP policies can be directly compared to CALP policies. In Table~\ref{cpomdp:tab:infinite-constrained-summary}, the policies generated using each algorithm are compared for selected small, medium, and large budgets, and $\SimV(\stBelief)$ is compared with the optimal unconstrained $\expVal(\stBelief)$. 
Because incremental pruning could not generate infinite horizon unconstrained POMDP policies for the query, $4\times 3$, and rocksample problems within the time limit (see Table~\ref{cpomdp:tab:unconstrained-summary}) the reported $\funcOptV(\stBelief)$ under IP/LB column are generated using the LB method for these three problems. 

Similar to finite-horizon case, the ensuing qualitative analysis focuses on policy feasibility and changes in $\funcApxV(\stBelief)$ with respect to $\parBudgetLim$. 
The results in Table~\ref{cpomdp:tab:infinite-constrained-summary} show no clear winner between the two algorithms: ITLP performs worse for the simpler toy problems, but performs significantly better for the $4\times 3$ problem. 
For the tiger problem, ITLP produces policies which grossly exceed the budget and show little variation in $\funcApxV(\stBelief)$ between the small and medium budget. In contrast, CALP produces three feasible policies with a significant difference in $\funcApxV(\stBelief)$ between all the budgets. 
In the paint problem, ITLP produces policies with nearly identical $\funcApxV(\stBelief)$ regardless of the budget, once again grossly exceeding the constraint. 
CALP performs considerably better, producing a much greater variation in $\funcApxV(\stBelief)$ and exceeding the budget by much less, but still produces infeasible policies for the small and medium budget cases.
\begin{table}[!t]
\centering
\caption{Infinite horizon CPOMDP results.}
\label{cpomdp:tab:infinite-constrained-summary}
\resizebox*{!}{0.995\textwidth}{
\begin{tabular}{lllrrrrrrr}
\toprule
& & & \multicolumn{3}{c}{Budget level -- ITLP} & \multicolumn{3}{c}{Budget level -- CALP} & \\
\cmidrule(lr){4-6}\cmidrule(lr){7-9}
Problem & $|\setGrid{}|$ & & small & medium & large & small & medium & large & IP/LB\\
\midrule
tiger & 200 & $\SimV(\stBelief)$ &  -38.99 &  -38.88 &   10.98 & -328.96 & -131.77 &  10.75 & 11.06 \\
    &     & $\LPV(\stBelief)$ & -291.91 & -118.46 &    8.51 & -332.02 & -135.40 &   6.13 & - \\
    &     & \% Over &   28.14 &    5.26 &    0.00 &    0.00 &    0.00 &   0.00 & - \\
    &     & cpu-trans &   58.99 &   58.99 &   58.99 &    1.21 &    1.21 &   1.21 & - \\
    &     & cpu-lp &    0.18 &    0.18 &    0.18 &    0.08 &    0.08 &   0.08 & - \\ 
\midrule
paint & 200 & $\SimV(\stBelief)$ &    1.58 &    1.61 &    1.65 &    0.05 &    0.95 &   1.62 & 1.69\\
    &     & $\LPV(\stBelief)$ &    0.37 &    0.88 &    1.34 &    0.21 &    0.83 &   1.37 & - \\
    &     & \% Over &   27.32 &   11.69 &    0.00 &    0.77 &    2.54 &   0.00 & - \\
    &     & cpu-trans &   50.26 &   50.26 &   50.26 &    1.47 &    1.47 &   1.47 & - \\
    &     & cpu-lp &    0.25 &    0.24 &    0.24 &    0.12 &    0.12 &   0.12 & - \\ 
\midrule
mcc & 200 & $\SimV(\stBelief)$ &    0.37 &    0.52 &    0.52 &    0.51 &    0.52 &   0.52 & 0.52 \\
    &     & $\LPV(\stBelief)$ &    0.56 &    0.66 &    0.69 &    0.52 &    0.52 &   0.52 & - \\
    &     & \% Over &    4.29 &    5.25 &    0.00 &   20.39 &   18.14 &  11.80 & - \\
    &     & cpu-trans &   38.64 &   38.64 &   38.64 &    1.67 &    1.67 &   1.67 & - \\
    &     & cpu-lp &    0.25 &    0.25 &    0.25 &    0.08 &    0.08 &   0.08 & - \\ 
\midrule
query & 50  & $\SimV(\stBelief)$ &   34.92 &   40.02 &   47.45 &   33.74 &   41.98 &  47.11 & 49.37 \\
    &     & $\LPV(\stBelief)$ &   36.66 &   42.97 &   48.75 &   42.43 &   45.75 &  49.11 & - \\
    &     & \% Over &    0.00 &    0.01 &    0.00 &    0.00 &    4.71 &   0.00 & - \\
    &     & cpu-trans &    7.83 &    7.83 &    7.83 &    0.31 &    0.31 &   0.31 & - \\
    &     & cpu-lp &    0.08 &    0.08 &    0.08 &    0.02 &    0.02 &   0.02 & - \\
\cmidrule(lr){2-10}
    & 200 & $\SimV(\stBelief)$ &   35.25 &   41.59 &   47.94 &   35.01 &   42.35 &  47.91 & 49.17 \\
    &     & $\LPV(\stBelief)$ &   36.55 &   42.22 &   48.51 &   41.52 &   45.21 &  48.69 & - \\
    &     & \% Over &    0.99 &    3.36 &    0.00 &    0.42 &    5.48 &   0.00 & - \\
    &     & cpu-trans &   81.71 &   81.71 &   81.71 &    1.91 &    1.91 &   1.91 & - \\
    &     & cpu-lp &    0.41 &    0.41 &    0.41 &    0.15 &    0.15 &   0.15 & - \\
\cmidrule(lr){2-10}
    & 500 & $\SimV(\stBelief)$ &   33.94 &   41.31 &   48.33 &   34.18 &   42.88 &  48.46 & 49.22 \\
    &     & $\LPV(\stBelief)$ &   36.53 &   42.24 &   48.24 &   41.46 &   44.98 &  48.80 & - \\
    &     & \% Over &    0.00 &    2.19 &    0.00 &    0.00 &    6.01 &   0.00 & - \\
    &     & cpu-trans &  421.66 &  421.66 &  421.66 &    8.56 &    8.56 &   8.56 & - \\
    &     & cpu-lp &    1.42 &    1.39 &    1.39 &    2.09 &    2.08 &   2.08 & - \\ 
\midrule 
$4\times 3$ & 200 & $\SimV(\stBelief)$ &    0.68 &    0.90 &    1.14 &   -0.14 &   -0.14 &  -0.13 & 1.16\\
    &     & $\LPV(\stBelief)$ &    0.39 &    0.48 &    0.81 &    0.24 &    0.26 &   0.38 & - \\
    &     & \% Over &    0.00 &    2.32 &    0.00 &    0.00 &    0.00 &   0.00 & - \\
    &     & cpu-trans &  155.41 &  155.41 &  155.41 &    9.22 &    9.22 &   9.22 & - \\
    &     & cpu-lp &    0.32 &    0.32 &    0.32 &    0.13 &    0.12 &   0.12 & - \\
\midrule 
rocksample & 1000 & $\SimV(\stBelief)$ &   3.09 &    3.73 &    5.90 &   3.02 &   4.14 &  5.69 & 15.32\\
    &     & $\LPV(\stBelief)$ &    20.57 &    22.10 &    22.38 &   20.57 &  22.17 &   22.44 & - \\
    &     & \% Over &    0.22 &    1.81 &    0.00 &   0.23 &    2.33 &   0.00 & - \\
    &     & cpu-trans &  4609.18 &  4609.18 &  4609.18 &   104.19 &    104.19 &  104.19 & - \\
    &     & cpu-lp &    7.01 &    7.01 &    7.01 &    3.26 &   3.03 &  3.03 & - \\
\bottomrule
\end{tabular}
}
\end{table}

CALP performs poorly for the mcc problem, producing policies which all significantly exceed the budget and all yield a nearly identical reward regardless of budget. ITLP fails to produce feasible policies for the small and medium budget cases, but shows a considerably higher change in $\funcApxV(\stBelief)$ between the small and medium budgets and, in the large budget case, yields a reward identical to the unconstrained optimal policy while coming in under budget. CALP is also able to produce a policy which yields the same reward as an optimal unconstrained policy, but it does so by exceeding the budget by nearly 12\%.

Unlike in the finite horizon case, ITLP produces policies for the $4\times 3$ problem which perform exceptionally well, only exceeding the medium budget. 
There is also a significant increase in $\funcApxV(\stBelief)$ with increasing budget. In contrast, CALP only generates excessively conservative policies, all of which are feasible but none of which comes close to $\funcApxV(\stBelief)$ generated by ITLP for the low budget case.

The CPU run time values in Table~\ref{cpomdp:tab:infinite-constrained-summary} suggest that, for infinite horizon problems, problems of similar size to the rocksample problem are approaching the size limit that the ITLP algorithm can solve within a reasonable amount of time. 
The ITLP algorithm yields a greater reward than CALP in both the small and large budget cases and adheres to the budget better than CALP in the small and medium budget cases as well. 
Additionally, the reward values increase as the budget is increased.
Each of these results suggests that the ITLP algorithm still performs well for this problem instance, but the runtime, particularly for generating the transition probabilities (i.e., cpu-trans), becomes excessively long at this scale.

Figure~\ref{cpomdp:fig:inf-horizon-sim-comparison} provides a comparison of ITLP and CALP for varying grid sizes for the query problem, which is selected as the representative problem instance. 
The most notable effect of increasing grid set size for infinite horizon ITLP is that, similar to the finite horizon results, the region of budgets over which ITLP exceeds the budget constraint generally tends to narrow as the number of grid points increases. 
In contrast, the region over which CALP policies exceed the budget constraint actually increases slightly with the number of grid points. 
The results demonstrate that for policies which exceed the budget, the policy with the greater value tends to be the one that exceeds the budget with a greater margin. 
As policies generated by ITLP tend to exceed the budget by less, they also yield moderately lower values.
\begin{figure}[htb]
    \centering
    \subfloat[$|\setGrid{}|=50$ -- simulated value \label{fig:query-inf-50-grids-value}]{\includegraphics[width=0.490\textwidth]{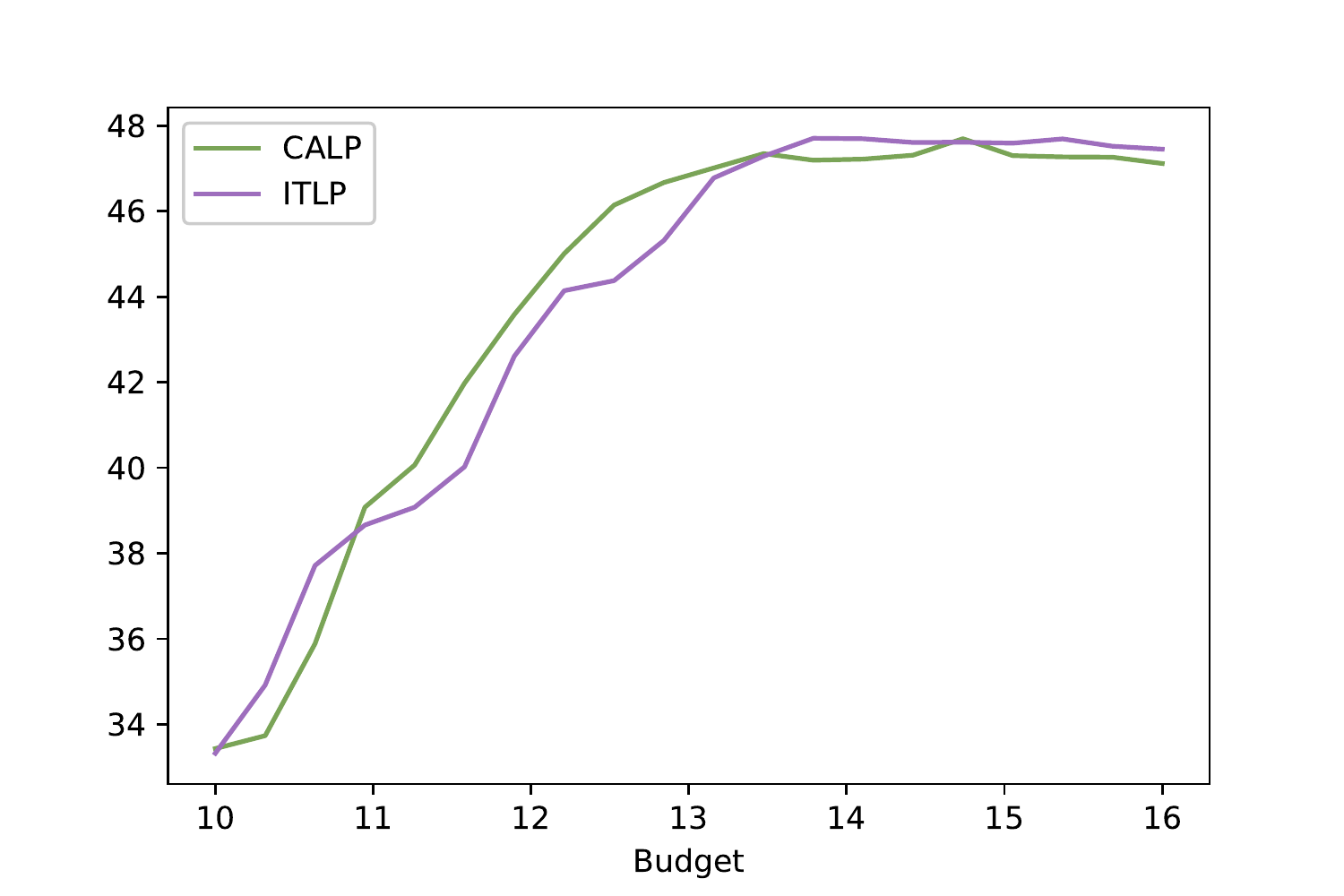}}
    \hfill
    \subfloat[$|\setGrid{}|=50$ -- budget adherence \label{fig:query-inf-50-grids-cost}]{\includegraphics[width=0.490\textwidth]{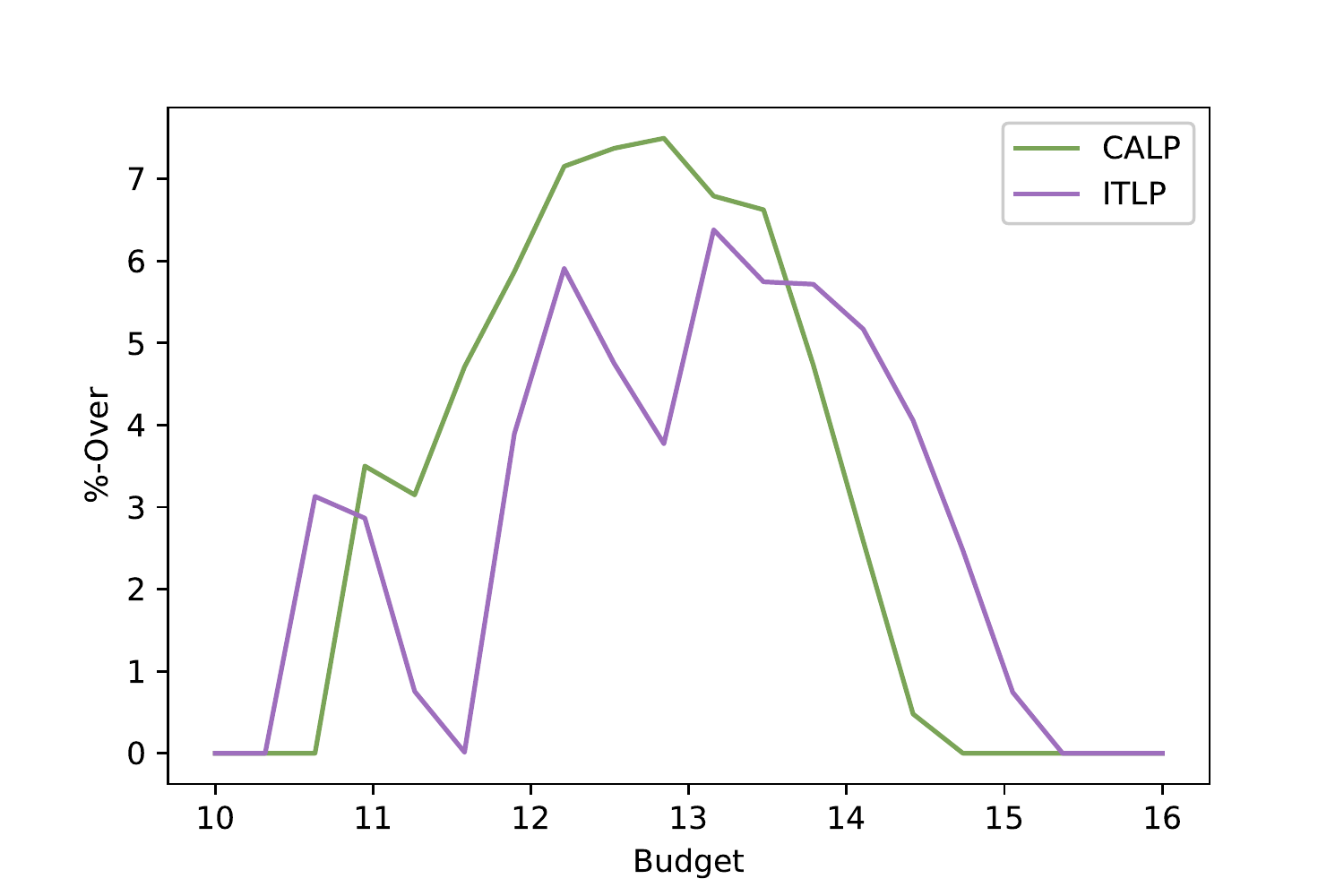}} \\
    \subfloat[$|\setGrid{}|=200$ -- simulated value \label{fig:query-inf-200-grids-value}]{\includegraphics[width=0.490\textwidth]{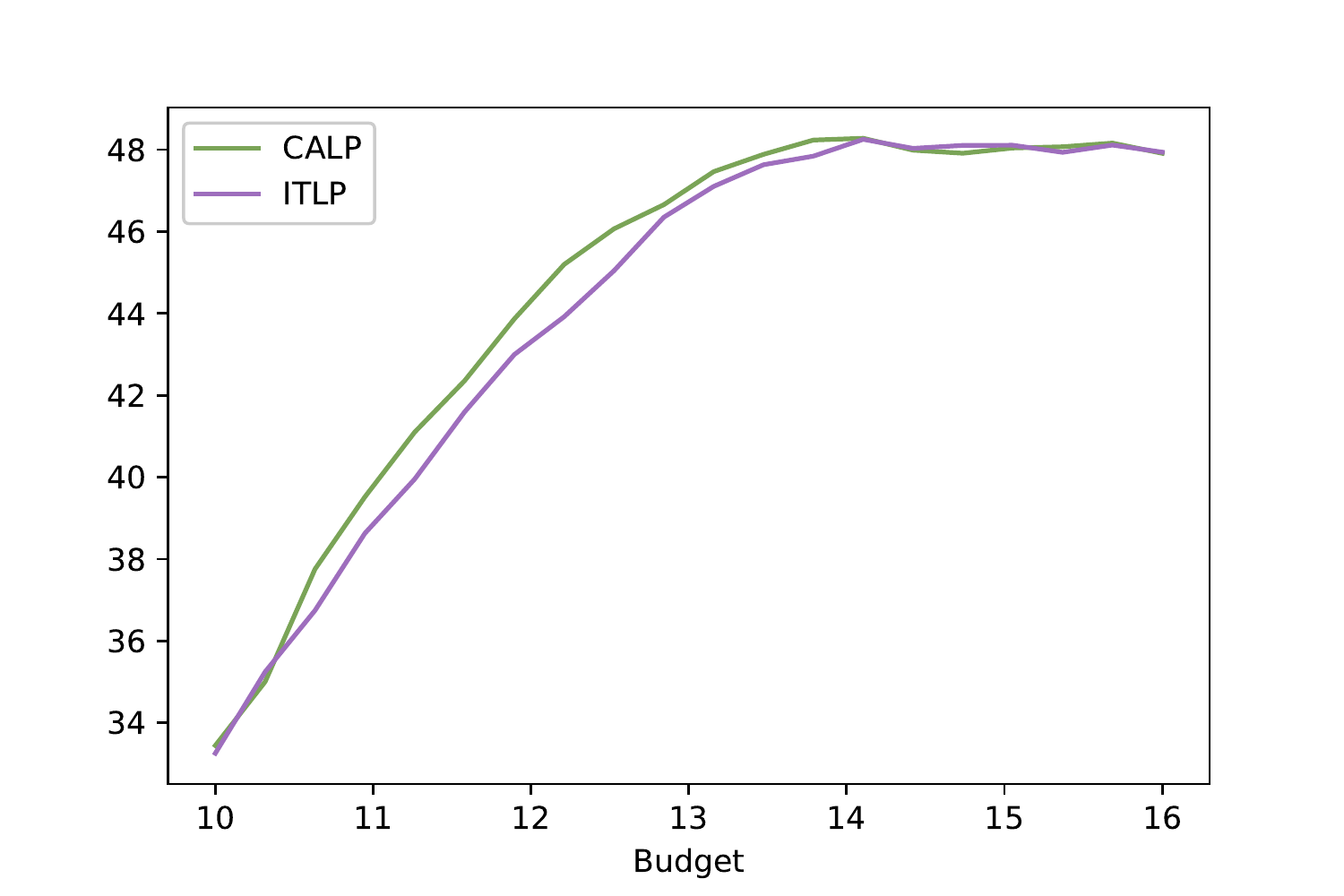}} 
    \subfloat[$|\setGrid{}|=200$ -- budget adherence \label{fig:query-inf-200-grids-cost}]{\includegraphics[width=0.490\textwidth]{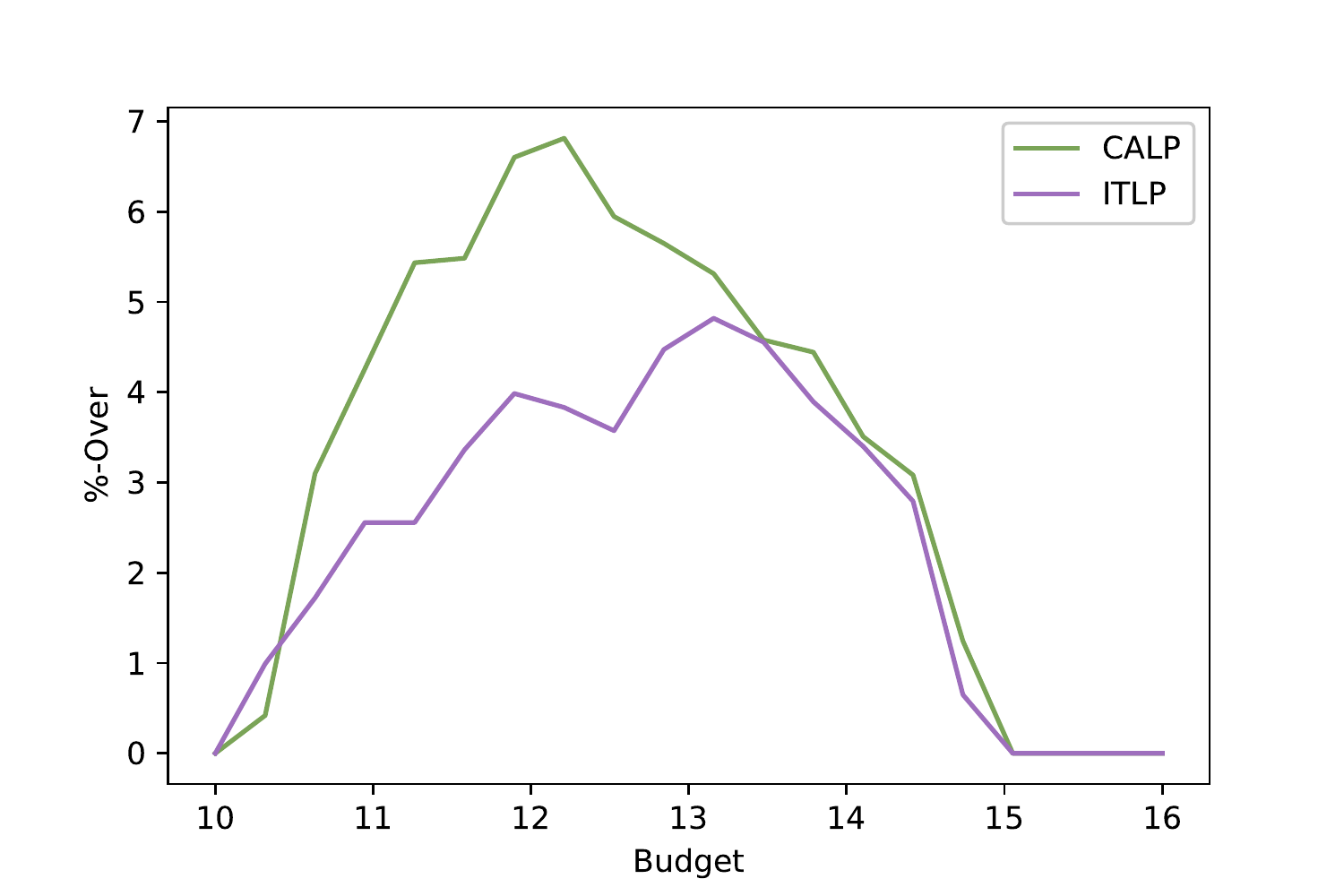}} \\
    \subfloat[$|\setGrid{}|=500$ -- simulated value\label{fig:query-inf-500-grids-value}]{\includegraphics[width=0.490\textwidth]{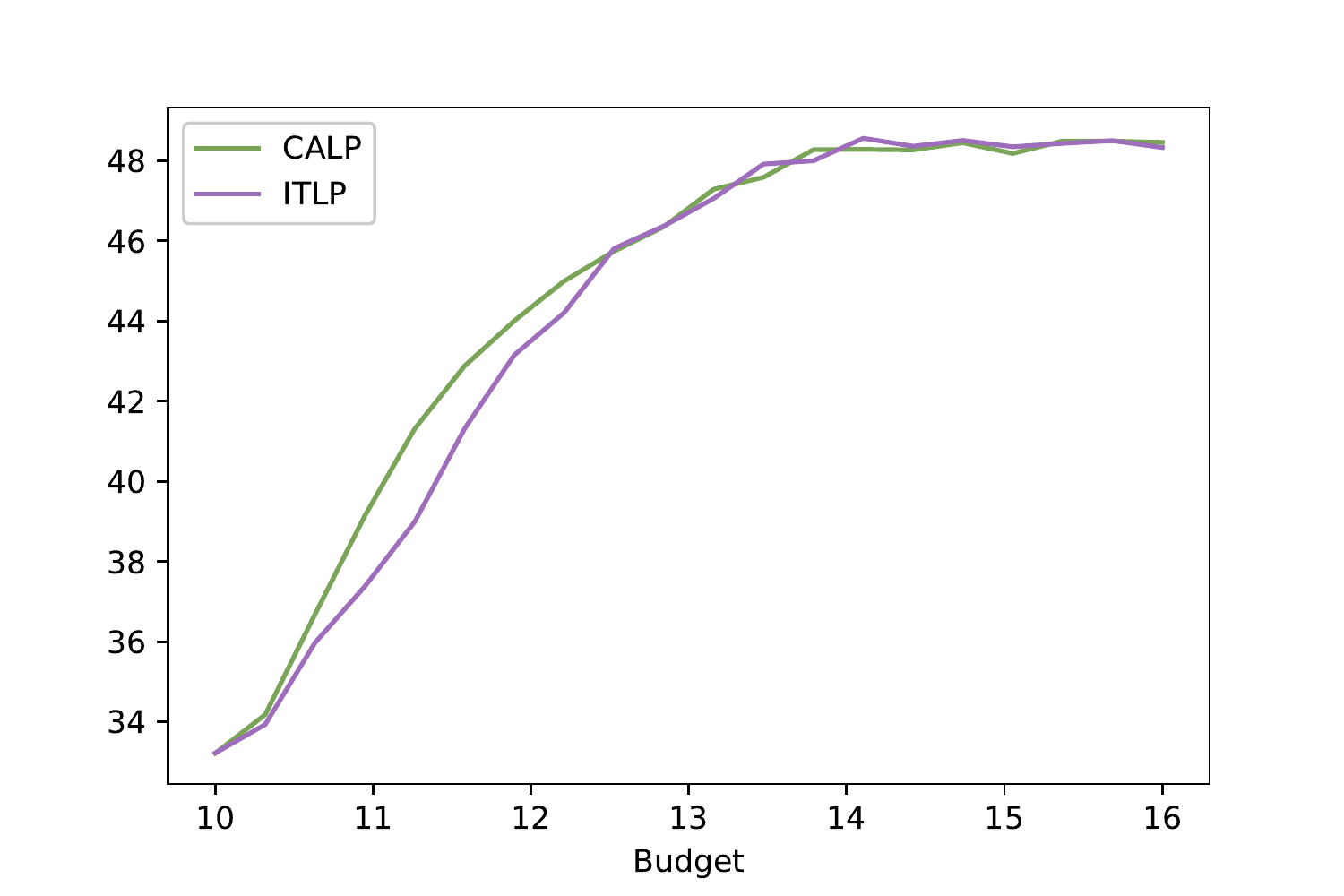}}
    \subfloat[$|\setGrid{}|=500$ -- budget adherence\label{fig:query-inf-500-grids-cost}]{\includegraphics[width=0.490\textwidth]{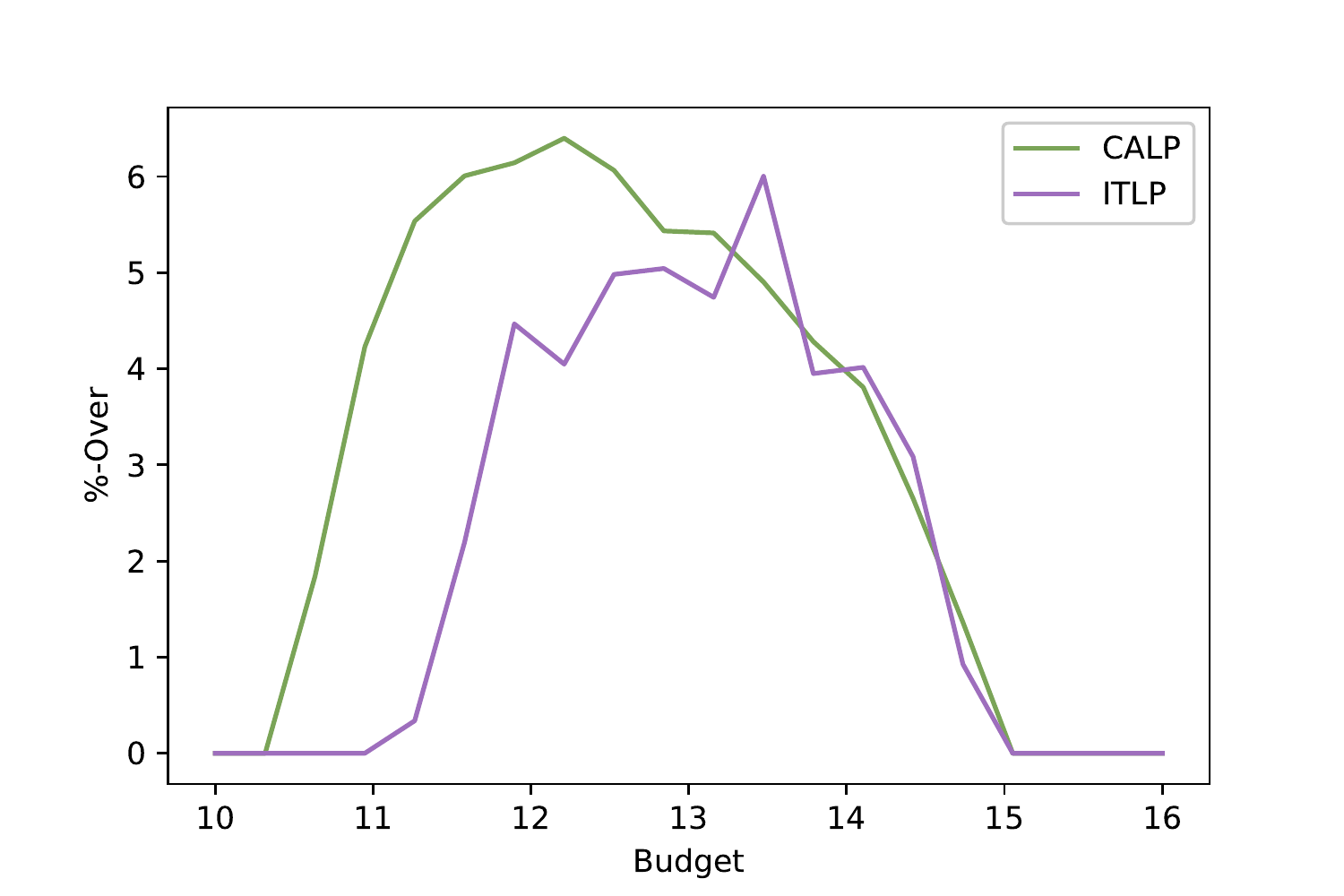}}
    \caption{Comparison of infinite horizon CPOMDP policies for different grid set sizes (budget adherence is given by Equation~\ref{cpomdp:eq:percent-over}).}
    \label{cpomdp:fig:inf-horizon-sim-comparison}
\end{figure}

The results for the query problem in Table~\ref{cpomdp:tab:infinite-constrained-summary} provide policy generation times for each of these policies for different grid set sizes. 
In the finite horizon case, the LP is considerably more difficult to solve, which is reflected in Table~\ref{cpomdp:tab:finite-complete-summary}. 
Specifically, observe that, unlike in Table~\ref{cpomdp:tab:finite-complete-summary}, cpu-lp is insignificant in comparison to cpu-trans for the infinite horizon case. 
In addition, the results show that ITLP's iterative approach to generating transition probabilities leads to considerably higher run times compared to CALP, with the cpu-lp also being slightly longer for the ITLP algorithm in most cases.

\subsection{Impact of deterministic policy constraints}\label{sec:results-deterministic-policy}

When a POMDP is transformed into a CPOMDP by adding a budget constraint, it can no longer be assumed that an optimal deterministic policy exists~\citep{puterman2014markov, Kim2011}. Despite this, there are still many applications where a suboptimal deterministic policy is preferred to an optimal randomized one, such as in the field of medical decision-making~\citep{Ayvaci2012, cevik2018analysis}. 
Accordingly, we quantify the impact of the deterministic policy constraints on the quality of the resulting policies. 
As a deterministic policy obtained with this method cannot collect a higher reward than its randomized policy counterpart, the performance of the LP and MIP approaches are compared using ``$\% \text{ Difference}$'', calculated as
\begin{equation}
    \% \text{ Difference}=100 \times \dfrac{
    \LPV(\stBelief) - \MIPV(\stBelief)
    }{
    |\LPV(\stBelief)|
    }
\end{equation}

In this experiment, 100 distinct belief states are considered for each of the three budget levels and the aggregate \% Difference values are reported in Figures~\ref{cpomdp:fig:finite-det-gap} and~\ref{cpomdp:fig:infinite-det-gap}. 
In the finite horizon case, except for the $4\times 3$ and rocksample problems, there is very little difference between the value attained by policies from LP and MIP models. 
In contrast, Figure~\ref{cpomdp:fig:infinite-det-gap} shows that a significant number of infinite horizon policies perform substantially worse when subject to deterministic policy constraints.

\begin{figure}[!ht]
    \centering
    \includegraphics[width=\textwidth]{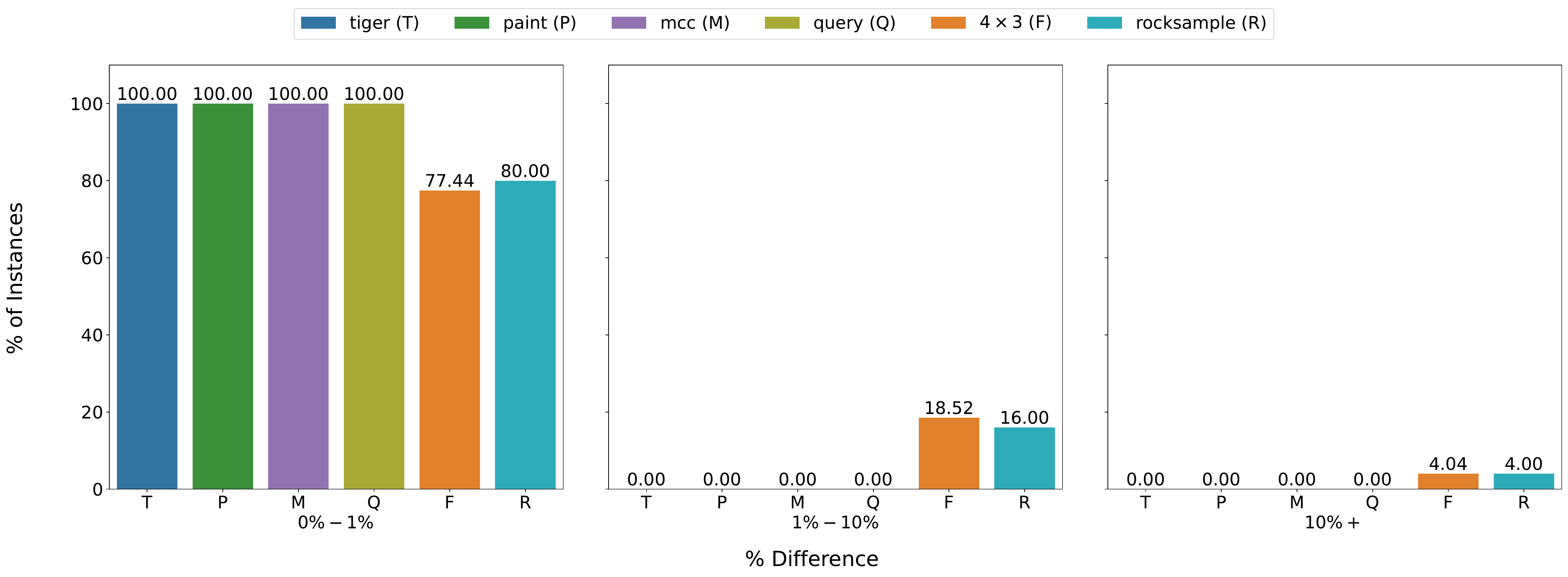}
    \caption{Finite horizon \% Difference between LP and MIP policy values.}
    \label{cpomdp:fig:finite-det-gap}
\end{figure}

\begin{figure}[!ht]
    \centering
    \includegraphics[width=\textwidth]{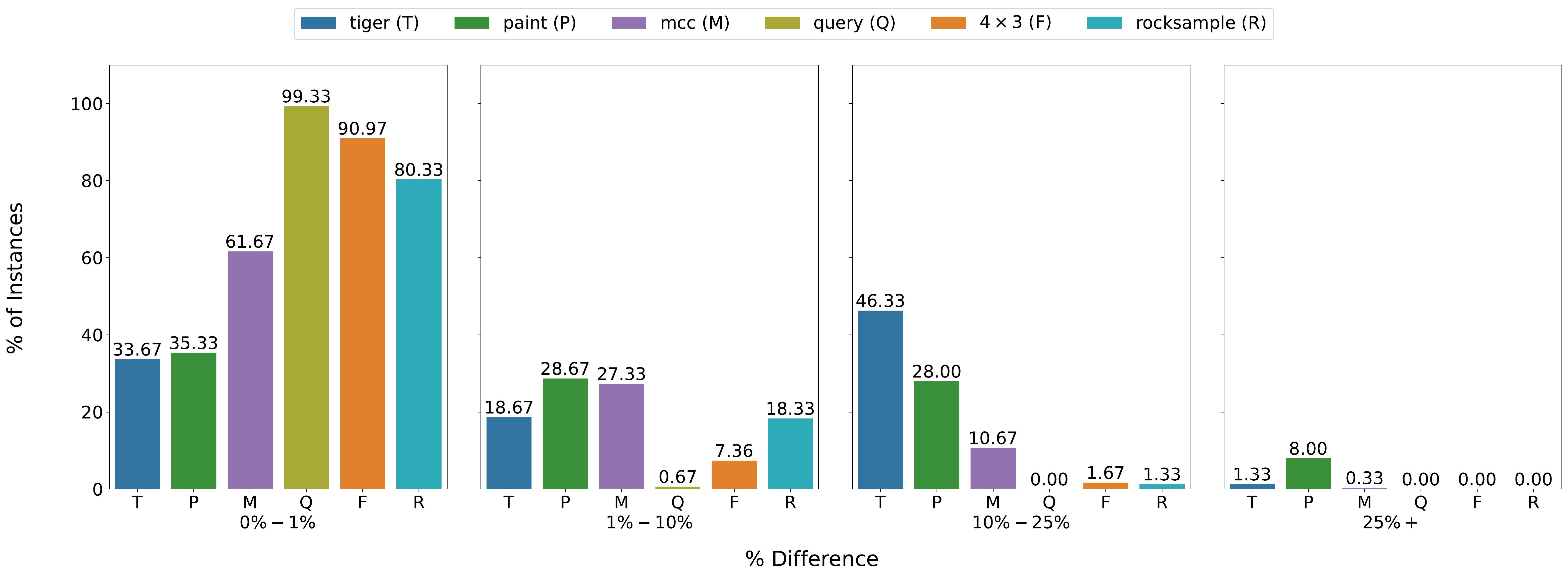}
    \caption{Infinite horizon \% Difference between LP and MIP policy values.}
    \label{cpomdp:fig:infinite-det-gap}
\end{figure}

Each of the problems with an exploratory action (tiger, paint, mcc) are able to adapt very well to the finite horizon deterministic policy constraints but, as expected, struggle in an infinite horizon setting as they cannot distribute their exploration over the various time steps in an infinite horizon setting. 
Notably, deterministic policies lead to a smaller value deterioration for the $4\times 3$ problem in an infinite horizon setting. 
This result is somewhat expected: recall that the $4\times 3$ agent must traverse a grid to reach the goal state. 
In the finite horizon case with $\timeHorizon=5$, budget limitations could prevent the agent from following a path that leads to the goal using the available number of actions. 
This can cause the agent not to collect any rewards in some cases. 
In the infinite horizon formulation, the agent can simply take a longer path, incurring a small penalty for taking a higher number of decision epochs to reach the goal, but these additional negative rewards are substantially smaller than the impact of not reaching the goal state.

The rocksample problem is most similar to the $4\times 3$ problem, as each of these problems involves an agent
traversing through a grid. 
Unsurprisingly, in Figure~\ref{cpomdp:fig:finite-det-gap} the results for each of these problems are similar, and of the problems in the test suite, it is these grid-traversal problems that are most sensitive to deterministic policy constraints in the finite-horizon setting. 
However, unlike the $4\times 3$ problem, the results for the rocksample problem do not improve much in the infinite horizon case. 
One possible explanation for this difference is that, unlike in the $4\times 3$ problem where deterministic constraints may make the agent take a longer path, in the rocksample problem the agent may be prevented from checking a rock, which ultimately means that it may not be able to sample that rock. 
As sampling is the most rewarding action in the rocksample problem, being unable to check if a rock is good or bad can have a substantial impact on the collected reward.
Despite these results, at least 80\% of starting beliefs for the rocksample problem result in a policy that yields a reward
within 1\% of its stochastic counterpart. 
This observation is promising for practical use, as it indicates that, in most instances, the impact of deterministic constraints is negligible.

Finally, an investigation into the effect of deterministic policy constraints on run time is conducted. 
Some adjustments are made to the previous setup to ensure that only statistically relevant differences in the run times are observed. 
In particular, only finite horizon problems are considered, all problems are considered with a horizon of $T=20$, and all grid sets are constructed with a size $|\setGrid{}|=500$ except for the rocksample problem, which uses $|\setGrid|=1000$. 
For all problems, the CPU times for 100 distinct budget levels in both the randomized (LP) and deterministic (MIP) cases are considered. 
The collected run times reflect only the amount of time spent by the LP solver searching for an optimal policy. 
The solver is given 1,000 seconds to search for a solution before timeout. 
It should be noted that these experiments were also conducted on each problem in an infinite horizon setting, but these results did not produce any statistically relevant run time differences. 
This result can be attributed to the fact that LP/MIP models are substantially smaller in the case of infinite horizon problems, rendering the impact of additional deterministic policy constraints on run times negligible.
However, for larger grid set sizes (e.g., in the order of tens of thousands), one might observe noticeable differences for infinite horizon problems as well.

\begin{figure}[!ht]
    \centering
    \includegraphics[width=0.75\textwidth]{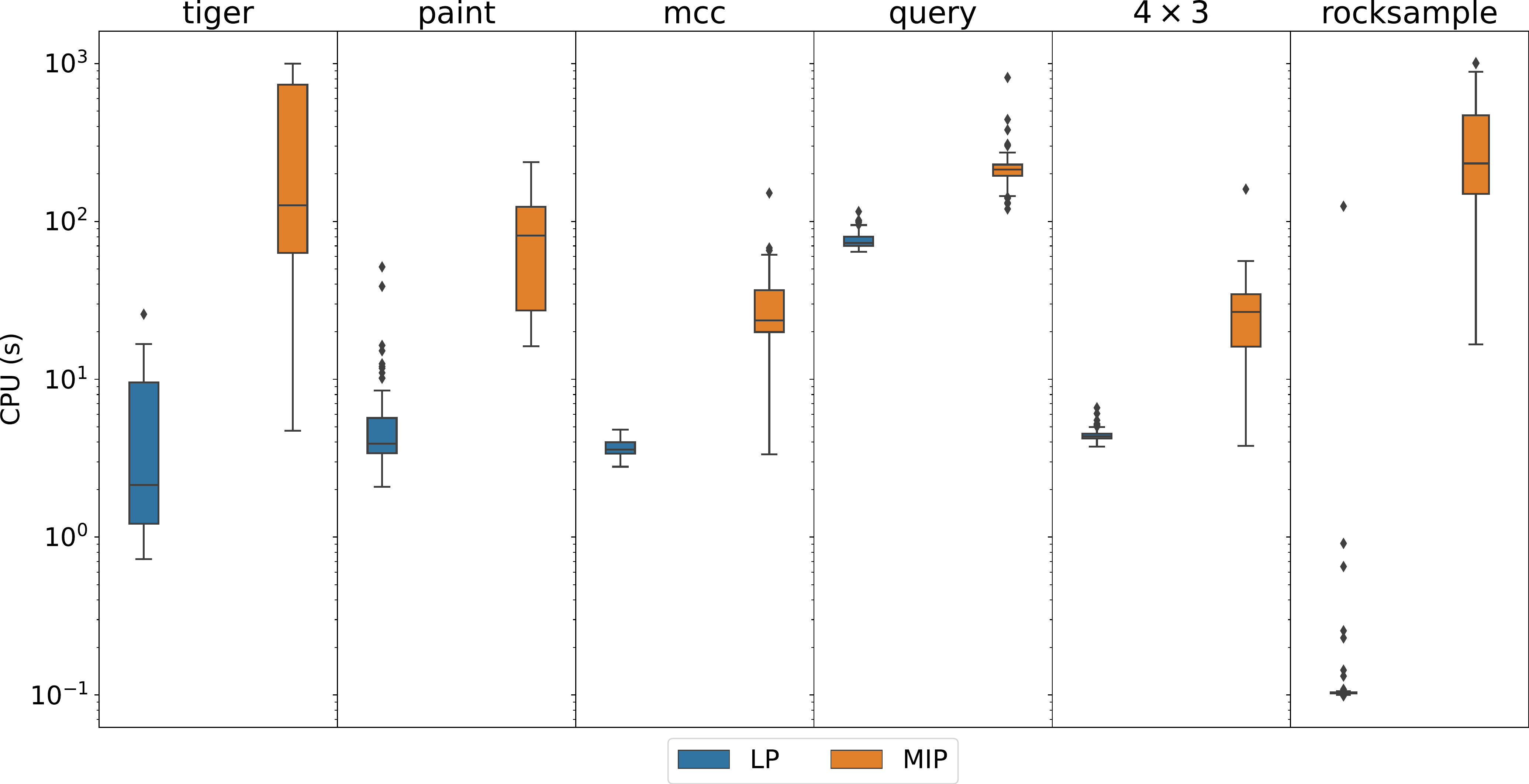}
    \caption{CPU time for LP versus MIP constrained problems over 100 budgets using a 1,000 second timeout. For each problem, $T=20$ and $|\setGrid{}|=500$ ($|\setGrid|=1000$ for the rocksample problem).}
    \label{fig:lp-vs-mip-run-times}
\end{figure}

Figure~\ref{fig:lp-vs-mip-run-times} compares the run times for the finite horizon LP and MIP problems on a log scale. 
Surprisingly, the tiger problem leads to the highest number of policies exceeding the time limit, with 23 of 100 cases timing out at 1000 seconds. 
The only other problem to experience timeout was the rocksample problem, which timed out at 1000 seconds for 13 of its 100 budgets. 
Additionally, Figure~\ref{fig:lp-vs-mip-run-times} shows that the rocksample problem, which has the largest number of states, the largest number of actions, and the number of grid points, is the quickest LP problem to optimize. 
The LP run times for the rocksample problem formulations average an order of magnitude below the tiger problem, which has the next closest LP run times. 
This result can be explained by the fact that the transition probability matrix for the rocksample problem is considerably more sparse than each of the other problems, as most actions (other than those taken in the terminal state) result in a deterministic transition. 
The fact that the LP formulation of the rocksample problem is the quickest to optimize while the MIP formulation is one of the slowest demonstrates the difficulty in predicting the runtime of MIP problems, which are theoretically NP-hard.

Considering each model individually, MIP models tend to take considerably longer than their LP counterparts, and they also contribute much more variance to run time. 
Additionally, MIP models require a timeout period to be specified to ensure that they complete execution in a reasonable amount of time. 
However, while the run times are considerably longer, they are not impractical and, in combination with the results in Figure~\ref{cpomdp:fig:finite-det-gap}, it is clear that deterministic policies, especially in the finite horizon case, are viable to produce in a practical setting, and are not always substantially inferior to stochastic policies. 
Furthermore, the results from the rocksample problem demonstrate that even for large core state spaces (e.g., for $|\setState| > 100$), generating deterministic policies can still be viable.

\section{Conclusion}
\label{cpomdp:conclusion}

This study provided an investigation into LP-based solution approaches for CPOMDPs.
Both finite horizon and infinite horizon CPOMDP problems were examined through an adaptation of six POMDP problem instances from the literature.
The validity of the proposed solution approach was first examined by quantifying the performance of the LB and UB approximation algorithms. 
Then, the proposed CPOMDP solution algorithm, ITLP, was developed as an extension of the UB algorithm.
The numerical analysis with ITLP for finite horizon problems revealed that, while some policies exceeded the budget constraint, all of them adapted to the constraint. 
This suggests that, as proposed by~\citet{poupart2015}, a post-processing step can be adopted to obtain feasible, tightly constrained policies.

For infinite horizon CPOMDPs, ITLP was compared to~\citet{poupart2015}'s CALP algorithm. 
These algorithms each work by reducing a CPOMDP to a constrained MDP through the use of a grid set. The numerical results revealed that while neither algorithm is best suited for all CPOMDP problems, each is able to outperform the other in several problem instances. 
Consistent with \citet{poupart2015}'s findings, both algorithms struggled to consistently produce policies that perfectly adhere to the budget constraint. 
To that extent, the degree to which generated policies exceeded their budget was compared. 
This allowed for a more accurate comparison of the performance of each algorithm and offered insights into the various factors which contribute to budget adherence. 
The key difference between the two algorithms lies in the generation of approximate transition probabilities. 
CALP approximates transition probabilities using a single time step, while ITLP iteratively generates transition probabilities until convergence. This iterative process causes longer run times for ITLP, but the results for both the $4\times 3$ and mcc problems showed that it allows the algorithm to generate far superior policies in some cases. ITLP further distinguishes itself from CALP by providing a mechanism for solving finite horizon problems, by supporting deterministic policy constraints, and by allowing for a probability distribution to be specified over $\setGrid$, as opposed to enforcing that a single grid $\vectGrPt\in\setGrid$ be chosen as the initial belief state.
In the finite horizon setting, the results were validated based on the construct of the ITLP algorithm.
Specifically, the ITLP results were compared against their unconstrained counterparts for each test problem.
It was observed that, as the budget restrictions were relaxed, the expected total rewards for CPOMDP policies tended toward that of their unconstrained POMDP counterparts.
Additionally, as the formulations of the ITLP algorithm for both finite and infinite horizon problems are similar, the comparable performance of ITLP with CALP in an infinite horizon context bolsters the promising observations made in the finite horizon setting.

The numerical study is concluded with an investigation into the performance of deterministic policy constraints on both the expected total rewards and run times for each CPOMDP problem. 
The results showed that, while in theory optimal CPOMDP policies are stochastic, in many contexts and especially in the finite horizon case, deterministic policies lead to a negligible deterioration in expected total rewards. Additionally, while run times were observed to be considerably longer for deterministic policy generation, typical run times are not prohibitively long when the grid set sizes or decision epochs are not excessively large. 
Together, these results demonstrate that deterministic CPOMDP policies can be viable in domains that require them, e.g., healthcare, while also demonstrating that stochastic policies, although potentially more difficult to implement in practice, are preferred when possible.

This study comes with several limitations, which can be potentially addressed in future research.
Specifically, while six different problem instances were used in the numerical study in order to identify the strengths and weaknesses in the proposed algorithm, a more extensive empirical analysis can help further validate these findings.
In this regard, future research can focus on the applications of these methods to practical problems in different domains, including disease screening problems in healthcare and machine maintenance problems in manufacturing systems.
Such specific applications can also allow for assessing the impact of incorporating different constraints, including threshold-type policy constraints, into the problem formulation.
The considered linear programming-based approaches can also be extended to solve multiobjective POMDPs/CPOMDPs, which have many practical applications~\citep{roijers2015point}.
Note that budget adherence was identified as a limitation for both ITLP and CALP through extensive simulations with the generated CPOMDP policies.
The numerical results showed that the design of the grid set, particularly its size, helps achieve increased budget adherence.
Furthermore, it was demonstrated that while the LP formulation allows for additional constraints to be incorporated, such additions come at the expense of computational overhead. 
In such cases, advanced optimization techniques such as column and row generation may help to improve solvability.
While the ITLP algorithm remained competitive for problems with a relatively low number of states, the results showed that for large problems with more than 100 states, the run time of the ITLP algorithm in the infinite horizon case can be significantly higher when compared to \citet{poupart2015}'s CALP algorithm. For such large problem instances, reinforcement learning-based approaches such as \citet{lee2018monte}'s Monte Carlo Tree Search algorithm can be considered as an alternative to the algorithms explored here.
In addition, similar to many other existing studies on CPOMDPs, this study does not establish theoretical bounds on the achievable CPOMDP expected total rewards with respect to the employed grid-based approximation.
Considering the lack of efficient exact solution methods for CPOMDPs, future research in this direction can explore approximation methods with performance guarantees.
Lastly, it was observed that the occupancy measure values were often very small, leading to numerical stability issues. This was particularly a problem in the finite horizon setting.
The investigation of such computational issues is left to future work.

\section*{Acknowledgements}
This research was enabled in part by the support provided by the Digital Research Alliance of Canada (\url{alliancecan.ca}).

\section*{Data availability statement}
All the datasets are publicly available, and can be obtained using the cited sources.

\section*{Disclosure statement}
No potential conflict of interest was reported by the authors.


\bibliographystyle{spbasic}
\bibliography{POMDPLit}

\begin{thebibliography}{47}
\providecommand{\natexlab}[1]{#1}
\providecommand{\url}[1]{{#1}}
\providecommand{\urlprefix}{URL }
\expandafter\ifx\csname urlstyle\endcsname\relax
  \providecommand{\doi}[1]{DOI~\discretionary{}{}{}#1}\else
  \providecommand{\doi}{DOI~\discretionary{}{}{}\begingroup
  \urlstyle{rm}\Url}\fi
\providecommand{\eprint}[2][]{\url{#2}}

\bibitem[{Ahluwalia et~al.(2021)Ahluwalia, Steimle, and
  Denton}]{ahluwalia2021policy}
Ahluwalia VS, Steimle LN, Denton BT (2021) Policy-based branch-and-bound for
  infinite-horizon multi-model markov decision processes. Computers \&
  Operations Research 126:105108

\bibitem[{Alagoz et~al.(2015)Alagoz, Ayvaci, and
  Linderoth}]{alagoz2015optimally}
Alagoz O, Ayvaci MU, Linderoth JT (2015) Optimally solving markov decision
  processes with total expected discounted reward function: Linear programming
  revisited. Computers \& Industrial Engineering 87:311--316

\bibitem[{Ayer et~al.(2012)Ayer, Alagoz, and Stout}]{Ayer2011a}
Ayer T, Alagoz O, Stout N (2012) A {POMDP} approach to personalize mammography
  screening decisions. {Operations Research} 60(5):1019--1034

\bibitem[{Ayvaci et~al.(2012{\natexlab{a}})Ayvaci, Alagoz, and
  Burnside}]{Ayvaci2012}
Ayvaci M, Alagoz O, Burnside E (2012{\natexlab{a}}) The effect of budgetary
  restrictions on breast cancer diagnostic decisions. M\&SOM 14(4):600--617

\bibitem[{Ayvaci et~al.(2012{\natexlab{b}})Ayvaci, Alagoz, and
  Burnside}]{ayvaci2012effect}
Ayvaci MU, Alagoz O, Burnside ES (2012{\natexlab{b}}) The effect of budgetary
  restrictions on breast cancer diagnostic decisions. Manufacturing \& Service
  Operations Management 14(4):600--617

\bibitem[{Bravo et~al.(2019)Bravo, Leiras, and Cyrino~Oliveira}]{bravo2019use}
Bravo RZB, Leiras A, Cyrino~Oliveira FL (2019) The use of uav s in humanitarian
  relief: an application of pomdp-based methodology for finding victims.
  Production and Operations Management 28(2):421--440

\bibitem[{Cassandra(1994)}]{Cassandra1995}
Cassandra A (1994) Optimal policies for partially observable {M}arkov decision
  processes. Brown University, Providence, RI

\bibitem[{Cassandra(2003)}]{CassandraExamples}
Cassandra A (2003) Simple examples. \url{http://www.pomdp.org/examples/},
  accessed: 2019-09-01

\bibitem[{Cassandra(1998)}]{cassandra1998exact}
Cassandra AR (1998) Exact and approximate algorithms for partially observable
  Markov decision processes. Brown University

\bibitem[{Cassandra et~al.(1994)Cassandra, Kaelbling, and
  Littman}]{cassandraActingOptimallyPartially1994}
Cassandra AR, Kaelbling LP, Littman ML (1994) Acting optimally in partially
  observable stochastic domains. In: AAAI, AAAI

\bibitem[{Celen and Djurdjanovic(2020)}]{celen2020integrated}
Celen M, Djurdjanovic D (2020) Integrated maintenance and operations decision
  making with imperfect degradation state observations. Journal of
  Manufacturing Systems 55:302--316

\bibitem[{Cevik et~al.(2018)Cevik, Ayer, Alagoz, and
  Sprague}]{cevik2018analysis}
Cevik M, Ayer T, Alagoz O, Sprague BL (2018) Analysis of mammography screening
  policies under resource constraints. Production and Operations Management
  27(5):949--972

\bibitem[{Egorov et~al.(2017)Egorov, Sunberg, Balaban, Wheeler, Gupta, and
  Kochenderfer}]{egorov2017pomdps}
Egorov M, Sunberg ZN, Balaban E, Wheeler TA, Gupta JK, Kochenderfer MJ (2017)
  Pomdps. jl: A framework for sequential decision making under uncertainty. The
  Journal of Machine Learning Research 18(1):831--835

\bibitem[{Erenay et~al.(2014)Erenay, Alagoz, and Said}]{Erenay2014}
Erenay F, Alagoz O, Said A (2014) Optimizing colonoscopy screening for
  colorectal cancer prevention and surveillance. M\&SOM 16(3):381--400

\bibitem[{Gan et~al.(2019)Gan, Scheller-Wolf, and Tayur}]{gan2019personalized}
Gan K, Scheller-Wolf AA, Tayur SR (2019) Personalized treatment for opioid use
  disorder. Available at SSRN 3389539

\bibitem[{Jiang et~al.(2017)Jiang, Wang, and Xi}]{jiang2017finding}
Jiang X, Wang X, Xi H (2017) Finding optimal polices for wideband spectrum
  sensing based on constrained pomdp framework. IEEE Transactions on Wireless
  Communications 16(8):5311--5324, \doi{10.1109/TWC.2017.2708124}

\bibitem[{Kavaklioglu and Cevik(2022)}]{kavaklioglu2022scalable}
Kavaklioglu C, Cevik M (2022) Scalable grid-based approximation algorithms for
  partially observable markov decision processes. Concurrency and Computation:
  Practice and Experience 34(5):e6743

\bibitem[{Kim et~al.(2011)Kim, Lee, Kim, and Poupart}]{Kim2011}
Kim D, Lee J, Kim K, Poupart P (2011) Point-based value iteration for
  constrained {POMDP}s. In: Twenty-Second International Joint Conference on
  Artificial Intelligence, pp 1968--1974

\bibitem[{Lee et~al.(2018)Lee, Kim, Poupart, and Kim}]{lee2018monte}
Lee J, Kim GH, Poupart P, Kim KE (2018) Monte-carlo tree search for constrained
  pomdps. Advances in Neural Information Processing Systems 31

\bibitem[{Lovejoy(1991{\natexlab{a}})}]{Lovejoy1991}
Lovejoy W (1991{\natexlab{a}}) {A Survey of Algorithmic Methods for Partially
  Observed Markov Decision Processes}. Annals of Operations Research 28:47--66

\bibitem[{Lovejoy(1991{\natexlab{b}})}]{Lovejoy1991b}
Lovejoy W (1991{\natexlab{b}}) Computationally feasible bounds for partially
  observed {Markov} decision processes. {Operations Research} 39(1):162--175

\bibitem[{Maillart(2006)}]{Maillart2006}
Maillart LM (2006) Maintenance policies for systems with condition monitoring
  and obvious failures. IIE Transactions 38(6):463--475

\bibitem[{McLay and Mayorga(2013)}]{mclay2013dispatching}
McLay LA, Mayorga ME (2013) A dispatching model for server-to-customer systems
  that balances efficiency and equity. Manufacturing \& Service Operations
  Management 15(2):205--220

\bibitem[{Monahan(1982)}]{Monahan1982}
Monahan G (1982) State of the art --- {A} survey of partially observable
  {M}arkov decision processes: Theory, models, and algorithms. Management
  Science 28(1):1--16

\bibitem[{Pajarinen and Kyrki(2017)}]{pajarinen2017robotic}
Pajarinen J, Kyrki V (2017) Robotic manipulation of multiple objects as a
  pomdp. Artificial Intelligence 247:213--228

\bibitem[{Parr and Russell(1995)}]{parrApproximatingOptimalPolicies}
Parr R, Russell S (1995) Approximating optimal policies for partially
  observable stochastic domains. In: IJCAI, IJCAI, vol~95, pp 1088--1094

\bibitem[{Pineau et~al.(2006)Pineau, Gordon, and Thrun}]{Pineau2006a}
Pineau J, Gordon G, Thrun S (2006) {Anytime Point-Based Approximations for
  Large POMDPs}. JAIR 27:335--380

\bibitem[{Poupart et~al.(2015)Poupart, Malhotra, Pei, Kim, Goh, and
  Bowling}]{poupart2015}
Poupart P, Malhotra A, Pei P, Kim KE, Goh B, Bowling M (2015) Approximate
  linear programming for constrained partially observable markov decision
  processes. In: Proceedings of the AAAI Conference on Artificial Intelligence,
  vol~29

\bibitem[{Puterman(2014)}]{puterman2014markov}
Puterman ML (2014) Markov decision processes: discrete stochastic dynamic
  programming. John Wiley \& Sons

\bibitem[{Roijers et~al.(2015)Roijers, Whiteson, and
  Oliehoek}]{roijers2015point}
Roijers DM, Whiteson S, Oliehoek FA (2015) Point-based planning for
  multi-objective pomdps. In: Twenty-Fourth International Joint Conference on
  Artificial Intelligence

\bibitem[{Rudin(1987)}]{rudin1987real}
Rudin W (1987) Real and complex analysis, 3rd edn. McGraw-Hill

\bibitem[{Sandikci(2010)}]{sandikci2010reduction}
Sandikci B (2010) Reduction of a pomdp to an mdp. Wiley Encyclopedia of
  Operations Research and Management Science

\bibitem[{Sand{\i}k{\c{c}}{\i} et~al.(2008)Sand{\i}k{\c{c}}{\i}, Maillart,
  Schaefer, Alagoz, and Roberts}]{sandikcci2008estimating}
Sand{\i}k{\c{c}}{\i} B, Maillart LM, Schaefer AJ, Alagoz O, Roberts MS (2008)
  Estimating the patient's price of privacy in liver transplantation.
  Operations Research 56(6):1393--1410

\bibitem[{Silver and Veness(2010)}]{silver2010monte}
Silver D, Veness J (2010) Monte-carlo planning in large pomdps. Advances in
  neural information processing systems 23

\bibitem[{Smith and Simmons(2012)}]{smith2012heuristic}
Smith T, Simmons R (2012) Heuristic search value iteration for pomdps. arXiv
  preprint arXiv:12074166

\bibitem[{Sondik(1971)}]{sondik1971optimal}
Sondik EJ (1971) The optimal control of partially observable Markov processes.
  Stanford University

\bibitem[{Spaan(2012)}]{spaan2012partially}
Spaan MT (2012) Partially observable markov decision processes. In:
  Reinforcement Learning, Springer, pp 387--414

\bibitem[{Steimle et~al.(2021{\natexlab{a}})Steimle, Ahluwalia, Kamdar, and
  Denton}]{steimle2021decomposition}
Steimle LN, Ahluwalia VS, Kamdar C, Denton BT (2021{\natexlab{a}})
  Decomposition methods for solving markov decision processes with multiple
  models of the parameters. IISE Transactions 53(12):1295--1310

\bibitem[{Steimle et~al.(2021{\natexlab{b}})Steimle, Kaufman, and
  Denton}]{steimle2021multi}
Steimle LN, Kaufman DL, Denton BT (2021{\natexlab{b}}) Multi-model markov
  decision processes. IISE Transactions 53(10):1124--1139

\bibitem[{{Suresh}(2005)}]{Suresh2005}
{Suresh} (2005) {Sampling from the simplex}. Available from
  http://geomblog.blogspot.com/2005/10/sampling-from-simplex.html Accessed on
  February 26, 2015.

\bibitem[{Sutton and Barto(2018)}]{sutton2018reinforcement}
Sutton RS, Barto AG (2018) Reinforcement learning: An introduction. MIT press

\bibitem[{Treharne and Sox(2002)}]{Treharne2002}
Treharne JT, Sox CR (2002) Adaptive inventory control for nonstationary demand
  and partial information. Management Science 48(5):607--624

\bibitem[{Walraven and Spaan(2018)}]{walraven2018column}
Walraven E, Spaan MT (2018) Column generation algorithms for constrained
  pomdps. Journal of artificial intelligence research 62:489--533

\bibitem[{Wray and Czuprynski(2022)}]{wray2022scalable}
Wray KH, Czuprynski K (2022) Scalable gradient ascent for controllers in
  constrained pomdps. In: 2022 International Conference on Robotics and
  Automation (ICRA), IEEE, pp 9085--9091

\bibitem[{Y{\i}lmaz(2020)}]{yilmaz2020integrated}
Y{\i}lmaz {\"O}F (2020) An integrated bi-objective u-shaped assembly line
  balancing and parts feeding problem: optimization model and exact solution
  method. Annals of Mathematics and Artificial Intelligence pp 1--18

\bibitem[{Y{\i}lmaz et~al.(2021)}]{yilmaz2021tactical}
Y{\i}lmaz {\"O}F, et~al. (2021) Tactical level strategies for multi-objective
  disassembly line balancing problem with multi-manned stations: an
  optimization model and solution approaches. Annals of Operations Research pp
  1--51

\bibitem[{Young et~al.(2013)Young, Ga{\v{s}}i{\'c}, Thomson, and
  Williams}]{young2013pomdp}
Young S, Ga{\v{s}}i{\'c} M, Thomson B, Williams JD (2013) Pomdp-based
  statistical spoken dialog systems: A review. Proceedings of the IEEE
  101(5):1160--1179

\end{thebibliography}

\clearpage
\appendix

\section*{Appendix}
\renewcommand{\thefigure}{A\arabic{figure}}
\setcounter{figure}{0}
\renewcommand{\thetable}{A\arabic{table}}
\setcounter{table}{0}
\section{Proof of upper bound approximation}\label{sec:appendix_ubproof}
The proof of the upper bound approximation requires two intermediate theorems:
\begin{theorem}\label{thm:piecewise-linear-and-convex}
    The optimal value function, $\funcOptV^*$, is piecewise linear and convex.
\end{theorem}
\begin{proof}
    See~\citet{Cassandra1995}.\hfill$\qed$
\end{proof}
\begin{theorem}\label{thm:jensens-inequality}
    For any real-valued convex function $f$ and discrete random variable $X$
    \begin{equation}
        f(\mathbb{E}[X])\leq \mathbb{E}[f(X)]
    \end{equation}
    where $\mathbb{E}(\cdot)$ gives the expected value.
\end{theorem}
\begin{proof}
    See~\citet[Thm. 3.3]{rudin1987real}.\hfill$\qed$
\end{proof}

For completeness, the proof of the upper bound approximation (Theorem~\ref{thm:ub}) is reproduced from~\citep{sandikcci2008estimating}.
\begin{theorem}\label{thm:ub}
Let $\funcApxV(\grPt)$ give the grid-based upper bound approximation of the true optimal value, $\funcOptV(\grPt)$. Then,
\begin{equation}
    \funcApxV(\grPt)\geq \funcOptV(\grPt)\qquad \forall \grPt\in\setGrid
\end{equation}
\end{theorem}
\begin{proof}
    \citet{sandikcci2008estimating} uses proof by induction using the value iteration algorithm. From Equations~\eqref{eq:bellman-t-optimal} and~\eqref{eq:bellman-iteration-step}, the optimal value at time $\indTime$ for the belief $\bePt$ is
    \begin{equation}
        \funcOptV_{\indTime}^{*}(\vectBePt) =  \max_{\indAction\in\setAction}\left\{\sum_{\indState \in \setState} \bePt_{\indState} \parfRew_{\indState}^{\indTime \indAction}  + \sum_{\indState \in \setState}\bePt_{\indState} \bigg( \sum_{\indObs \in \setObs} \sum_{\indStateNew \in \setState} \parfObs_{\indStateNew \indObs}^{\indTime \indAction} \ \parfTrp_{\indState \indStateNew}^{\indTime \indAction}\ \funcOptV_{\indTime+1}^*(\vectBePt') \bigg)\right\}
    \end{equation}
    For the first iteration, where $\indTime=\timeHorizon$, from Equation~\eqref{eq:bellman-T-optimal}, the initial values for the optimal and approximate value functions are~$\funcApxV_\timeHorizon(\grPt)=\funcOptV^*_\timeHorizon(\grPt)=\sum_{\indState \in \setState} \bePt_{\indState}\rewFinal_{\indState}$. Assume that for all $\indTime\in[1,\timeHorizon]$ Theorem~\ref{thm:ub} holds. The value for generic $\indTime-1$ is
    \begin{subequations}
    \begin{align}
        \funcOptV_{\indTime-1}^{*}(\vectBePt) &=  \max_{\indAction\in\setAction}\left\{\sum_{\indState \in \setState} \bePt_{\indState} \parfRew_{\indState}^{(\indTime-1) \indAction}  + \sum_{\indState \in \setState}\bePt_{\indState} \bigg( \sum_{\indObs \in \setObs} \sum_{\indStateNew \in \setState} \parfObs_{\indStateNew \indObs}^{(\indTime-1) \indAction} \ \parfTrp_{\indState \indStateNew}^{(\indTime-1) \indAction}\ \funcOptV_{\indTime}^*(\vectBePt') \bigg)\right\}
        \intertext{Recall that because $\setGrid$ contains all of the corner points, any $\bePt\in\beSimp$ can be represented as a convex combination of $\grPt\in\setGrid$ i.e., $\bePt$ can be replaced with $\sum_{\indGrid\in\setGridIndex}\beta_{\indGrid}\grPt_\indGrid$ for non-negative $\beta$ satisfying $\sum_{\indGrid\in\setGridIndex}\beta_{\indGrid}=1$. Thus}
        &=\max_{\indAction\in\setAction}\left\{\sum_{\indState \in \setState} \bePt_{\indState} \parfRew_{\indState}^{(\indTime-1) \indAction}  + \sum_{\indState \in \setState}\bePt_{\indState} \left( \sum_{\indObs \in \setObs} \sum_{\indStateNew \in \setState} \parfObs_{\indStateNew \indObs}^{(\indTime-1) \indAction} \ \parfTrp_{\indState \indStateNew}^{(\indTime-1) \indAction}\ \funcOptV_{\indTime}^*\left(\sum_{\indGrid\in\setGridIndex}\beta_{\indGrid}\grPt_\indGrid\right) \right)\right\}
        \intertext{Theorem~\ref{thm:piecewise-linear-and-convex} states that $\funcOptV^*$ satisfies the properties needed to apply Theorem~\ref{thm:jensens-inequality} as follows:}
        &\leq\max_{\indAction\in\setAction}\left\{\sum_{\indState \in \setState} \bePt_{\indState} \parfRew_{\indState}^{(\indTime-1) \indAction}  + \sum_{\indState \in \setState}\bePt_{\indState} \left( \sum_{\indObs \in \setObs} \sum_{\indStateNew \in \setState} \parfObs_{\indStateNew \indObs}^{(\indTime-1) \indAction} \ \parfTrp_{\indState \indStateNew}^{(\indTime-1) \indAction}\ \sum_{\indGrid\in\setGridIndex}\beta_{\indGrid}\funcOptV_{\indTime}^*(\grPt_\indGrid) \right)\right\}
        \intertext{Per the induction hypothesis, this is}
        &\leq\max_{\indAction\in\setAction}\left\{\sum_{\indState \in \setState} \bePt_{\indState} \parfRew_{\indState}^{(\indTime-1) \indAction}  + \sum_{\indState \in \setState}\bePt_{\indState} \left( \sum_{\indObs \in \setObs} \sum_{\indStateNew \in \setState} \parfObs_{\indStateNew \indObs}^{(\indTime-1) \indAction} \ \parfTrp_{\indState \indStateNew}^{(\indTime-1) \indAction}\ \sum_{\indGrid\in\setGridIndex}\beta_{\indGrid}\funcApxV{\indTime}^*(\grPt_\indGrid) \right)\right\}\\
        &=\funcApxV_{\indTime-1}(\grPt)
    \end{align}
    \end{subequations}
    \hfill$\qed$
\end{proof}

\clearpage
\section{Proof of lower bound approximation}\label{sec:appendix_lbproof}
The lower bound approximation outlined in~\citet{Lovejoy1991b} is given in Theorem~\ref{thm:lovejoy-lb}.

\begin{theorem}\label{thm:lovejoy-lb}
Let $\setGridAlpha$ be the set of $\alpha$-vectors corresponding to the grid $\setGrid$ and $\funcApxV(\bePt)$ be the approximate expected value generated by $\setGridAlpha$ for the belief point $\bePt$. Then
\begin{equation}\label{eq:lb-theorem}
    \funcApxV(\bePt) \leq \funcOptV^*(\bePt),\qquad \forall\bePt\in\beSimp
\end{equation}
\end{theorem}

\begin{proof}
    After dropping the time index in Equation~\eqref{eq:alpha-vectors}, the inequality in Equation~\eqref{eq:lb-theorem} becomes
    \begin{subequations}
    \begin{align}\label{eq:lb-proof-eq-1}
    \funcApxV(\bePt) &\leq \max_{\alpha \in \setAlpha}\{\vectBePt \cdot \alpha\}\\
    \intertext{As $\setAlpha$ is defined as the set of all $\alpha$-vectors, by definition $\setGridAlpha\subseteq\setAlpha$. Thus, Equation~\eqref{eq:lb-proof-eq-1} can be rewritten as:}
    \label{eq:lb-proof-eq-2}
    \funcApxV(\bePt) &\leq \max\left\{\max_{\alpha \in \setGridAlpha}\{\vectBePt \cdot \alpha\}, \max_{\alpha \in \setAlpha \setminus \setGridAlpha}\{\vectBePt \cdot \alpha\}\right\} \\
    \intertext{Substituting in Equation~\eqref{eq:alpha-vectors}}
    \label{eq:lb-proof-eq-3}
    \funcApxV(\bePt) &\leq \max\left\{\funcApxV(\bePt), \max_{\alpha \in \setAlpha \setminus \setGridAlpha}\{\vectBePt \cdot \alpha\}\right\}
    \end{align}
    \end{subequations}
    It follows that right hand side of Equation~\eqref{eq:lb-proof-eq-3} is at least $\funcApxV(\bePt)$.
    \hfill $\qed$
\end{proof}

\clearpage
\section{Grid construction}
\label{sec:appendix_grid_construction}
\newcommand{\setAllowedBeliefComponents}{\mathcal{H}}

The approximation techniques employed in this study require the generation of a set of grid points, which provide a finite set of discrete belief states to approximate the continuous, infinite belief simplex. 
Specifically, these grid sets are generated using a slightly modified version of the fixed-resolution grid approach.

\subsection{Fixed-resolution grid construction approach}

Using the fixed-resolution grid approach, the resulting grid set contains beliefs sampled at equidistant intervals in each dimension of the state space according to a resolution parameter $\resoVal$. Specifically, for any grid point $\vectGrPt$, the $\indState$th component $\grPt_\indState$ can be any integer multiple of $\resoVal^{-1}$, subject to the constraint that all components must be non-negative and cannot exceed 1. Thus, each component of $\vectGrPt$ must belong to the set
\begin{equation}
    \setAllowedBeliefComponents(\resoVal)=\left\{0, \frac{1}{\resoVal}, \frac{2}{\resoVal}, \ldots, \frac{\resoVal-2}{\resoVal}, \frac{\resoVal-1}{\resoVal}, 1\right\}
\end{equation}
As $\vectGrPt$ represents a belief state, the sum of its components must equal 1. Therefore, for a problem with $|\setState|$ states and a resolution value of $\resoVal$, the approximate grid set $\setGrid$ can be generated by computing the Cartesian product of $\setAllowedBeliefComponents(\resoVal)$ with itself $|\setState|$ times and filtering off all grids whose components do not sum to 1. That is
\begin{align}
    \setGrid&=\left\{\grPt\mid \grPt\in \tilde{\setGrid},~\sum_i \grPt_i =1 \right\}\\
    \intertext{where}
    \tilde{\setGrid} &= \underbrace{\setAllowedBeliefComponents(\resoVal)\times \setAllowedBeliefComponents(\resoVal) \times \ldots \times \setAllowedBeliefComponents(\resoVal)}_{|\setState| \text{ times}}
\end{align}

\paragraph{Fixed-resolution grid set example}
Consider a problem with two states and a resolution value $\resoVal=2$. Following the fixed-resolution grid construction approach, first note that each component of $\vectGrPt$ must belong to the set
\begin{subequations}
\begin{equation}
    \setAllowedBeliefComponents(2)=\left\{0, \frac{1}{2}, 1\right\}
\end{equation}
Then, $\tilde{\setGrid}$ can then be computed as
\begin{align}
    \tilde{\setGrid}&=\left\{0, \tfrac{1}{2}, 1\right\}\times \{0, \tfrac{1}{2}, 1\}\\
    &=\left\{\left[ 0, 0 \right], [ 0, \tfrac{1}{2} ], \left[ 0, 1 \right], [ \tfrac{1}{2}, 0 ], [ \tfrac{1}{2}, \tfrac{1}{2} ], [ \tfrac{1}{2}, 1 ], [ 1, 0 ], [ 1, \tfrac{1}{2} ], [ 1, 1 ]\right\}
\end{align}
Only the grids in $\tilde{\setGrid}$ whose components sum to 1 are kept, resulting in
\begin{equation}
    \setGrid=\left\{[0, 1], [\tfrac{1}{2}, \tfrac{1}{2}], [1, 0]\right\}
\end{equation}
\end{subequations}

\subsection{Modified grid construction approach}

Following the fixed-resolution grid approach, the grid set generated using a resolution value of $\resoVal$ is denoted $\setGrid_\resoVal$, where
\begin{equation}
    \label{eq:size-of-grid-set}
    |\setGrid_\resoVal| = \binom{|\setState| + \resoVal - 1}{|\setState| - 1}
\end{equation}
By strictly constructing grid sets according to this approach is not feasible for problems with more than a few core states because, for many grid-based solution algorithms, the grid sets become prohibitively large for increasing $\resoVal$ values.
For example, in the $4\times 3$ problem, which has 11 states, the grid set sizes are $|\setGrid_{\resoVal=1}|=11$, $|\setGrid_{\resoVal=2}|=66$, $|\setGrid_{\resoVal=3}|=286$, $|\setGrid_{\resoVal=4}|=1001$, and $|\setGrid_{\resoVal=5}|=3003$ for resolution values ranging from one to five.
That is, the size of $\setGrid$ cannot be finely controlled. 
To address this, one can combine multiple grid sets generated using the fixed-resolution grid approach, taking only as many grids as they desire. 
For a desired grid set size of $\numGrids$, one can iteratively construct fixed-resolution grid sets until the first one whose size exceeds $\numGrids$ is encountered. 
Denoting this set as $\setGrid_\iota$, by definition, $|\setGrid_{\iota - 1}|\leq \numGrids$. 
Then, grid points can be sampled from $\setGrid_\iota\setminus\setGrid_{\iota - 1}$ and added to $\setGrid_{\iota -1}$ until $\setGrid_{\iota -1}$ has $\numGrids$ grid points. 
As there are many grid points to choose from, it might be important not to select these points arbitrarily, as this may result in a higher grid density in some dimensions. 
Accordingly, grids are drawn at equidistant intervals from the sorted set difference $\setGrid_{\iota}\setminus\setGrid_{\iota - 1}$. 
The overall procedure is summarized in Algorithm~\ref{alg:modified-grid-construction}.
\begin{algorithm}[!ht]
\caption{Modified grid-construction approach.}
\label{alg:modified-grid-construction}
\begin{algorithmic}[1] 
\Procedure{\texttt{Get\_Grid\_Set}}{$\numGrids,~|\setState|$}
    \State{$\iota\gets 1$}\label{alg-line-start:get-resolution}
    \Repeat
        \State{$\iota\gets \iota + 1$}
        \State{Compute $|\setGrid_\iota|$ according to Equation~\ref{eq:size-of-grid-set}.}
    \Until{$|\setGrid_\iota| > N$} \label{alg-line-end:get-resolution}
    \State{$\setGrid\gets \setGrid_{\iota - 1}$} \label{alg-line:assign-small-grid}
    \State{$\setGrid^*\gets \operatorname{sorted}(\setGrid_{\iota} \setminus \setGrid_{\iota -  1})$} \label{alg-line:assign-difference}
    \State{$\eta\gets \lfloor N~/~|\setGrid^*|\rfloor$.} \Comment{Step size, where $\lfloor \ldots \rfloor$ is the floor function.} \label{alg-line:step-size}
    \State{$\indGrid \gets 1$} \Comment{Here, the first element of $\setGrid^*$ is at index $\indGrid = 1$.}
    \Repeat \label{alg-line-start:add-element}
        \State{Add the $\indGrid$th element of $\setGrid^*$ to $\setGrid$.}
        \State{$\indGrid\gets \indGrid + \eta$}
    \Until{$|\setGrid| = N$} \label{alg-line-end:add-element}
    \State{\Return{$\setGrid$}}
\EndProcedure
\end{algorithmic}
\end{algorithm}

\paragraph{Modified grid-construction example}
Consider a problem with 3 states, where the desired number of grids, $\numGrids$, is 5. 
Following Algorithm~\ref{alg:modified-grid-construction}, the final value of $\iota$ is found to be 2, as $\setGrid_2$ is the first grid set whose size (6) exceeds $\numGrids$. 
Subsequently,
\begin{subequations}
\begin{align}
    \setGrid = \setGrid_{\iota=1} &= [[0, 0, 1], [0, 1, 0], [1, 0, 0]] \\
    \setGrid^*&=[[0, \tfrac{1}{2}, \tfrac{1}{2}], [\tfrac{1}{2}, 0, \tfrac{1}{2}], [\tfrac{1}{2}, \tfrac{1}{2}, 0]]
\end{align}
The step size $\eta$ is obtained as
\begin{align}
    \eta&=\lfloor N~/~|\setGrid^*| \rfloor\\
    &=\lfloor 5~/~3\rfloor \\
    &=\lfloor 1.\overline{6} \rfloor \\
    &= 1
\end{align}
Lastly, the 1st and 2nd elements of $\setGrid^*$ are added to $\setGrid$. Thus, the
final grid set becomes
\begin{equation}
\setGrid=[[0, 0, 1], [0, \tfrac{1}{2}, \tfrac{1}{2}], [0, 1, 0], [\tfrac{1}{2}, 0, \tfrac{1}{2}], [1, 0, 0]]
\end{equation}
\end{subequations}

\clearpage
\section{Calculating the grid transition probabilities}\label{sec:appendix_trp}
Figure~\ref{fig:grid_trp_calculation} summarizes the process of calculating the grid transition probabilities.
Note that these transition probabilities are essential inputs to the linear programming formulations for CPOMDPs discussed in this paper.
The steps in this process can be summarized as follows:
\begin{enumerate}[label={\textbf{(\arabic*)}}]\setlength\itemsep{0.3em}
    \item For each grid point, action, observation combination, the updated belief state is computed.
    \item The interpolation weights for $\grPt'$ are computed.
    \item The transition probability between grid points $\grPt\in \setGrid$ are computed.
    \item The approximate value function for time $\indTime+1$ is computed.
\end{enumerate}

\begin{figure}[!ht]
\centering
\begin{tikzpicture}[
    draw,
    solid,
    rounded corners=1mm,
    fill=blue!5,thick,minimum width=0em
]
    \clip (-3, -8) rectangle (3, .5);
    \node[draw, solid, rectangle, fill=none] (input-cross-product) {$\setGrid\times\setAction\times\setObs$};
    \node [below of=input-cross-product](input-tuple) {$(\grPt, \indAction, \indObs)$};
    \draw[-, solid, thick, shorten >= -3pt, shorten <= 2pt] (input-cross-product) edge (input-tuple);
    \node[draw, solid, rectangle, fill=red!15] [below of=input-tuple](cpomdp-env) {\begin{varwidth}{15em}\singlespacing\centering CPOMDP\\Environment\end{varwidth}};
    \draw[->, solid, thick, shorten >= 2pt, shorten <= -3pt] (input-tuple) edge node[left] {} (cpomdp-env);
    \node[draw, solid, rectangle, fill=none] [left=0cm and 0.5 cm of cpomdp-env](grid-set) {$\setGrid$};
    \node[draw, solid, rectangle, fill=none] [right=0cm and 0.5cm of cpomdp-env](v-hat) {$\funcApxV$};
    \node [below of=cpomdp-env](g-prime) {$\ugrPt$};
    \draw[-, solid, thick, shorten >= -3pt, shorten <= 2pt] (cpomdp-env) edge (g-prime);
    \node[draw, solid, rectangle, fill=blue!15] [below of=g-prime](get-beta) {Interpolate-weights};
    \draw[->, solid, thick, shorten >= 2pt, shorten <= 2pt] (grid-set) edge (get-beta);
    \draw[->, solid, thick, shorten >= 2pt, shorten <= 2pt] (v-hat) edge (get-beta);
    \draw[->, solid, thick, shorten >= 2pt, shorten <= -3pt] (g-prime) edge node[left] {} (get-beta);
    \node[draw, solid, rectangle, fill=none] [below=.5cm and 0cm of get-beta](beta) {$\beta_{\indGrid \indGridNew}^{\indTime \indAction \indObs}$};
    \draw[->, solid, thick, shorten >= 2pt, shorten <= 2pt] (get-beta) edge (beta);
    \node[draw, solid, rectangle, fill=none] [below=.5cm and 0cm of beta](f-values) {$\parfGridTr_{\indGrid \indGridNew}^{\indTime \indAction}$};
            \draw[->, solid, thick, shorten >= 2pt, shorten <= 2pt] (beta) edge (f-values);
            \draw[->, solid, thick, shorten >= 2pt, shorten <= 2pt] (f-values) to[out=-80, in=70, looseness=3] (v-hat);
\end{tikzpicture}
\caption{A flowchart for the grid transition probability calculation process.}
\label{fig:grid_trp_calculation}
\end{figure}
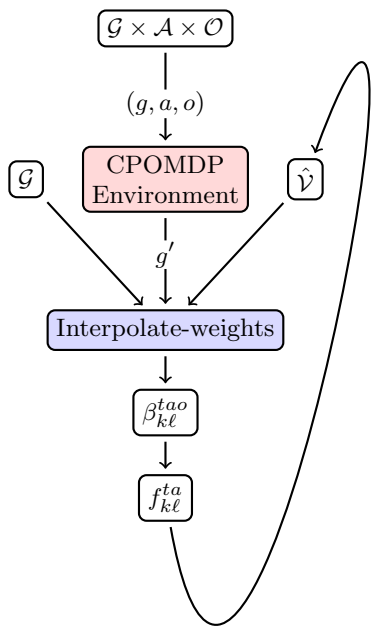

Algorithms \ref{alg:finite-horizon-trans-probs} and \ref{alg:infinite-horizon-trans-probs} show the calculation of the grid transition probabilities in detail for finite horizon and infinite horizon cases, respectively.
Before solving the corresponding LP models for the CPOMDPs, these algorithms are first run to obtain the $\parfGridTr$ values. 

\begin{algorithm}[!ht]
\caption{Grid transition probabilities for finite horizon CPOMDPs.}
\label{alg:finite-horizon-trans-probs}
\begin{algorithmic}[1] 
\Procedure{\texttt{Get\_Finite\_Horizon\_Transition\_Probabilities}}{$\setGrid,~\timeHorizon$}
    \State{Assign terminal reward $\funcApxV_\timeHorizon(\grPt)$ according to Equation~\eqref{eq:apx-terminal-reward} for all $\grPt\in\setGrid$.}\label{alg-line:terminal-reward}
    \For{$\indTime \in \{\timeHorizon-1, \hdots, 1\}$}\label{alg-line:for-time}
        \For{$(\grPt^\indGrid, \indAction, \indObs) \in (\setGrid\times\setAction\times \setObs)$}
            \State{Update belief state $\grPt^{\indGrid\prime}\leftarrow \grPt^\indGrid$ according to Equation~\eqref{eq:bayesian-update}.}
            \State{Calculate interpolation weights $\beta^k_{\indAction\indObs}$ according to the LP in~\eqref{eq:backward-induction}.}
        \EndFor
        \For{$(\indAction, \indGrid, \indGridNew)\in(\setAction\times\setGridIndex\times\setGridIndex)$}
            \State{Calculate grid transition probability $\parfGridTr_{\indGrid \indGridNew}^{\indTime \indAction}$ according to Equation~\eqref{eq:gridTransitionProbs}.}
        \EndFor
        \For{$\grPt\in\setGrid$}
            \State{Compute $\funcApxV_{\indTime}(\grPt)$ according to Equation~\eqref{eq:apx-value-given-grid-trans}.}
        \EndFor
    \EndFor
    \State{\Return{$\parfGridTr$, the grid transition probabilities.}}
\EndProcedure
\end{algorithmic}
\end{algorithm}

\begin{algorithm}[!ht]
\caption{Grid transition probabilities for infinite horizon CPOMDPs.}
\label{alg:infinite-horizon-trans-probs}
\begin{algorithmic}[1] 
\Procedure{\texttt{Get\_Infinite\_Horizon\_Transition\_Probabilities}}{$\setGrid,~\varepsilon$}
    \State{Initialize $\funcApxV_{\indTime+1}(\grPt)$ to zero for all $\grPt\in\setGrid$.}\label{alg-line:inf-terminal-reward}
    \Repeat
        \For{$(\grPt^\indGrid, \indAction, \indObs) \in (\setGrid\times\setAction\times \setObs)$}
            \State{Update belief state $\grPt^{\indGrid\prime}\leftarrow \grPt^\indGrid$ according to Equation~\eqref{eq:bayesian-update}.}
            \State{Calculate interpolation weights $\beta^k_{\indAction\indObs}$ according to the LP in~\eqref{eq:backward-induction}.}
        \EndFor
        \For{$(\indAction, \indGrid, \indGridNew)\in(\setAction\times\setGridIndex\times\setGridIndex)$}
            \State{Calculate grid transition probability $\parfGridTr_{\indGrid \indGridNew}^{\indAction}$ according to Equation~\eqref{eq:gridTransitionProbs}.}
        \EndFor
        \For{$\grPt\in\setGrid$}
            \State{Compute $\funcApxV_\indTime(\grPt)$ according to Equation~\eqref{eq:apx-value-given-grid-trans}.}
        \EndFor
        \State{$\texttt{converged}\gets$ \texttt{true}}
        \For{$\grPt\in\setGrid$}
            \If{$|\funcApxV_{\indTime+1}(\grPt) - \funcApxV_\indTime(\grPt)|>\varepsilon$}
                \State{$\texttt{converged}\gets$ \texttt{false}}
            \EndIf
        \EndFor
        \For{$\grPt\in\setGrid$}
            \State{$\funcApxV_{\indTime+1}(\grPt)\gets\funcApxV_\indTime(\grPt)$}
        \EndFor
    \Until{\texttt{converged}}
    \State{\Return{$\parfGridTr$, the grid transition probabilities.}}
\EndProcedure
\end{algorithmic}
\end{algorithm}

\clearpage
\section{Numerical examples}\label{sec:appendix_numerical_ex}

\subsection{Tiger problem parameters}
The following examples illustrate the application of ITLP to both the finite horizon and infinite horizon formulations of the tiger problem. 
Specifically, this problem is defined as follows~\citep{Cassandra1995}:
\begin{itemize}\setlength\itemsep{0.253em}
    \item \textbf{States:} $\{\text{Tiger left}~(s_0),~\text{Tiger right}~(s_1)\}$
    \item \textbf{Actions:} $\{\text{Listen}~(a_0),~\text{Open left door}~(a_1),~\text{Open right door}~(a_2)\}$
    \item \textbf{Observations:} $\{\text{Tiger left}~(o_0),~\text{Tiger right}~(o_1)\}$
    \item \textbf{Transition Probabilities:} Recall that the probability of transitioning from state $\indState$ to state $\indStateNew$ at time $t$ after taking action $a$ is $\parfTrp_{\indState \indStateNew}^{\indTime\indAction}$. In the tiger problem, transition probabilities are stationary: they are constant with respect to time. As a result, transition probabilities can simply be written as $\parfTrp_{\indState \indStateNew}^{\indAction}$. 
    The transition probabilities are provided in Table~\ref{tab:tiger_data_trp}.
    \begin{table}[!ht]
    \centering
    \caption{Tiger problem transition probabilities}\label{tab:tiger_data_trp}
    \resizebox{0.75\textwidth}{!}{
    \begin{tabular}{ccc}
        \begin{tabular}{crr}
        \toprule
            $\parfTrp_{\indState\indStateNew}^0$ & $\indStateNew=0$ & $\indStateNew=1$  \\
        \midrule
            $\indState = 0$ & 0.5 & 0.5 \\
            $\indState=1$ & 0.5 & 0.5 \\
        \bottomrule
        \end{tabular} &
        \begin{tabular}{crr}
        \toprule
            $\parfTrp_{\indState\indStateNew}^1$ & $j=0$ & $j=1$  \\
        \midrule
            $\indState = 0$ & 1 & 0 \\
            $\indState=1$ & 0 & 1 \\
        \bottomrule
        \end{tabular} &
        \begin{tabular}{crr}
        \toprule
            $\parfTrp_{\indState\indStateNew}^2$ & $j=0$ & $j=1$  \\
        \midrule
            $\indState = 0$ & 0.5 & 0.5 \\
            $\indState=1$ & 0.5 & 0.5 \\
        \bottomrule
        \end{tabular}
    \end{tabular}
    }
    \end{table}
    
    \item \textbf{Observation Probabilities:} As with transition probabilities, the observation probabilities are stationary for the tiger problem. Therefore, the probability of making observation $\indObs$ at time $\indTime$ after taking action $\indAction$ and arriving in state $\indStateNew$, $\parfObs_{\indStateNew \indObs}^{\indTime\indAction}$, can be written without the time index as $\parfObs_{\indStateNew \indObs}^{\indAction}$. 
    These observation probabilities are given in Table~\ref{tab:tiger_data_obs}.
    \begin{table}[!ht]
    \centering
    \caption{Tiger problem observation probabilities}\label{tab:tiger_data_obs}
    \resizebox{0.75\textwidth}{!}{
    \begin{tabular}{ccc}
        \begin{tabular}{crr}
        \toprule
            $\parfObs_{\indState\indObs}^0$ & $\indObs=0$ & $\indObs=1$  \\
        \midrule
            $\indState = 0$ & 0.85 & 0.15 \\
            $\indState=1$ & 0.15 & 0.85  \\
        \bottomrule
        \end{tabular} &
        \begin{tabular}{crr}
        \toprule
            $\parfObs_{\indState\indObs}^1$ & $\indObs=0$ & $\indObs=1$  \\
        \midrule
            $\indState= 0$ & 0.5 & 0.5 \\
            $\indState=1$ & 0.5 & 0.5 \\
        \bottomrule
        \end{tabular} &
        \begin{tabular}{crr}
        \toprule
            $\parfObs_{\indState\indObs}^2$ & $\indObs=0$ & $\indObs=1$  \\
        \midrule
            $\indState= 0$ & 0.5 & 0.5 \\
            $\indState=1$ & 0.5 & 0.5 \\
        \bottomrule
        \end{tabular}
    \end{tabular}
    }
    \end{table}
    
    \item \textbf{Immediate Rewards:} The reward for taking action $\indAction$ in state $\indState$ at time $\indTime$ is $\parfRew_{\indState\indAction}^{\indTime}$. The rewards are stationary for the tiger problem, so they can be simply denoted as $\parfRew_{\indState\indAction}$, and are given in Table~\ref{tab:tiger_data_rew}.
    \begin{table}[!ht]
        \centering
        \caption{Tiger problem immediate rewards.}\label{tab:tiger_data_rew}
        \resizebox{0.25\textwidth}{!}{
        \begin{tabular}{crr}
        \toprule
            $\parfRew_{\indState\indAction}$ & $\indState = 0$ & $\indState = 1$ \\
            \midrule
            $\indAction = 0$ & $-100$ & 10 \\
            $\indAction = 1$ & $-1$ & $-1$ \\
            $\indAction = 2$ & 10 & $-100$ \\
            \bottomrule
        \end{tabular}
        }
    \end{table}
    \item \textbf{Costs:} In this paper, the tiger POMDP is extended beyond~\citet{Cassandra1995}'s definition to include a cost on each action. Specifically, the cost of each door opening action ($a_1$, $a_2$) is taken as 1, and the cost of listening ($a_0$) is taken as 2.
\end{itemize}

For the purposes of the ensuing examples, the belief simplex for the tiger problem $\beSimp$ is approximated with the grid set $\setGrid=\{[0, 1], [0.5, 0.5], [1, 0]\}$. The terminal reward for exiting the decision process in one of these grids is set to the reward earned by taking the unconstrained optimal action in that grid. As the optimal action in $[1, 0]$ is to open the door on the right, the terminal reward is 10. Similarly, the reward for ending in $[0, 1]$ is 10. The optimal action when the belief is $[0.5, 0.5]$ is to listen, so the reward is $-1$. For simplicity, the initial belief distribution parameter is set as $\parBeDistn=[0, 1, 0]$, meaning that the starting belief state is $[0.5, 0.5]$.

\subsection{Finite horizon CPOMDP for the tiger problem}

Consider the tiger problem as described above with a horizon of 3 (i.e., 2 decision epochs), a budget of 3, and a discount factor of $\parDiscount=1$.
For a given belief state $\bePt$, the expected reward for taking action $\indAction$ is simply $\bePt\cdot\parfRew_{\indState\indAction}$. For example, the expected reward for opening the door on the left ($a_1$) when $\bePt=[0.5, 0.5]$ is 
\begin{align}
&\phantom{=}\bePt\cdot\parfRew_{\indState1}\\
&=\begin{pmatrix}0.5 & 0.5\end{pmatrix}\cdot\begin{pmatrix}-100 \\ 10\end{pmatrix} \nonumber \\
&=45\nonumber 
\end{align}
It directly follows that the coefficient of $\varDualLP_{\indTime \indGrid \indAction}$ in the objective function for the LP is this expected reward (i.e., the coefficient of $\varDualLP_{\indTime \indGrid \indAction}$ is $\bePt_\indGrid \cdot \parfRew_{\indState\indAction}$). In the tiger problem formulated above, this value is time independent. Following Equation~\eqref{eq:gridTransitionProbs}, it turns out that for the selected grid set, the obtained grid transition probabilities ($\parfGridTr_{\indGrid \indGridNew}^{\indTime \indAction}$) are time independent as well, so the time index can be dropped. These probabilities are given in Table~\ref{tab:finite-tiger-ex-grid-trans-probs}.
\begin{table}[!ht]
\newlength{\ftabwidth}
\setlength{\ftabwidth}{0.30\textwidth} 
\caption{Grid transition probabilities ($\parfGridTr_{\indGrid \indGridNew}^{\indAction}$)}
\label{tab:finite-tiger-ex-grid-trans-probs}
\begin{subtable}[h]{\ftabwidth}
    \caption*{$\indAction=0$}
    \resizebox{\ftabwidth}{!}{\begin{tabular}{lrrr}
        \toprule
        {} &  $\indGridNew=0$ &  $\indGridNew=1$ &  $\indGridNew=2$ \\
        \midrule
        $\indGrid=0$ &             1.00 &              0.0 &             0.00 \\
        $\indGrid=1$ &             0.35 &              0.3 &             0.35 \\
        $\indGrid=2$ &             0.00 &              0.0 &             1.00 \\
        \bottomrule
    \end{tabular}}
\end{subtable}
\hfill
\begin{subtable}[h]{\ftabwidth}
    \caption*{$\indAction=1$}
    \resizebox{\ftabwidth}{!}{\begin{tabular}{lrrr}
        \toprule
        {} &  $\indGridNew=0$ &  $\indGridNew=1$ &  $\indGridNew=2$ \\
        \midrule
        $\indGrid=0$ &              0.0 &              1.0 &              0.0 \\
        $\indGrid=1$ &              0.0 &              1.0 &              0.0 \\
        $\indGrid=2$ &              0.0 &              1.0 &              0.0 \\
        \bottomrule
    \end{tabular}}
\end{subtable}
\hfill
\begin{subtable}[h]{\ftabwidth}
    \caption*{$\indAction=2$}
    \resizebox{\ftabwidth}{!}{\begin{tabular}{lrrr}
        \toprule
        {} &  $\indGridNew=0$ &  $\indGridNew=1$ &  $\indGridNew=2$ \\
        \midrule
        $\indGrid=0$ &              0.0 &              1.0 &              0.0 \\
        $\indGrid=1$ &              0.0 &              1.0 &              0.0 \\
        $\indGrid=2$ &              0.0 &              1.0 &              0.0 \\
        \bottomrule
    \end{tabular}}
\end{subtable}
\end{table}

In this problem, the objective function given by Equation~\eqref{eq:duallp1Obj} is
\begin{subequations}
\begin{equation}\label{eq:finite-tiger-objective}
  \begin{split}- x_{0,0,0} &+ 10 x_{0,0,1} - 100 x_{0,0,2} - x_{0,1,0}
   - 45 x_{0,1,1} - 45 x_{0,1,2} - x_{0,2,0} \\ &- 100 x_{0,2,1}
   + 10 x_{0,2,2} - x_{1,0,0} + 10 x_{1,0,1} - 100 x_{1,0,2}
   - x_{1,1,0}\\ &- 45 x_{1,1,1} - 45 x_{1,1,2} - x_{1,2,0}
   - 100 x_{1,2,1} + 10 x_{1,2,2} \\ &+ 10 x_{N0} - x_{N1} + 10 x_{N2}
   \end{split}
\end{equation}
the constraints from Equation~\eqref{eq:duallp1Eqn1} are
\begin{align}\label{eq:finite-tiger-dist}
     x_{0,0,0} + x_{0,0,1} + x_{0,0,2} &= 0\\
     x_{0,1,0} + x_{0,1,1} + x_{0,1,2} &= 1\\
     x_{0,2,0} + x_{0,2,1} + x_{0,2,2} &= 0
\end{align}
the constraints from Equation~\eqref{eq:duallp1Eqn2} are
\begin{gather}
    - x_{0,0,0} - 0.35 x_{0,1,0} + x_{1,0,0} + x_{1,0,1}
   + x_{1,0,2} = 0\\
\resizebox{0.9\textwidth}{!}{$- x_{0,0,1} - x_{0,0,2} - 0.3 x_{0,1,0} - x_{0,1,1}
   - x_{0,1,2} - x_{0,2,1} - x_{0,2,2} + x_{1,1,0} 
   + x_{1,1,1} + x_{1,1,2} = 0$}\\
 - 0.35 x_{0,1,0} - x_{0,2,0} + x_{1,2,0} + x_{1,2,1}
   + x_{1,2,2} = 0
\end{gather}
the constraints from Equation~\eqref{eq:duallp1Eqn3} are
\begin{gather}
         - x_{1,0,0} - 0.35 x_{1,1,0} + x_{N0} = 0 \\
 - x_{1,0,1} - x_{1,0,2} - 0.3 x_{1,1,0} - x_{1,1,1}
   - x_{1,1,2} - x_{1,2,1} - x_{1,2,2} + x_{N1} = 0\\
 - 0.35 x_{1,1,0} - x_{1,2,0} + x_{N2} = 0
\end{gather}
and the budget constraint in Equation~\eqref{eq:duallp1Eqn4} is
\begin{equation}\label{eq:finite-tiger-budget}
    \begin{split}
    2 x_{0,0,0} &+ x_{0,0,1} + x_{0,0,2} + 2 x_{0,1,0}
   + x_{0,1,1} + x_{0,1,2} + 2 x_{0,2,0} \\ & + x_{0,2,1}
   + x_{0,2,2}  + 2 x_{1,0,0} + x_{1,0,1} + x_{1,0,2}
   + 2 x_{1,1,0} \\&+ x_{1,1,1} + x_{1,1,2} + 2 x_{1,2,0}
   + x_{1,2,1} + x_{1,2,2} <= 3
\end{split}
\end{equation}
\end{subequations}
giving the final LP for this finite horizon problem as:
\begin{subequations}
\begin{align}
\max &\quad\eqref{eq:finite-tiger-objective} \\
\sthat&\quad \eqref{eq:finite-tiger-dist}-\eqref{eq:finite-tiger-budget}\\
&\quad\eqref{eq:duallp1Eqn4}
\end{align}
\end{subequations}
The optimal $\varDualLP_{\indTime \indGrid \indAction}$ values obtained by solving this model are given in Table~\ref{tab:finite-tiger-ex-optimal-xtka}.
\begin{table}[ht]
    \centering
    \caption{Finite horizon optimal $\varDualLP_{\indTime \indGrid \indAction}$}
    \label{tab:finite-tiger-ex-optimal-xtka}
    \resizebox{0.75\textwidth}{!}{
    \begin{tabular}{cc}
        \begin{tabular}{crrr}
        \toprule
            $\varDualLP_{0 \indGrid \indAction}$ & $\indAction=0$ & $\indAction=1$ & $\indAction=2$ \\
        \midrule
            $\indGrid=0$ & 0 & 0 & 0 \\
            $\indGrid=1$ & 1 & 0 & 0 \\
            $\indGrid=2$ & 0 & 0 & 0 \\
        \bottomrule
        \end{tabular}
        &
        \begin{tabular}{crrr}
        \toprule
            $\varDualLP_{1 \indGrid \indAction}$ & $\indAction=0$ & $\indAction=1$ & $\indAction=2$ \\
        \midrule
            $\indGrid=0$ & 0 & 0.35 & 0 \\
            $\indGrid=1$ & 0 & 0.30 & 0 \\
            $\indGrid=2$ & 0 & 0 & 0.35 \\
        \bottomrule
        \end{tabular}
    \end{tabular}
    }
\end{table}

\subsection{Infinite horizon CPOMDP for the tiger problem}

We also provide the sample approximate formulation for the infinite horizon version of the tiger problem. 
Here, a discount factor $\parDiscount=0.9$ is chosen. 
Given the problem definition, Table~\ref{tab:finite-tiger-ex-grid-trans-probs} also gives the $\parfGridTr_{\indGrid \indGridNew}^{\indAction}$ values for this infinite horizon formulation. 
This is due to the very small grid size used here for this problem formulation. 
The objective from Equation~\eqref{eq:infinite-lp-obj} is then
\begin{subequations}
\begin{gather}
    - x_{0, 0} + 10 x_{0, 1} - 100 x_{0, 2} - x_{1, 0} - 45 x_{1, 1} - 45 x_{1, 2} - x_{2, 0} - 100 x_{2, 1} + 10 x_{2, 2}\label{eq:inf-tiger-obj}\\
    \intertext{the constraints in Equation~\eqref{eq:infinite-lp-dist-constraint} are}
         0.1 x_{0, 0} + x_{0, 1} + x_{0, 2} - 0.315 x_{1, 0} = 0\label{eq:inf-tiger-dist}\\
         - 0.9 x_{0, 1} - 0.9 x_{0, 2} + 0.73 x_{1, 0} + 0.1 x_{1, 1}  + 0.1 x_{1, 2} - 0.9 x_{2, 1} - 0.9 x_{2, 2} = 1\\
         - 0.315 x_{1, 0} + 0.1 x_{2, 0} + x_{2, 1} + x_{2, 2} = 0\\
    \intertext{and the budget constraint from Equation~\eqref{eq:infinite-lp-budget-constraint} is} 
    2 x_{0, 0} + x_{0, 1} + x_{0, 2} + 2 x_{1, 0} + x_{1, 1}
   + x_{1, 2} + 2 x_{2, 0} + x_{2, 1} + x_{2, 2} <= 11.5\label{eq:inf-tiger-budget}
\end{gather}
\end{subequations}
leading to a final LP for this infinite horizon CPOMDP example as 
\begin{align}
\max &\quad\eqref{eq:inf-tiger-obj} \\
\sthat&\quad \eqref{eq:inf-tiger-dist}-\eqref{eq:inf-tiger-budget}\\
&\quad\eqref{eq:infinite-lp-x-constraint}
\end{align}



\clearpage
\section{Analysis with threshold-type policies}\label{sec:appendix_threshold}
This section provides an application of threshold-type policy constraints for the tiger problem to provide a use case for specific constraints that can be incorporated into the ITLP algorithm.
As illustrated in Figure~\ref{fig:tiger-threshold}, threshold-type policy constraints help establish decision thresholds over the belief states.
The threshold-type policy constraints are imposed based on the stochastic dominance relation between the belief states.
That is, $\bePt^{\ell}$ stochastically dominates $\bePt^{k}$, denoted as $\bePt^{\ell} \succ_s \bePt^{k}$, if $\sum_{i=j} \bePt_i^{\ell} \geq \sum_{i=j} \bePt_i^{k}$ for all $j \in \{0,1,\hdots, |\setState|-1\}$.

\begin{figure}[!ht]
\centering
\tikzset{pblock/.style = {rectangle split, rectangle split horizontal,
                      rectangle split parts=3,
                      rectangle split part fill={red!20,yellow!20,blue!20}, draw, thick, align=center}}
\resizebox{0.75\textwidth}{!}{
\begin{tikzpicture}
  \node[name=s,pblock] (x) {\nodepart[text width=2cm]{one} Open Left
                \nodepart[text width=4cm]{two}Listen
                \nodepart[text width=3cm]{three}Open Right};
\draw[shift=(x.north west)] plot coordinates{(0,0)}
       node[above] {\scriptsize\texttt{$\bePt=[0, 1]$}};
\draw[shift=(x.one split north)] plot coordinates{(0,0)}
       node[above] {\scriptsize\texttt{$\bePt=[\rho_1, 1-\rho_1]$}};
\draw[shift=(x.two split north)] plot coordinates{(0,0)}
       node[above] {\scriptsize\texttt{$\bePt=[\rho_2, 1-\rho_2]$}};
\draw[shift=(x.north east)] plot coordinates{(0,0)}
       node[above] {\scriptsize\texttt{$\bePt=[1, 0]$}};
\end{tikzpicture}
}
\caption{Tiger problem threshold-type policy illustration}
\label{fig:tiger-threshold}
\end{figure}
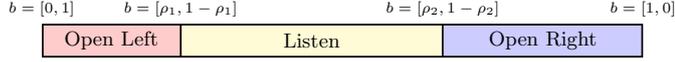

The threshold-type policy constraints for the action $\indAction_1 = \texttt{Open Left}$ can be formulated as
\begin{align}
\label{eq:threshold_tiger}
\theta_{tk1} \leq \theta_{t\ell1}, \quad \forall \grPt^{\ell} \succ_s \grPt^{k}, \forall t
\end{align}
where $\theta_{tka} \in \{0,1\}$ takes value 1 if action $a$ is selected at decision epoch $t$, and take value 0 otherwise (recall that these are the same binary variables that are included to obtain deterministic policies).
For instance, if $g^0 = [0,1]$ and $g^1 = [0.1, 0.9]$, then $\grPt^{0} \succ_s \grPt^{1}$.
It follows that if the optimal action for $g^1$ is $\indAction_1$, then the optimal action for $g^0$ should also be $\indAction_1$.
This is because, if the POMDP policy prescribes taking action $\indAction_1 = \texttt{Open Left}$ for a belief state $g^1$ that indicate a lower likelihood of tiger occupying state $s=\texttt{Tiger Right}$ (i.e., $g_1^1 = 0.9$), then this policy should prescribe the same action to a belief state $g^0$ that is expected to be more suitable for taking action $\indAction_1 = \texttt{Open Left}$ (i.e., because $g_0^1 = 1.0$).
Note that the threshold-type policies can be similarly derived for $\indAction_2 = \texttt{Open Right}$.
Figure~\ref{fig:tiger-threshold-type-policies} provides sample policies obtained by incorporating the threshold-type policy constraints into a finite horizon CPOMDP model with three different budget levels.
As expected, the obtained policies are all of threshold type. 

\begin{figure}[!ht]
    \centering
    \subfloat[Budget: 21.]{\includegraphics[width=0.33\linewidth]{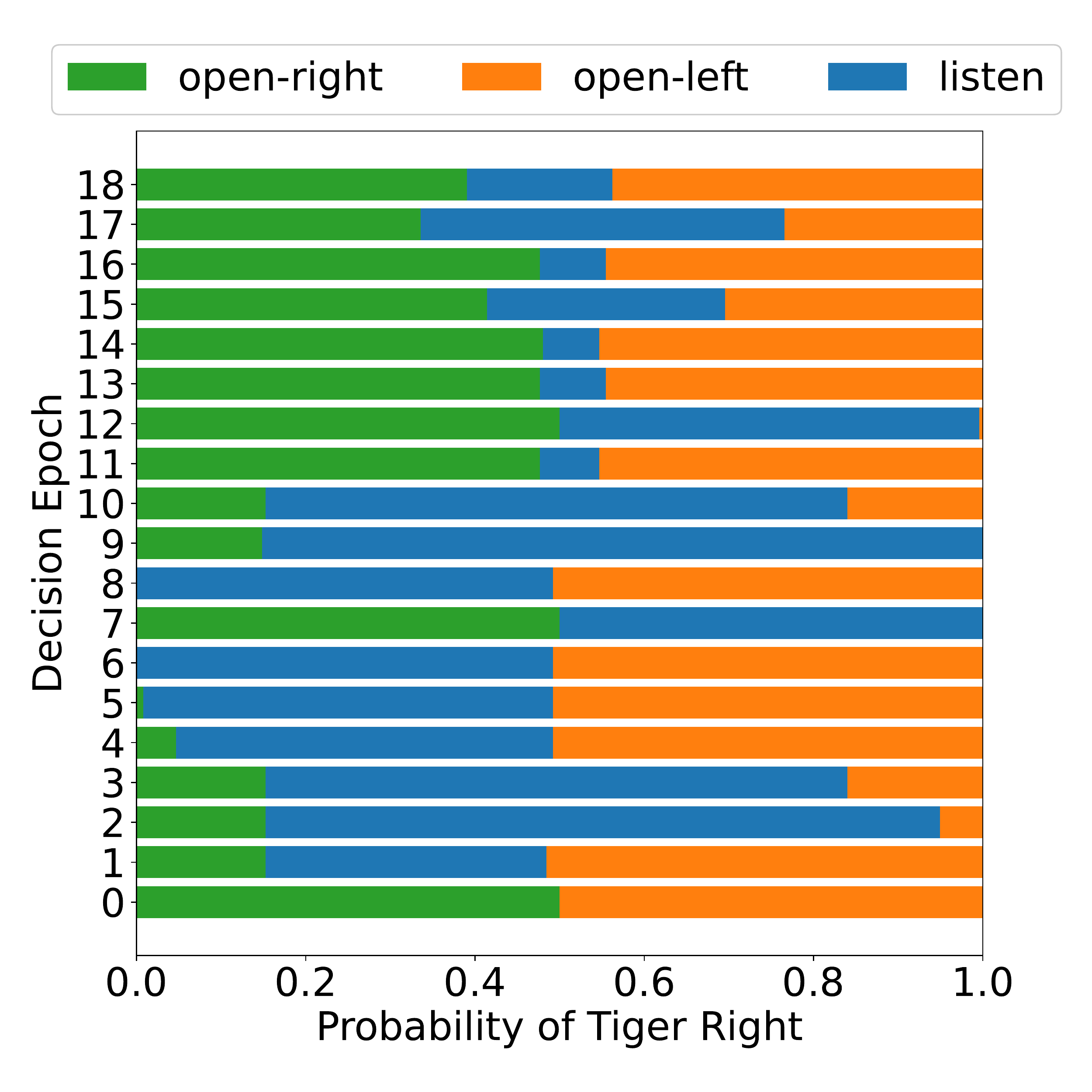}}
    \subfloat[Budget: 25.]{\includegraphics[width=0.33\linewidth]{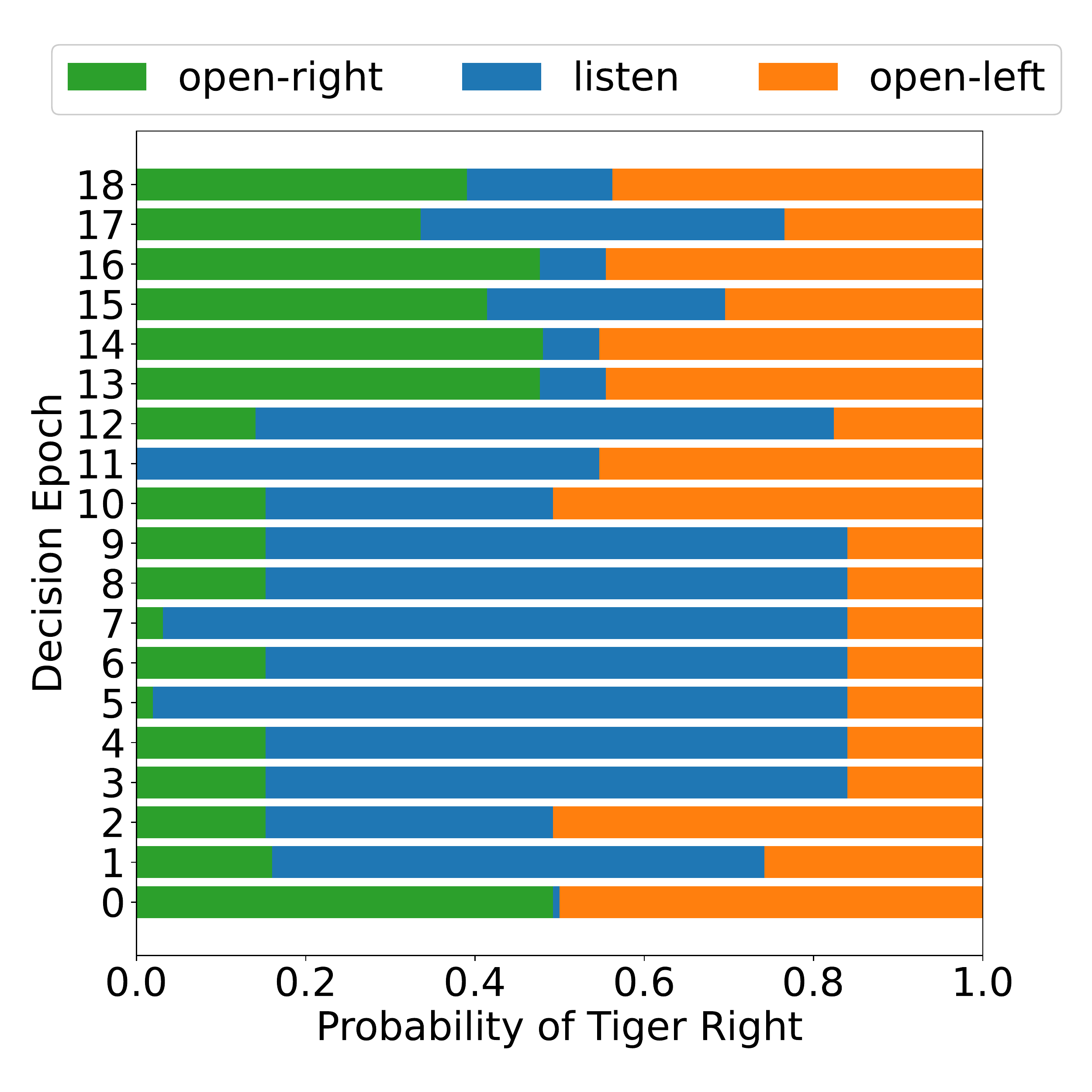}}
    \subfloat[Budget: 50.]{\includegraphics[width=0.33\linewidth]{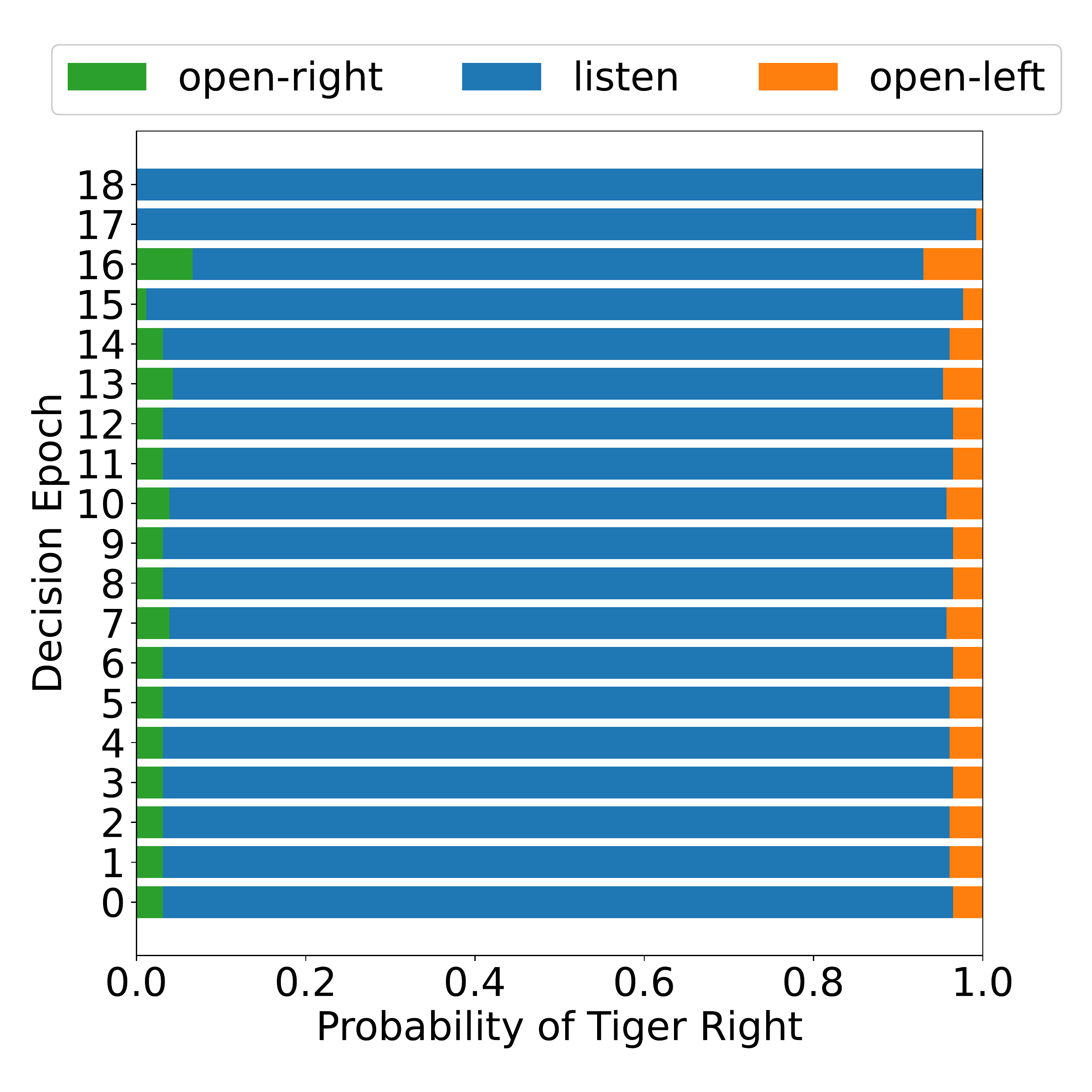}}
    \caption{Tiger problem threshold-type policy samples with $|\setGrid|=250$.}
    \label{fig:tiger-threshold-type-policies}
\end{figure}

Depending on the problem specifications, such threshold-type policy constraints can be formulated and added to the CPOMDP model. However, as their characterization rely on the stochastic ordering of the belief states, not all problems might be suitable for imposing threshold-type policy constraints.
For instance, in the paint problem, it is expected to \texttt{ship} items that are unflawed, unblemished and painted (i.e., $s=\texttt{NFL-NBL-PA}$), however, the ordering of the other three states (i.e., $s\in \{ \texttt{NFL-NBL-NPA}, \texttt{FL-NBL-PA}, \texttt{FL-BL-NPA} \}$) cannot be easily established for this action, making the characterization of threshold-type policies non-trivial for this problem.
Similar observations can be made for the other three actions (i.e., \texttt{paint}, \texttt{inspect}, and \texttt{reject}) as well.

\end{document}